\documentclass[accepted]{uai2022} 
\usepackage[utf8]{inputenc} 
\usepackage[T1]{fontenc}    
\usepackage{url}            
\usepackage{booktabs}       
\usepackage{nicefrac}       
\usepackage{bm}
\usepackage{amsthm,amsmath,amsfonts,amssymb}
\usepackage{mathtools}
\usepackage{algorithmic}
\usepackage{algorithm}
\usepackage{xcolor}
\usepackage{graphicx}
\usepackage{subcaption}
\usepackage{multirow}
\usepackage{multicol}
\usepackage{hyperref}
\usepackage{mathrsfs, fancyhdr}

\usepackage{makecell}
\usepackage{rotating}
 
\newtheorem{theorem}{Theorem}
\newtheorem{lemma}{Lemma}
\newtheorem{proposition}{Proposition}

\newtheorem{definition}{Definition}

\newtheorem{assumption}{Assumption}

\newtheorem{remark}{Remark}
\def\S{S}

\usepackage{natbib} 
\usepackage{xr}

\makeatletter
\newcommand*{\addFileDependency}[1]{
  \typeout{(#1)}
  \@addtofilelist{#1}
  \IfFileExists{#1}{}{\typeout{No file #1.}}
}
\makeatother

\newcommand*{\myexternaldocument}[1]{%
    \externaldocument{#1}%
    \addFileDependency{#1.tex}%
    \addFileDependency{#1.aux}%
}

\myexternaldocument{yang_122-supp}

\usepackage{yang-macros}

\title{Differentially Private SGDA for Minimax Problems}

\author[1]{Zhenhuan Yang}
\author[2]{Shu Hu}
\author[3]{Yunwen Lei}
\author[4]{Kush R Varshney}
\author[2]{Siwei Lyu}
\author[1]{\href{mailto:<yying@albany.edu>?Subject=Your UAI 2022 paper}{Yiming Ying}{}}
\affil[1]{%
    University at Albany\\
    Albany, New York, USA
}
\affil[2]{%
    University at Buffalo\\
    Buffalo, New York, USA
}
\affil[3]{%
    University of Birmingham\\
    Birmingham, UK
}
\affil[4]{%
    IBM Research\\
    Yorktown Heights, New York, USA
  }

\begin{document}
\maketitle

\begin{abstract}
Stochastic gradient descent ascent (SGDA) and its variants have been the workhorse for solving minimax problems. However,  in contrast to the well-studied stochastic gradient descent (SGD) with differential privacy (DP) constraints,  there is  little work on understanding the generalization (utility)  of SGDA with DP constraints. In this paper, we use the algorithmic stability approach to establish the generalization (utility) of DP-SGDA in different settings. In particular, for the convex-concave setting, we prove that the DP-SGDA can achieve  an optimal utility rate in terms of the weak primal-dual population risk in both smooth and non-smooth cases. To our best knowledge, this is the first-ever-known result for DP-SGDA in the non-smooth case.  We further provide its  utility  analysis in   the nonconvex-strongly-concave setting which is  the  first-ever-known result in terms of the primal population risk.  The convergence and generalization results for this nonconvex setting  are new even in the non-private setting.  Finally,  numerical experiments are conducted to  demonstrate the effectiveness of DP-SGDA  for both convex and nonconvex cases.
\end{abstract}

\section{Introduction}\label{sec:introduction}

In recent years, there is a growing interest on studying the minimax problems which involve both minimization over the primal variable $\wbf$ and maximization over the dual variable $\vbf$. Notable examples include generative adversarial networks (GANs) \citep{goodfellow2014generative,arjovsky2017wasserstein}, AUC maximization \citep{gao2013one,ying2016stochastic,Natole2018,liu2019stochastic,zhao2011online}, robust learning \citep{audibert2011robust,xu2009robustness},   adversarial training \citep{sinha2017certifying},  algorithmic fairness   \citep{mohri2019agnostic,li2019fair,wang2020robust,martinez2020minimax,diana2021minimax}, and  Markov Decision Process (MDP) \citep{puterman2014markov,wang2017primal}. Details of these motivating examples are given in Appendix \ref{sec:motivating-example}.

The minimax  problem can be formulated as 
\begin{equation}\label{eq:SSP}
\min_{\wbf \in \Wcal} \max_{\vbf \in \Vcal}\Big\{ F(\wbf, \vbf) := \Ebb_{\zbf\sim \Dcal}[f(\wbf,\vbf; \zbf)]\Big\}, 
\end{equation}
where $\Wcal\subseteq \mathbb{R}^{d_1}$ and $\Vcal\subseteq \mathbb{R}^{d_2}$ are two nonempty closed and convex domains and $\zbf$ is a random variable from some distribution $\Dcal$ taking values in $\mathcal{Z}$.  
Since the distribution $\Dcal$ is usually unknown and one has access only to an i.i.d. training dataset $\S = \{\zbf_1, \cdots, \zbf_n\}$,  one resorts to solving its  empirical minimax problem 
\begin{align*}
\min_{\wbf \in \Wcal} \max_{\vbf \in \Vcal} \Big\{F_\S(\wbf,\vbf) : = \frac{1}{n}\sum_{i=1}^{n} f(\wbf, \vbf; \zbf_i)\Big\}.
\end{align*}
One popular optimization algorithm for solving this problem is SGDA. Specifically, at iteration $t$, upon receiving a random data point or mini-batch  from $\S$, it performs gradient descent over $\wbf$ with the stepsize $\eta_{\wbf,t}$ and gradient ascent over   $\vbf$ with the stepsize $\eta_{\vbf,t}$. 

As SGDA is conceptually simple and easy  to implement, it is widely  deployed in solving minimax problems, e.g., GANs \citep{goodfellow2014generative}, adversarial learning \citep{sinha2017certifying}, and AUC maximization \citep{ying2016stochastic}. Its local convergence analysis for nonconvex-(strongly)-concave problems was established  in \citet{lin2020gradient}. Other variants of SGDA  were proposed and studied in \citet{luo2020stochastic,nouiehed2019solving,rafique2021weakly,yan2020optimal}.

On another front, collected data often contain sensitive information such as individual records from hospitals, online behavior from social media, and genomic data from cancer diagnosis.   Differential privacy \citep{dwork2014algorithmic} has emerged as a well-accepted mathematical definition of privacy which ensures that an attacker gets roughly the same information from the dataset regardless of whether an individual is present or not. Its related technologies have been adopted by Google \citep{google-DP}, Apple \citep{miscrosoft-DP}, and the US Census Bureau \citep{us-census-bureau-DP}. While SGD and SGDA have become the workhorse behind  the remarkable progress of machine learning and AI, it is of pivotal importance for developing  their counterparts  with DP constraints.   

Many studies analyze the privacy and utility of DP-SGD for the ERM problem that only involves the minimization over $\wbf$ \citep{bassily2019private,bassily2020stability,feldman2020private,song2013stochastic,WLYZ,wang2020differentially,wang2019subsampled,wu2017bolt,zhou2020private}. In contrast, there is little work on analysing the utility of minimax optimization algorithms with DP constraints except the  recent work of \citet{boob2021optimal}. However, \citet{boob2021optimal} focus on the noisy stochastic extragradient method on convex-concave and smooth settings.  

Studying the computational and statistical behavior of DP-SGDA is fundamental towards the understanding of stochastic optimization algorithm for minimax problem under the differential privacy constraint. In this paper, we propose novel convergence and stability analysis to establish the utility of DP-SGDA in empirical saddle point and population forms such as the weak primal-dual population risk and the primal population risk. We collect in Table \ref{tab:summary} the notations and results of performance measures in this paper. In particular, our contributions can be summarized as follows. 

\begin{table*}[ht]
    \centering
    \setlength{\tabcolsep}{3pt}
    \begin{tabular}{|c|c|c|c|c|c|}
    \hline
     Algorithm  &  Assumption & Measure & Rate & Complexity & Simplicity \\\hline
         NSEG & C-C, Lip, S & $\triangle^w(A_\wbf(S), A_\vbf(S))$ & $\Ocal\Big(\frac{1}{\sqrt{n}} + \frac{\sqrt{d\log(1/\delta)}}{n\epsilon}\Big)$ & $\Ocal(n^2)$ & Single-loop\\\hline
         NISPP & C-C, Lip, S & $\triangle^w(A_\wbf(S), A_\vbf(S))$ & $\Ocal\Big(\frac{1}{\sqrt{n}} + \frac{\sqrt{d\log(1/\delta)}}{n\epsilon}\Big)$ & $\Ocal(n^{3/2}\log(n))$ & Double-loop\\ \hline
         \multirow{3}{*}{\makecell{DP-SGDA\\(Ours)}} & C-C, Lip, S & $\triangle^w(A_\wbf(S), A_\vbf(S))$ & $\Ocal\Big(\frac{1}{\sqrt{n}} + \frac{\sqrt{d\log(1/\delta)}}{n\epsilon}\Big)$ & $\Ocal(n^{3/2})$ & \multirow{3}{*}{Single-loop}\\\cline{2-5}
         & C-C, Lip & $\triangle^w(A_\wbf(S), A_\vbf(S))$ & $\Ocal\Big(\frac{1}{\sqrt{n}} + \frac{\sqrt{d\log(1/\delta)}}{n\epsilon}\Big)$ & $\Ocal(n^{5/2})$ & \\\cline{2-5}
          & PL-SC, Lip, S & $R(A_\wbf(S)) - \min_{\wbf}R(\wbf)$  & $\Ocal\Big(\frac{1}{n^{1/3}} + \frac{\sqrt{d\log(1/\delta)}}{n^{5/6}\epsilon}\Big)$ & $\Ocal(n^{3/2})$ & \\\hline
    \end{tabular}
    \caption{\it Summary of Results. DP-SGDA is Algorithm \ref{alg:dp-sgda} in this paper. NSEG and NISPP are Algorithm 1 and 2 in \citet{boob2021optimal}, respectively. Here C-C means convexity and concavity, PL-SC means PL condition and strong concavity, Lip means Lipschitz continuity, S means the smoothness. $\triangle^w(A_\wbf(S), A_\vbf(S))$ is the weak PD population risk and $R(A_\wbf(S)) - \min_{\wbf}R(\wbf)$ is the excess primal population risk. \label{tab:summary}}
\end{table*}

\noindent$\bullet$ We analyze the privacy and utility of  DP-SGDA under the convex-concave setting in terms of the weak primal-dual population risk, i.e., $\max_{\vbf\in\Vcal}\Ebb\big[F(A_\wbf(\S),\vbf)\big]\!-\!\min_{\wbf\in\Wcal}\Ebb\big[F(\wbf,A_\vbf(\S))\big]$, \! where $(A_\wbf(\S),\! A_\vbf(\S))$ is the output of DP-SGDA.  Specifically, we show that it can guarantee $(\epsilon,\delta)$-DP  and achieve the optimal rate $\Ocal\Big(\frac{1}{\sqrt{n}} +  \frac{\sqrt{d\log (1/\delta)}}{n\epsilon}\Big)$  for smooth and nonsmooth cases where $d = \max\{d_1, d_2\}.$ To our best knowledge, this is the first-ever known result for DP-SGDA in the nonsmooth case.    

\noindent $\bullet$ We further study the utility of DP-SGDA in the nonconvex-strongly-concave case in terms of the primal population risk, i.e., $R(A_\wbf(\S)) =  \max_{\vbf\in\Vcal}\Ebb\big[F(A_\wbf(\S),\vbf)\big].$ In particular, under the \PL (PL) condition of $F_\S$, we prove that the excess primal population risk, i.e.,  $R(A_\wbf(\S)) - \min_{\wbf\in \Wcal} R(\wbf)$, enjoys the rate $\Ocal\Big(\frac{1}{n^{1/3}} + \frac{\sqrt{d\log(1/\delta)}}{n^{5/6}\epsilon}\Big)$ while guaranteeing $(\epsilon, \delta)$-DP. The key techniques involve the convergence analysis of $R_S(A_\wbf(\S)) - \min_{\wbf} R_S(\wbf)$ and the stability analysis for $A_\wbf(S)$ which are of interest in their own rights. As far as we are aware, these results are the first ones known for DP-SGDA in the nonconvex setting. 
 
\noindent $\bullet$ We perform numerical experiments on three benchmark datasets which validate the effectiveness of DP-SGDA for both convex and non-convex cases.

\subsection{Motivating Examples}
We give two examples of minimax problems under the DP constraint. See Appendix \ref{sec:motivating-example} for more examples and details.

\textbf{AUC Maximization.} Area Under the ROC Curve (AUC) is a widely used measure for binary classification. It has been shown optimizing AUC is equivalent to a minimax problem once auxiliary variables $a, b, v \in \Rbb$ are introduced \citep{ying2016stochastic}.
\begin{align*}
\min_{\theta, a, b}\max_{v}\Big\{	F(\theta,a,b,v) = \Ebb_\zbf[f(\theta, a, b, v;\zbf)]\Big\}.
\end{align*}
Differential privacy has been applied to learn private classifier by optimizing AUC \citep{wang2021differentially}.

\textbf{Generative Adversarial Networks.} Originally proposed in \citet{goodfellow2014generative}, GAN in general can be written as a minimax problem between a generator network $G_\vbf$ and a discriminator network $D_\wbf$
\begin{align*}
\min_{\wbf}\max_{\vbf} \mathbb{E}[f(\wbf,\vbf;\zbf,\xi)] \!=\! \mathbb{E}_{\zbf} [D_\wbf(\zbf)] \!-\! \mathbb{E}_{\xi} [D_\wbf(G_\vbf(\xi))].
\end{align*}
DP-SGDA and its variants were employed to train differential private GANs by \citet{xie2018differentially}. Recently differential privacy has successfully applied to private data generation by GAN framework \citep{jordon2018pate, beaulieu2019privacy}.

\subsection{Related Work}


Below  we briefly discuss some related work. 

\noindent{\bf Convergence analysis for SGDA.} It is a classical result that SGDA can achieve a convergence rate   $\Ocal(1/\sqrt{T})$ in the convex and concave case \citep{nedic2009subgradient,nemirovski2009robust} where $T$ is the number of iterations. For the nonconvex-(strongly)-concave case,  the work of  \citet{lin2020gradient} shows the local convergence of SGDA if  the stepsizes $\eta_{\wbf,t}$ and  $\eta_{\vbf,t}$ are chosen to be appropriately different. Other important studies consider variants of SGDA and prove their local  convergence for the nonconvex case. Such algorithms include nested algorithms  \citep{rafique2021weakly} for weakly-convex-weakly-concave problems, multi-step GDA  \citep{nouiehed2019solving} under the one-sided PL condition,  epoch-wise SGDA  \citep{yan2020optimal}, and  stochastic recursive SGDA \citep{luo2020stochastic} for  nonconvex-strongly-concave problems, to mention but a few.

\noindent{\bf Stability and generalization of non-private SGD and SGDA.} 
The studies of \citep{hardt2016train,charles2018stability,kuzborskij2018data} use uniform stability \citet{bousquet2002stability} to derive the generalization of non-private SGD for the convex and smooth case while the convex and nonsmooth case was established by \citet{bassily2020stability,lei2020fine}. The nonconvex case under the PL-condition was  considered by \citet{charles2018stability,lei2021sharper}. The stability and generalization of SGDA for minimax problems were studied by \citet{lei2021stability} in different forms for convex and nonconvex, smooth, and nonsmooth cases, and by  \citet{farnia2021train} with focus on the smooth cases.

{\bf DP-SGD and DP-SGDA.} DP-SGD was shown to attain the optimal excess population risk  $\Ocal({1}/{\sqrt{n}} +  {\sqrt{d\log (1/\delta)}}/{n\epsilon})$ in \citet{bassily2019private,bassily2020stability,WLYZ,wang2020differentially} for the convex case.  For nonconvex objectives,  \citet{wang2019differentially} studied the DP Gradient Langevin Dynamics, and \citet{zhang2021private} studied a multi-stage type of DP-SGD assuming the  weakly-quasi-convexity and PL condition.  In \citet{xie2018differentially,zhang2018differentially},  DP-SGDA and its variants together with clipping techniques were employed to train differentially private GANs which showed  promising results in applications. However, no utility analysis was given there.  \citet{boob2021optimal} focused on the noisy stochastic extragradient method with DP constraints for minimax problems in the convex-concave and smooth settings and provided its utility analysis using variational inequality (VI) and stability approaches. 



\section{Problem  Formulation }\label{sec:preliminaris}
In this section, we introduce necessary assumptions, notations and the DP-SGDA algorithm.
\subsection{Assumptions and Notations}\label{sec:assumption}
Firstly, we introduce necessary  assumptions and notations. A function $h: \Wcal \rightarrow \Rbb$ is said to be convex if,  for all $\wbf, \wbf' \in \Wcal$, there holds $h(\wbf) \geq h(\wbf')+ \langle\nabla h(\wbf'), \wbf - \wbf'\rangle$ where $\nabla$ is the gradient operator and $\langle \cdot, \cdot\rangle$ is the inner product. Let $\|\cdot\|_2$ denote the Euclidean norm. We say $h$ is $\rho$-strongly-convex if $h - \frac{\rho}{2}\|\wbf\|_2^2$ is convex, $h$ is concave if $-h$ is convex, and $\rho$-strongly-concave if $-h-\frac{\rho}{2}\|\wbf\|_2^2$ is convex. Let $[n]:=\{1,2,\ldots,n\}$.

\begin{definition}
Given a function $h: \Wcal \times \Vcal \rightarrow \Rbb$. We say $h$ is convex-concave if for any $\vbf \in\Vcal$, the function $\wbf \mapsto h(\wbf,\vbf)$ is convex and for any $\wbf \in\Wcal$, the function $\vbf \mapsto h(\wbf,\vbf)$ is concave.
\end{definition}

\begin{assumption}[\textbf{A1}]\label{ass:lipschitz}
The function $f$ is said to be Lipschitz continuous if there exist $G_\wbf, G_\vbf > 0$ such that, for any $\wbf, \wbf' \in \Wcal, \vbf, \vbf' \in \Vcal$ and $\zbf \in \Zcal$, 
$\|f(\wbf, \vbf; \zbf) - f(\wbf', \vbf; \zbf)\|_2 \leq  G_\wbf \|\wbf - \wbf'\|_2$, and    
$\|f(\wbf, \vbf; \zbf) - f(\wbf, \vbf'; \zbf)\|_2 \leq  G_\vbf \|\vbf - \vbf'\|_2.$ And denote $G = \max\{G_\wbf, G_\vbf\}$.
\end{assumption}

\begin{assumption}[\textbf{A2}]\label{ass:bounded-variance}
For randomly drawn $j\in [n]$, the gradients $\nabla_\wbf f(\wbf, \vbf; \zbf_j)$ and $\nabla_\vbf f(\wbf, \vbf; \zbf_j)$ have bounded variances $B_\wbf$ and $B_\vbf$ respectively. And let $B=\max\{B_\wbf, B_\vbf\}$.
\end{assumption}

\begin{assumption}[\textbf{A3}]\label{ass:smooth}
The function $f$ is said to be smooth if it is continuously differentiable and there exists a constant $L > 0$ such that for any $\wbf,\wbf'\in \Wcal$, $\vbf, \vbf'\in \Vcal$ and $\zbf \in \Zcal$,
\begin{align*}
\left\|\!\begin{pmatrix}\!\nabla_\wbf f(\wbf, \vbf; \zbf) \!-\! \nabla_\wbf f(\wbf', \vbf'; \zbf)\!\\
\!\nabla_\vbf f(\wbf, \vbf; \zbf) \!-\! \nabla_\vbf f(\wbf', \vbf'; \zbf)\!
\end{pmatrix}\!\right\|_2 \!\leq\! L\! \left\|\!\begin{pmatrix}
\!\wbf \!-\! \wbf'\\
\!\vbf \!-\! \vbf'
\end{pmatrix}\!\right\|_2 
\end{align*}
\end{assumption}

We also require the \PL (PL) condition. 
\begin{definition}[\citep{polyak1964gradient}]\label{ass:pl}
A function $h: \Wcal \rightarrow \Rbb$ satisfies the PL condition if there exist a constant $\mu>0$ such that, for any $\wbf \in \Wcal$, 
$ \frac{1}{2}\|\nabla h(\wbf)\|_2^2 \geq  \mu(h(\wbf) - \min_{\wbf' \in \Wcal} h(\wbf')).$
\end{definition}
We refer to \citet{karimi2016linear} for a nice discussion of this condition and other general conditions that allow the global convergence of gradient descent.
 
\subsection{DP-SGDA Algorithm}\label{sec:algorithm}
We now move on to the definition of differential privacy and the description of  DP-SGDA.  Differential privacy was introduced by \citet{dwork2006calibrating,dwork2014algorithmic}. We say that two datasets $\S,\S'$ are neighboring datasets if they differ by at most one example.

\begin{algorithm}[ht!]
\caption{Differentially Private Stochastic Gradient Descent Ascent (DP-SGDA) Method\label{alg:dp-sgda}}
\begin{algorithmic}[1]
\STATE {\bf Inputs:} data $\S = \{\zbf_i: i \in [n]\}$, privacy budget $\epsilon, \delta$, number of iterations $T$, learning rates $\{\eta_{\wbf,t}, \eta_{\vbf,t}\}_{t=1}^T$, and initialize $(\wbf_0, \vbf_0)$
\STATE Compute noise parameters $\sigma_\wbf$ and $\sigma_\vbf$ based on Eq. \eqref{eq:sigma-sigma}
\FOR{$t=1$ to $T$}
\STATE Sample a mini-batch $I_t = \{i_t^1, \cdots, i_t^m\in [n]\}$ uniformly with replacement
\STATE Sample independent noises $\xi_t \sim \Ncal(0, \sigma_\wbf^2 I_{d_1})$ and $\zeta_t \sim \Ncal(0, \sigma_\vbf^2 I_{d_2})$
\STATE
$
\!\wbf_{t\!+\!1} \!=\! \Pi_{\Wcal}\!\Big(\!\wbf_t \!-\! \eta_{\wbf,t} \!\Big(\frac{1}{m}\!\sum_{j\!=\!1}^m \!\nabla_\wbf \!f(\wbf_t, \vbf_t; \zbf_{i_t^j}) \!+\! \xi_t\!\Big)\!\Big)
$
\STATE
$
\vbf_{t\!+\!1} \!=\! \Pi_{\Vcal}\!\Big(\!\vbf_t \!+\! \eta_{\vbf,t} \!\Big(\!\frac{1}{m}\!\sum_{j=1}^m\!\nabla_\vbf \!f(\wbf_t, \vbf_t; \zbf_{i_t^j}) \!+\! \zeta_t\!\Big)\!\Big)
$
\ENDFOR
\STATE {\bf Outputs:}\!\! $(\bar{\wbf}_T, \bar{\vbf}_T)\! =\! \frac{1}{T}\displaystyle\sum_{t=1}^T (\wbf_t, \vbf_t)$ or $({\wbf}_T, {\vbf}_T)$
\end{algorithmic}
\end{algorithm} 
\begin{definition}[Differential Privacy]
A (randomized) algorithm $A$ is called $(\epsilon,\delta)$-differentially private (DP) if, for all neighboring datasets $\S, \S'$ and for all events $O$ in the output space of $A$, the following holds 
\[
\Pbb[A(\S) \in O] \leq e^\epsilon	\Pbb[A(\S') \in O] + \delta.
\]	
\end{definition}
Our aim is to design a randomized algorithm satisfying $(\epsilon, \delta)$-DP which solves the empirical minimax problem: 
\begin{equation}\label{eq:ESPP} \min_{\wbf\in \Wcal}\max_{\vbf\in \Vcal} \Big\{F_\S(\wbf,\vbf) = \frac{1}{n}\sum_{i=1}^n f(\wbf,\vbf; \zbf_i)\Big\}. 
\end{equation}
Notice that in the standard ERM problem, which involves the minimization only with respect to $\wbf$,  DP-SGD \citep{wu2017bolt,song2013stochastic,bassily2019private,wang2020differentially} uses the gradient perturbation at each iteration. Specifically,   at each iteration of this algorithm, a randomized gradient estimated from a random subset (mini-batch) of $\S$ is perturbed by a Gaussian noise and then the model parameter is updated based on this noisy gradient. 

Following the same spirit, DP-SGDA  \citep{xie2018differentially,zhang2018differentially} adds Gaussian noises per iteration to the randomized gradient mapping $(g_{\wbf,t}, g_{\vbf,t}) = (\frac{1}{m}\sum_{j=1}^m\nabla_\wbf f(\wbf_t, \vbf_t; \zbf_{i_t^j}), \frac{1}{m}\sum_{j=1}^m\nabla_\vbf f(\wbf_t, \vbf_t; \zbf_{i_t^j}))$ where the index of example $\zbf_{i_t^j}$ is from the mini-batch $I_t$.  Then, the primal variable $\wbf$ is updated by gradient descent based on the noisy gradient $g_{\wbf,t} +\xi_t$ and the dual variable $\vbf$ is updated by gradient ascent based on the noisy gradient $g_{\vbf,t} +\zeta_t$.  The pseudo-code for DP-SGDA is given in Algorithm \ref{alg:dp-sgda}. The noise levels $\sigma_\wbf, \sigma_\vbf$ are given by \eqref{eq:sigma-sigma} which will be specified soon in Section \ref{sec:results} in order to guarantee $(\epsilon, \delta)$-DP.  The notations $\Pi_{\Wcal}(\cdot)$ and $ \Pi_{\Vcal}(\cdot)$ denote the projections to $\Wcal$ and $\Vcal$, respectively. 
From now on, the notation $A$ denotes the DP-SGDA algorithm and its  output is denoted by  $A(\S) = (A_\wbf(\S), A_\vbf(\S)).$

\subsection{Measures of Utility}\label{sec:measure}
Since the model $A(\S)$ is only trained based on the training data $\S$, its empirical behavior as measured by $F_\S$ may not generalize well on test data.  Our goal is to investigate the statistical behavior of  $A(\S)$ on the test data in terms of some population risk. However, unlike the standard statistical learning theory (SLT) setting where there is only a minimization of $\wbf$, we have different measures of population risk due to the minimax structure \citep{zhang2021generalization,lei2021stability}. 
Let $\Ebb[\cdot]$ denote the expectation with respect to the randomness of  algorithm $A$ and data $S$. We are particularly  interested in the following metrics.  

\begin{definition}[Weak Primal-Dual (PD) Risk]
The weak primal-dual population risk of $A(\S)$,  denoted by $\triangle^w(A_\wbf(\S), A_\vbf(\S))$,  is defined as
$$\max_{\vbf\in\Vcal}\Ebb\big[F(A_\wbf(\S),\vbf)\big]\!-\!\min_{\wbf\in\Wcal}\Ebb\big[F(\wbf,A_\vbf(\S))\big].$$
The corresponding weak PD empirical risk, denoted by $\triangle^w_S(A_\wbf(\S), A_\vbf(\S))$,  is defined as
\[
\max_{\vbf\in\Vcal}\Ebb\big[F_S(A_\wbf(\S),\vbf)\big]\!-\!\min_{\wbf\in\Wcal}\Ebb\big[F_S(\wbf,A_\vbf(\S))\big].
\]
\end{definition}
\begin{definition}[Primal Risk] The primal population risk of $A(\S)$ is given by $R(A_\wbf(\S))=\max_{\vbf\in\Vcal}F(A_\wbf(\S),\vbf)$ and empirical risk is defined by $R_S(A_\wbf(\S))=\max_{\vbf\in\Vcal}F_S(A_\wbf(\S),\vbf)$, respectively. The excess primal population risk is defined as
$$
\Ebb\big[R(A_\wbf(\S))-\min_{\wbf\in\Wcal}R(\wbf)\big].
$$ 
The corresponding excess primal empirical risk is then
$$
\Ebb\big[R_S(A_\wbf(\S))-\min_{\wbf\in\Wcal}R_S(\wbf)\big].
$$ 
\end{definition}

Meanwhile, the strong PD risk defined as $\triangle^s({\wbf},{\vbf})=\Ebb\big[\sup_{\vbf'\in\Vcal}F({\wbf},\vbf')-\inf_{\wbf'\in\Wcal}F(\wbf',{\vbf})\big]$. We have $\triangle^w(A_\wbf(\S), A_\vbf(\S)) \leq \triangle^s(A_\wbf(\S), A_\vbf(\S))$ by applying Jensen's inequality. However, when $F$ is strongly-convex-strongly-concave, the point distance from the model $(A_\wbf(\S), A_\vbf(\S))$ to the true saddle point $(\wbf^*, \vbf^*) \in \arg\min_{\wbf\in \Wcal}\max_{\vbf \in \Vcal} F(\wbf, \vbf)$ can be bounded by the weak PD population risk, i.e.\ $\Ebb[\|A_\wbf(\S) - \wbf^*\|_2^2 + \|A_\vbf(\S) - \vbf^*\|_2^2] \leq \Ocal(\triangle^w(A_\wbf(\S), A_\vbf(\S)))$. For certain problems, it is suffices to bound the weak PD risk, such as the learning problem for Markov decision process in Appendix \ref{sec:motivating-example}. The primal risk is more meaningful when one is concerned about the risk with respect to the primal variable, such as the AUC maximization problem.


\section{Main Results}\label{sec:results}

In this section, we present our main theoretical results for DP-SGDA. For the privacy guarantee,  we leverage the moments accountant method \citep{abadi2016deep}, which implies tight privacy loss for adaptive Gaussian mechanisms with amplification by subsampling.  Below we summarize a specific version of this method that suffices for our purpose.

\begin{theorem}\label{thm:moments-accountant-privacy}Let \textbf{(A1)} hold true. Then, 
there exist constants $c_1, c_2$ and $c_3$ so that given the mini-batch size $m$ and total iterations $T$, for any $\epsilon < c_1 m^2T/n^2$, Algorithm \ref{alg:dp-sgda} is $(\epsilon, \delta)$-differentially private for any $\delta > 0$ if we choose
\begin{equation}\label{eq:sigma-sigma}
\sigma_\wbf \!=\! \frac{c_2 G_\wbf \sqrt{T\log(1/\delta)}}{n\epsilon},\, \sigma_\vbf \!=\! \frac{c_3 G_\vbf \sqrt{T\log(1/\delta)}}{n\epsilon}.
\end{equation}
\end{theorem}
 
The proof of Theorem \ref{thm:moments-accountant-privacy} is given in Appendix \ref{sec:proof-privacy}.

\begin{remark}\label{rem:choice-of-param}
In practice, given privacy budget $\epsilon, \delta$ and parameters $m, T$, the constant $c_2$ and hence $\sigma$ can be found by grid search \citep{abadi2016deep}. Here we provide a set of parameters that satisfies the condition  in that reference and our Theorem \ref{thm:moments-accountant-privacy}. That is, by choosing $\epsilon \leq 1, \delta\leq 1/n^2$ and $m = \max(1, n\sqrt{\epsilon/(4T)})$, then we have explicit values for the variances as $
\sigma_\wbf = \frac{8 G_\wbf \sqrt{T\log(1/\delta)}}{n\epsilon}, \sigma_\vbf = \frac{8 G_\vbf \sqrt{T\log(1/\delta)}}{n\epsilon}.$
\end{remark}

\begin{remark}\label{rem:different-sensitivity}
Our Algorithm \ref{alg:dp-sgda} allows the application of independent noises $\xi_t, \zeta_t$ with different $\sigma_\wbf, \sigma_\vbf$, respectively. In  \citet{boob2021optimal}, a uniform $\sigma$ is used (Theorem 5.4 or 7.4 there) for both primal and dual variables. In many examples, the primal and dual gradients $\nabla_\wbf f(\wbf_t, \vbf_t,\zbf_{i_t^j}), \nabla_\vbf f(\wbf_t, \vbf_t,\zbf_{i_t^j})$ enjoy different Lipschitz constants ($\ell_2$-sensitivity).  Therefore, our treatment leads to a more delicate way of calibrating the variances of the Gaussian noises. As we shall see in the experiments in Section \ref{sec:exp}, this treatment enables Algorithm \ref{alg:dp-sgda} to achieve better performance.
\end{remark}

In the subsequent subsections, we present our main contribution of this paper, i.e., the  utility bounds of DP-SGDA for the convex-concave and nonconvex-strongly-concave cases, respectively.

\subsection{Convex-Concave Case}\label{sec:convex} 
In this subsection,  we present the utility bound of DP-SGDA for  the convex-concave case in terms of the weak PD risk of the output $(\bar{\wbf}_T,\bar{\vbf}_T)$ of Algorithm \ref{alg:dp-sgda}. 

\begin{theorem}\label{thm:sgda-utility}
Assume the function $f$ is convex-concave. Assume $\Wcal$ and $\Vcal$ are bounded so that $\max_{\wbf\in \Wcal}\|\wbf\|_2 \leq D_\wbf$, $\max_{\vbf\in \Vcal}\|\vbf\|_2 \leq D_\vbf$. And let $D = \max\{D_\wbf, D_\vbf\}$. Let the stepsizes $\eta_{\wbf,t} = \eta_{\vbf,t} = \eta$ for all $t \in [T]$ with some $\eta > 0$. Under one of the condition 
\begin{enumerate}[topsep=-1pt, partopsep=-1pt]
\item[a)] Assumption \textbf{(A1)} and \textbf{(A3)} hold true and we choose $T \asymp n$ and $\eta \asymp 1/\big(\max\{\sqrt{n}, \sqrt{d\log(1/\delta)}/\epsilon\}\big)$, 
\item[b)] or Assumption  \textbf{(A1)} holds true and we choose $T \asymp n^2$ and $\eta \asymp 1/\big(n\max\{\sqrt{n}, \sqrt{d\log(1/\delta)}/\epsilon\}\big)$,
\end{enumerate}
then Algorithm \ref{alg:dp-sgda} satisfies
\begin{align*}
\triangle^w(\bar{\wbf}_T,\bar{\vbf}_T) = \Ocal\Big( \max\Big\{\frac{1}{\sqrt{n}}, \frac{\sqrt{d\log(1/\delta)}}{n\epsilon}\Big\}\Big).
\end{align*}
\end{theorem}

Its detailed proof can be found in Appendix \ref{sec:proof-convex}.  The proof mainly relies on the concept of stability \citep{bousquet2002stability,charles2018stability,hardt2016train,kuzborskij2018data}.  Specifically,   the weak PD population risk  can be decomposed as follows:  
\begin{align*}\label{eq:weak-err-decomp}
\triangle^w(\bar{\wbf}_T, \bar{\vbf}_T)  = & \triangle^w(\bar{\wbf}_T, \bar{\vbf}_T) - \triangle^w_\S(\bar{\wbf}_T, \bar{\vbf}_T)\\
& + \triangle^w_\S(\bar{\wbf}_T, \bar{\vbf}_T), \numberthis
\end{align*}
where the term $\triangle^w(\bar{\wbf}_T, \bar{\vbf}_T) - \triangle^w_\S(\bar{\wbf}_T, \bar{\vbf}_T)$ is the generalization error and $\triangle^w_\S(\bar{\wbf}_T, \bar{\vbf}_T)$ is the optimization error.

The estimation for the optimization error can be conducted by standard techniques \citep{nemirovski2009robust}. We give a self-contained proof in Appendix \ref{sec:cc-opt}. The generalization error is estimated using  a concept of weak stability \citep{lei2021stability}. Specifically, we say the randomized algorithm $A$ is  {\em $\varepsilon$-weakly-stable} if,  for any neighboring sets $\S, \S'$ differing at one single datum,  there holds \begin{align*} & \sup_\zbf\big(\!\sup_{\vbf \in \Vcal}\Ebb_{A}[f(A_\wbf(\S), \vbf; \zbf) - f(A_\wbf(\S'), \vbf; \zbf)]\\  &  + \sup_{\wbf \in \Wcal}\Ebb_{A}[f(\wbf, A_\vbf(\S); \zbf) \!\!-\!\! f(\wbf, A_\vbf(\S'); \zbf)]\big) \!\!\leq \varepsilon.\end{align*}
We know from \citet{lei2021stability} that  $\varepsilon$-weak-stability implies 
    $\triangle^w(A_{\wbf}(\S),A_{\vbf}(\S))-\triangle^w_\S(A_{\wbf}(\S),A_{\vbf}(\S))\leq\varepsilon.$

In Appendix \ref{sec:cc-gen}, we prove the weak stability of DP-SGDA (i.e.\ Algorithm 1) for both smooth and nonsmooth cases. Putting the estimations for the optimization error and generalization error into \eqref{eq:weak-err-decomp} can yield the bound in Theorem \ref{thm:sgda-utility}. 
We end this subsection with some remarks. 


\begin{remark}\label{rem:optimality}
The utility bound $\Ocal\Big(\max\Big\{\frac{1}{\sqrt{n}}, \frac{\sqrt{d\log(1/\delta)}}{n\epsilon}\Big\}\Big)$ is optimal for convex-concave minimax problem. A lower bound with the same order has been established in the convex ERM setting \citep{bassily2014private, bassily2019private, feldman2020private} and the measure of utility is given by $\Ebb[F(A_\wbf(S)) - min_{\wbf \in \Wcal} F(\wbf)]$. Here we slightly abuse the notation to indicate $F$ as the population risk and $A_\wbf(S)$ as the algorithm for the ERM problem. Since the convex-concave minimax problem is a special case of convex ERM problems when the dual variable is constant, this lower bound also applies to our setting.
\end{remark}

\begin{remark}\label{rem:compare-extragradient}
The same optimal utility was claimed in \citet{boob2021optimal}. Yet our results also possess two theoretical gains compared to theirs. Firstly, when the smoothness assumption holds, Part a) in our Theorem \ref{thm:sgda-utility} shows the optimal utility with $T = \Ocal(n)$ iterations and $\Ocal(n^{3/2})$ gradient computations by Remark \ref{rem:choice-of-param}, while their single-looped algorithm (Algorithm 1 there) requires $\Ocal(n^2)$ gradient computations in their Theorem 5.4. They further improved the gradient complexity to $\Ocal(n^{3/2}\log(n))$ in Theorem 7.4, which, however, requires an extra subroutine algorithm (inner-loop)  (Algorithm 2 there). Secondly, we also derive the same optimal bound with only Lipschitz continuous  assumption for the nonsmooth case which was not addressed in \citet{boob2021optimal}.
\end{remark}


\subsection{Nonconvex-Strongly-Concave Case}\label{sec:nonconvex-strongly-concave}
We proceed to the case when $f$ is non-convex-strongly-concave. In this case, we can present utility bounds of DP-SGDA in terms of the primal excess risk, i.e.,  $R(\wbf_T) - \min_{\wbf\in \Wcal} R(\wbf)$,  where $\wbf_T$ is the last iterate  of Algorithm \ref{alg:dp-sgda}. Generally speaking, a saddle point may not always exist without the convexity assumption. Since our goal in this paper is to find global optima, we assume that the saddle point of the empirical minimax problem exists, i.e., there exists $(\hat{\wbf}_\S, \hat{\vbf}_\S)$ such that,   for any $\wbf\in \Wcal$ and $\vbf\in \Vcal$,
\begin{align*}
F_\S(\hat{\wbf}_\S, \vbf) \leq F_\S(\hat{\wbf}_\S, \hat{\vbf}_\S) \leq F_\S(\wbf, \hat{\vbf}_\S). 
\end{align*}

To estimate the primal excess risk,  we define $R^*_\S = \min_{\wbf\in\Wcal} R_\S(\wbf), \text{ and } R^* = \min_{\wbf\in\Wcal} R(\wbf).$ Then, for any $\wbf^* \in \arg\min_\wbf R(\wbf)$ we have the error decomposition:
\begin{align*}\label{eq:err-decomp}
\Ebb[R(\wbf_T) & \!-\! R^*] \!=\!  \Ebb[R(\wbf_T) \!\!-\!\! R_\S(\wbf_T)] \!+\! \Ebb[R_\S(\wbf_T) \!\!-\!\! R_\S^*]\\
& +  \Ebb[R_\S^* \!-\! R_\S(\wbf^*)] \!+\! \Ebb[R_\S(\wbf^*) \!-\! R(\wbf^*)]\\
\leq & \Ebb[R(\wbf_T) \!-\! R_\S(\wbf_T)] \!+\! \Ebb[R_\S(\wbf^*) \!-\! R(\wbf^*)]\\
& + \Ebb[R_\S(\wbf_T) \!-\! R_\S^*],\numberthis
\end{align*}
where the last inequality follows from the fact that $R_\S^* - R_\S(\wbf^*)\le 0 $ since $R_\S^* = \min_{\wbf\in \Wcal}R_\S(\wbf)$. The term $\Ebb[R_\S(\wbf_T) - R_\S^*]$ is the {\em optimization error} which characterizes the discrepancy between the primal empirical risk of an output of Algorithm \ref{alg:dp-sgda} and the least possible one. The term $\Ebb[R(\wbf_T) - R_\S(\wbf_T)]  + \Ebb[R_\S(\wbf^*) - R(\wbf^*)]$ is called the {\em generalization error} which measures the discrepancy  between the primal population risk and the empirical one. The estimations for these two errors are described as follows. 

\textbf{Optimization Error.} The next theorem characterizes the primal empirical risk of DP-SGDA under the PL-SC assumption.

\begin{theorem}\label{thm:sgda-primal-opt}
Assume Assumptions \textbf{(A1)} and \textbf{(A2)} hold true,  and the function $F_\S(\wbf, \cdot)$ is $\rho$-strongly concave and $F_\S(\cdot, \vbf)$ satisfies $\mu$-PL condition. Assume $\Vcal$ is bounded. Let $\kappa = L/\rho$.  If we choose $\eta_{\wbf,t} \asymp \frac{1}{\mu t}$ and $\eta_{\vbf,t} \asymp \frac{\kappa^{2.5}}{\mu^{1.5} t^{2/3}}$, then $$\Ebb[R_\S(\wbf_{T+1}) - R_\S^*] = \Ocal\Big(\frac{\kappa^{3.5}}{\mu^{2.5}}\Big(\frac{1/m + d(\sigma_\wbf^2 + \sigma_\vbf^2)}{T^{2/3}}\Big)\Big).$$
\end{theorem}
We provide the proof of  Theorem \ref{thm:sgda-primal-opt} in Appendix \ref{sec:sgda-primal-opt}. In the non-private setting, i.e. $\sigma_\wbf=\sigma_\vbf=0$, Theorem \ref{thm:sgda-primal-opt} implies that the convergence rate in terms of the primal empirical risk is of the order  $\Ocal(\frac{\kappa^{3.5}}{\mu^{2.5} T^{2/3}}),$ which is a new result even in the non-private case as far as we are aware of.  

In \citet{lin2020gradient}, the local convergence of SGDA in the non-private case was proved in terms of the metric $\Ebb_{\tau}[\|\nabla R_\S(\wbf_\tau)\|_2^2]$ where $\tau$ is chosen uniformly at random from the set $\{1, 2, \ldots, T\}.$  Our analysis is much more involved since it proves the global convergence of the last iterate $\wbf_T$.  Our main idea is to prove the coupled recursive inequalities for two terms, i.e., 
$a_t = R_\S(\wbf_t) - R_\S^*$ and $b_t = \|\vbf_t - \hat{\vbf}_\S(\wbf_t)\|_2^2$ where $ \hat{\vbf}_\S(\wbf_t)= \arg\max_{\vbf\in \Vcal} F_\S(\wbf_t, \vbf)$,  and  then  carefully derive the the convergence rate for $a_t+ \lambda_t b_t$ by choosing $\lambda_t$ appropriately. The convergence rate and its proof can be of interest in their own right. One can find more detailed arguments in Appendix \ref{sec:sgda-primal-opt}. 


\textbf{Generalization Error.} We present the bound for the generalization error which is proved again using the stability approach.

We begin with a discussion of the saddle points. While the saddle point $(\hat{\wbf}_\S, \hat{\vbf}_\S)$ may not be unique,  $\hat{\vbf}_\S$ must be unique if $F_\S(\wbf, \vbf)$ is strongly-concave in $\vbf$ (see Proposition \ref{lem:unique-v} in Appendix \ref{sec:proof-nonconvex}).  Therefore, we can define $\pi_\S(\wbf)$ the projection of $\wbf$ to the set of saddle points, as  ${\Omega}_\S = \bigl\{\hat{\wbf}_\S:   (\hat{\wbf}_\S, \hat{\vbf}_\S)  \in \arg\min_{\wbf\in \Wcal}\max_{\vbf\in \Vcal}F_\S(\wbf, \vbf)\bigr\} =\bigl\{\hat{\wbf}_\S:  \hat{\wbf}_\S \in \arg\min_{\wbf\in \Wcal}F_\S(\wbf, \hat{\vbf}_\S)\bigr\} $.

Recall that $\wbf_T$ is the iterate of DP-SGDA at time $T$ based on the training data $\S$. Likewise, we denote by $\wbf'_T$   based on the training set $\S'$ which differs from $\S$ at one single datum. Due to the possibly multiple saddle points, we need the following critical assumption for estimating the generalization error.  
\begin{assumption}[\textbf{A4}]\label{ass:unique-projection}
For the (randomized) algorithm DP-SGDA,   assume that $\pi_{\S'}(\pi_\S (\wbf_T)) = \pi_{\S'} (\wbf'_T)$ for any neighboring sets $\S$ and $ \S'.$    
\end{assumption}  
 
Assumption \textbf{(A4)} was introduced in \citet{charles2018stability} for studying the stability of SGD in  the non-convex case which only involves the minimization over $\wbf$. In our case, \textbf{(A4)} holds true whether the saddle point is unique (e.g., $F_\S$ is strongly-convex and strongly-concave) or the two sets of saddle points based on $\S$ and $\S'$, i.e. $\Omega_\S$  and $\Omega_{\S'}$ do not change too much.  Since our algorithm satisfies $(\epsilon, \delta)$-DP  it means that the distributions of $\wbf_T$ and $\wbf'_T$ generated from two neighboring sets $S$  and $S'$ are ``close'', which indicates  $\sup_{S,S'}\|\pi_{S'}(\pi_S(\wbf_T)) - \pi_{S'}(\wbf'_T)\|_2$ can be small. Proving such statement serves as an interesting open problem.

Now we can state the results on the generalization error. 

\begin{theorem}[Generalization Error]\label{thm:sgda-primal-gen}
Assume Assumptions \textbf{(A1)}, \textbf{(A3)} and \textbf{(A4)}  hold true, and  assume  the function $f(\wbf, \cdot; \zbf)$ is $\rho$-strongly concave and $F_\S(\cdot, \vbf)$ satisfies $\mu$-PL condition. Let $\kappa = L/\rho$. If $\Ebb[R_\S(\wbf_{T}) - R_\S^*] \leq \varepsilon_T$, then
$$\!\Ebb[R(\wbf_T) \!-\! R_\S(\wbf_T)]  \!\leq \!  (1\!+\!\kappa)G_\wbf\Big(\sqrt{\frac{\varepsilon_T}{2\mu}} \!+\! \frac{1}{n}\sqrt{\frac{G_\wbf^2}{4\mu^2}\! + \!\frac{G_\vbf^2}{\rho\mu}} \Big), 
$$
and 
$$
\Ebb[R_\S(\wbf^*) - R(\wbf^*)] \leq  \frac{4G_\vbf^2}{\rho n}.$$
\end{theorem}
The proof of Theorem \ref{thm:sgda-primal-gen} is provided in Appendix \ref{sec:sgda-primal-gen}.

\begin{remark} The  generalization error bounds given in Theorem \ref{thm:sgda-primal-gen} indicate that if the optimization error $\Ebb[R_\S(\wbf_{T}) - R_\S^*] $ is small then the generalization error will be small.  This is consistent with the observation in the stability and generalization analysis of SGD \citep{charles2018stability,hardt2016train,lei2021sharper} for the minimization problems in the sense of ``optimization can help generalization". 
\end{remark}

We can derive the following utility bound for DP-SGDA by combining the results in Theorems \ref{thm:sgda-primal-gen} and  \ref{thm:sgda-primal-opt}. 

\begin{theorem}\label{thm:utility-nonconvex}
Under the same assumptions of Theorem \ref{thm:sgda-primal-gen}, if we choose $T \asymp n$, $\eta_{\wbf,t} \asymp \frac{1}{\mu t}$ and $\eta_{\vbf,t} \asymp \frac{\kappa^{2.5}}{\mu^{1.5} t^{2/3}}$, then 
$$
\Ebb[R(\wbf_{T+1}) - R^*] = \Ocal\Big(\frac{\kappa^{2.75}}{\mu^{1.75}}\Big(\frac{1}{n^{1/3}} + \frac{\sqrt{d\log(1/\delta)}}{n^{5/6}\epsilon}\Big)\Big).$$
\end{theorem}

The proof can be found in Appendix \ref{sec:agda-utility}.



\section{Experiments}\label{sec:exp}
In this section, we evaluate the performance of DP-SGDA by taking AUC maximization as an example. Due to space limitation, we present the most significant information and results of our experiments while  more detailed information and additional results are given in Appendix \ref{sec:add-details} and \ref{sec:additional-exp}. 
\subsection{Experimental Settings}\label{sec:setting}
\textbf{Baseline Model.} We perform experiments on the problem of AUC maximization with the least square loss to evaluate the DP-SGDA algorithm in linear and non-linear settings (two-layer multilayer perceptron (MLP)). In this case, AUC maximization can be formulated as
 \begin{align*}
\min_{\mathbb{\theta}\in \Theta} \Ebb_{\zbf,\zbf'}[(1 - h(\theta;\xbf) + h(\theta;\xbf'))^2|y=1, y'=-1], \end{align*}
where $h: \Theta\times \Rbb^d\rightarrow\Rbb$ is the scoring function. As shown in \citet{ying2016stochastic}, it  is equivalent to a minimax problem:
\begin{align*}
\min_{\wbf=(\theta, a, b)}\max_{\vbf}	\Ebb_\zbf[f(\theta, a, b, \vbf;\zbf)],
\end{align*}
where $f = (1-p)(h(\theta;\xbf) - a)^2\Ibb[y=1] + p(h(\theta;\xbf) - b)^2\Ibb[y=-1] + 2(1+\vbf)(ph(\theta;\xbf)\Ibb[y=-1] - (1 - p)h(\theta;\xbf)\Ibb[y=1])] - p(1-p)\vbf^2$ and $p = \Pbb[y=1]$.

When $h$ is a linear function, the AUC learning objective above is convex-strongly-concave. On the other hand, when $h$ is a MLP function, it becomes a nonconvex-strongly-concave minimax problem. In addition, following \citet{liu2019stochastic}, we use Leaky ReLU as an activation function for MLP. It was shown in their paper the empirical AUC objective satisfies the PL condition with this choice of $h$. Without a special statement, we set $256$ as the number of hidden units in MLP and $64$ as the mini-batch size during the training.    

\textbf{Datasets and Evaluation Metrics.}
Our experiments are based on three popular datasets, namely ijcnn1 \citep{chang2011libsvm}, MNIST \citep{lecun1998gradient}, and Fashion-MNIST \citep{xiao2017fashion} that have been used in previous studies. For MNIST and Fashion-MNIST, following \citet{gao2013one, ying2016stochastic}, we transform their classes into  binary classes by randomly partitioning the data into two groups, each with an equal number of classes. 
For ijcnn1, we randomly split its original training set into new training ($80\%$) and testing ($20\%$) sets. For MNIST and Fashion-MNIST, we use their original training set and testing set. For each method, the reported performance is obtained by averaging the AUC scores on the test set according to $5$ random seeds (for initial $\wbf$ and $\vbf$, sampling and noise generation).

\begin{table*}[th!]
\centering
\setlength\tabcolsep{2.5pt}
{
\begin{tabular}{|c|cc|c|cc|c|cc|c|}
\hline
Dataset & \multicolumn{3}{c|}{ijcnn1}         & \multicolumn{3}{c|}{MNIST}         & \multicolumn{3}{c|}{Fashion-MNIST}         \\ \hline
\multirow{2}{*}{Algorithm} & \multicolumn{2}{c|}{Linear} &  MLP  & \multicolumn{2}{c|}{Linear} &   MLP    & \multicolumn{2}{c|}{Linear} &    MLP   \\ \cline{2-10} 
             &    NSEG       &    DP-SGDA       &    DP-SGDA   &    NSEG       &    DP-SGDA       &    DP-SGDA       &    NSEG       &    DP-SGDA       &    DP-SGDA       \\ \hline
Original &    92.191 &  92.448 &  96.609  & 93.306  & 93.349  &99.546  & 96.552 &    96.523    &   98.020   \\ \hline
 \makecell{$\epsilon$=0.1}     & 90.106    &  91.110   &  92.763 &   91.247  &  91.858 &  97.878&  95.446  &  95.468   & 95.692\\ \cline{1-10} 
                           \makecell{$\epsilon$=0.5}    &  90.346 & 91.357   &  95.840 &   91.324 &  92.058 &  98.656 &  95.530&  95.816    & 96.988   \\ \cline{1-10} 
                            \makecell{$\epsilon$=1}     & 90.355 & 91.371 &  96.167 &   91.330 &  92.070 & 98.705&  95.534 &  95.834  & 97.102\\ \cline{1-10} 
                           \makecell{$\epsilon$=5}   & 90.363 & 91.383 &  96.294 &  91.334  &  92.078 & 98.742 & 95.538&   95.848 & 97.198\\ \cline{1-10} 
                             \makecell{$\epsilon$=10}  & 90.363& 91.386 &  96.297  &   91.334 &  92.080 &  98.747& 95.539&    95.850  & 97.213  \\ \hline
\end{tabular}
\caption{\it Comparison of AUC performance in NSEG and DP-SGDA (Linear and MLP settings) on three datasets with different $\epsilon$ and $\delta$=1e-6. The ``Original'' means no noise ($\epsilon=\infty$) is added in the algorithms.}
\label{tab:general_performance_partial}
}
\end{table*}

\begin{figure*}[t]
\begin{subfigure}[t]{0.32\linewidth}
    \includegraphics[width=\linewidth]{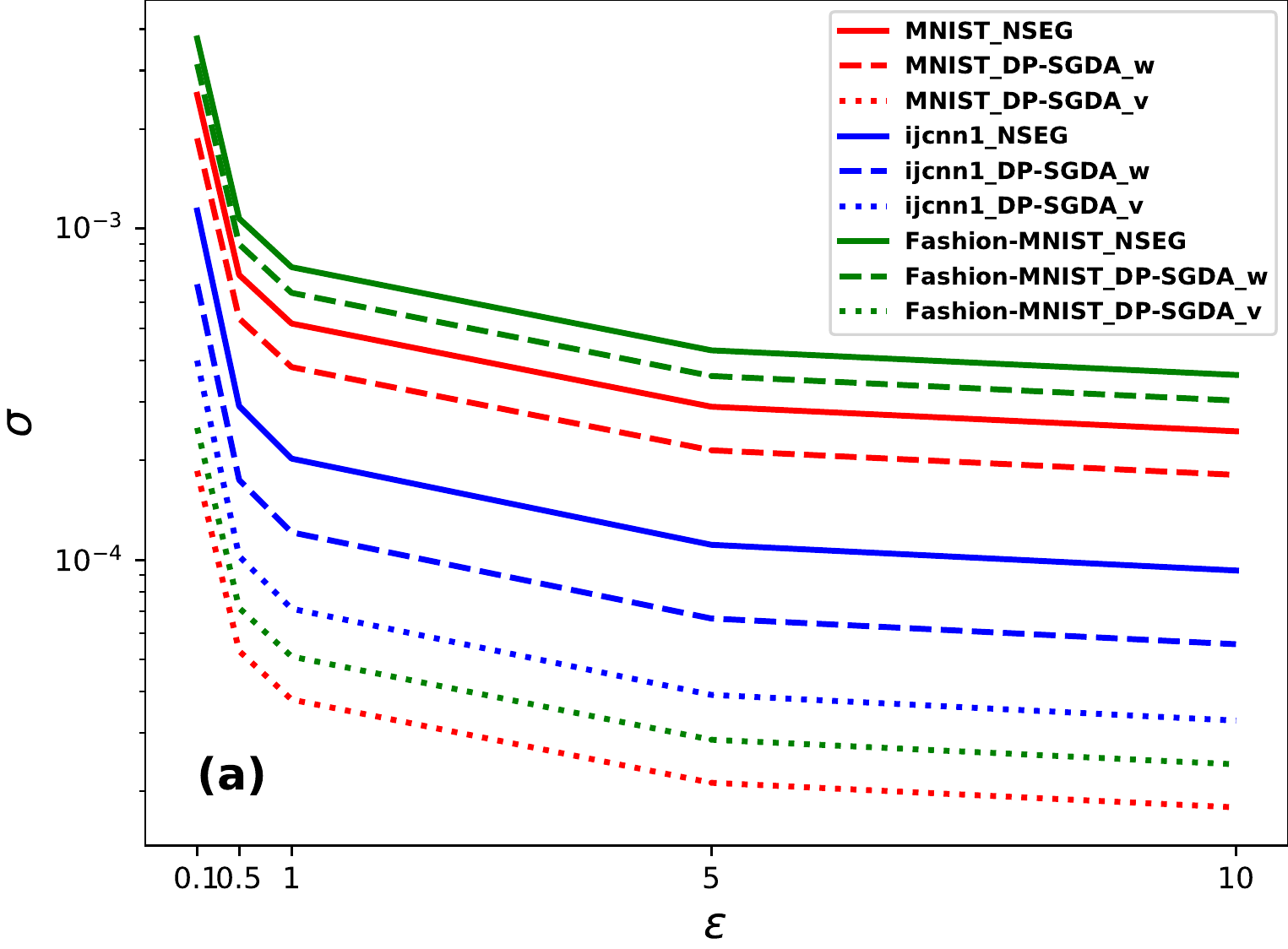}
\end{subfigure}%
    \hfill%
\begin{subfigure}[t]{0.32\linewidth}
    \includegraphics[width=\linewidth]{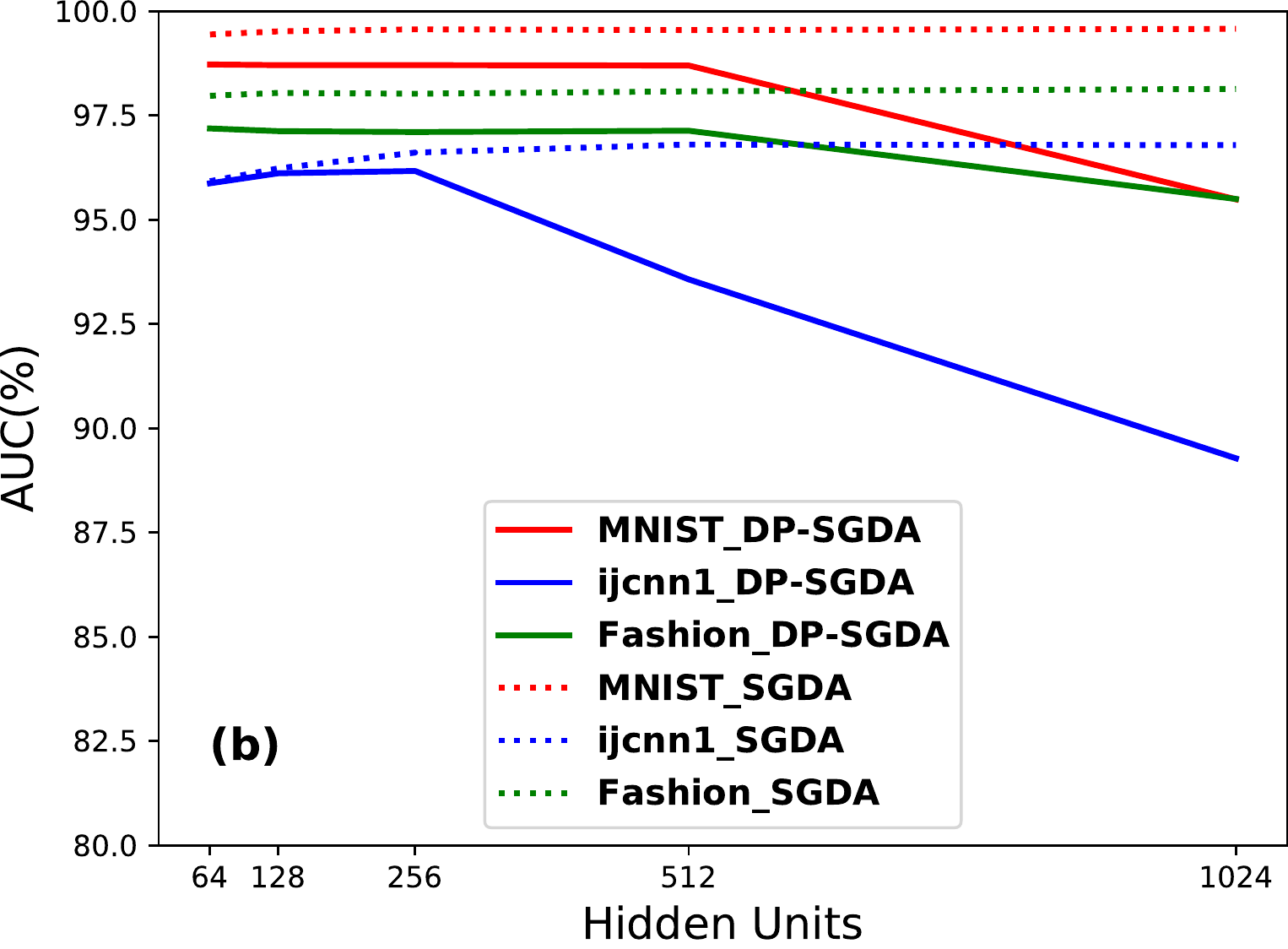}
\end{subfigure}
    \hfill%
\begin{subfigure}[t]{0.32\linewidth}
    \includegraphics[width=\linewidth]{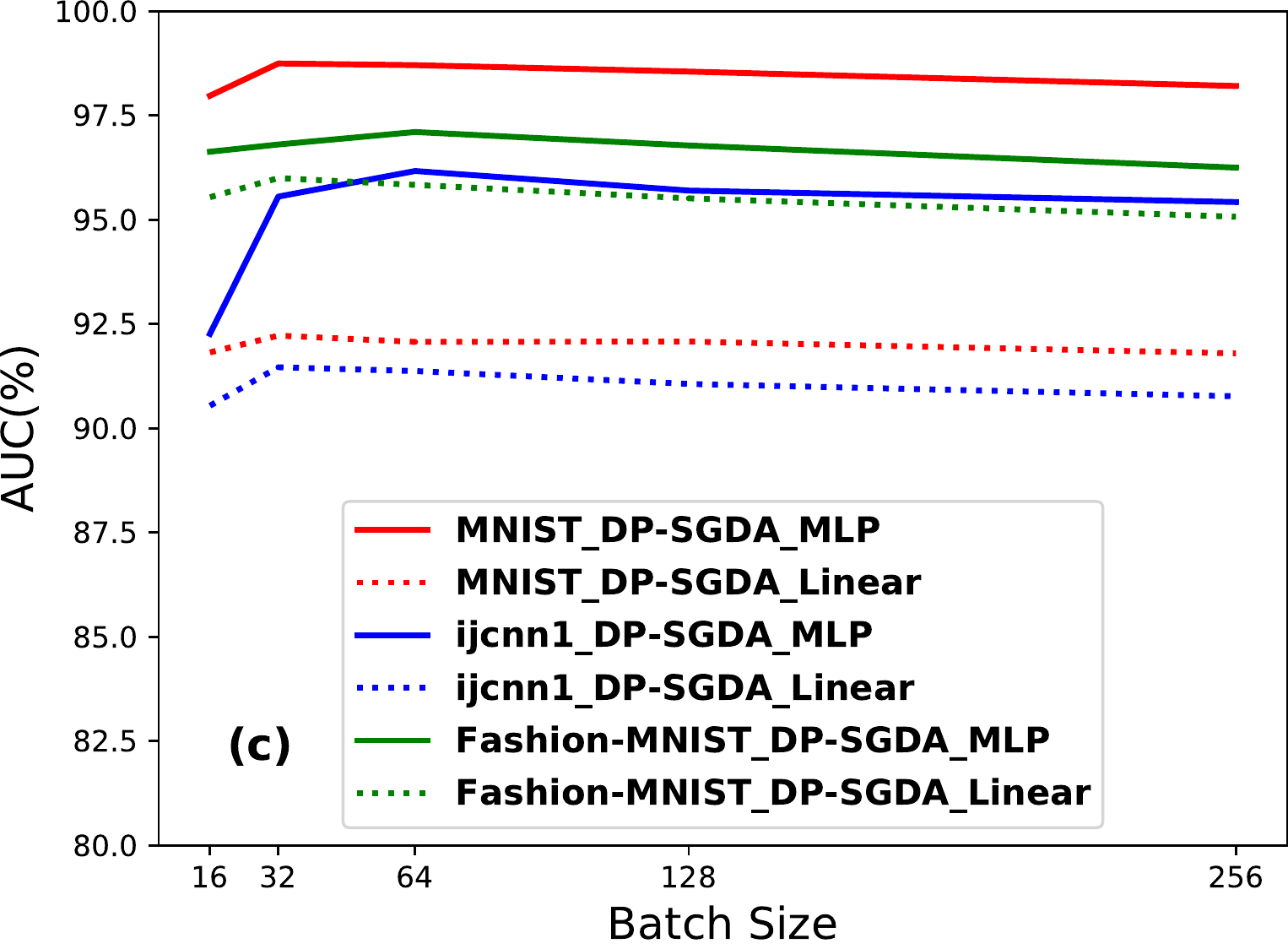}
\end{subfigure}
\caption{\em  (a) Comparison of $\sigma$ for NSEG and DP-SGDA (Linear setting) on three datasets with different $\epsilon$ and $\delta$=1e-6. (b)Comparison of AUC performance for SGDA and DP-SGDA in MLP settings on three datasets with different hidden units and $\epsilon$=1 and $\delta$=1e-6. (c) Comparison of AUC performance for DP-SGDA (Linear and MLP settings) on three datasets with different batch size and $\epsilon$=1 and $\delta$=1e-6.}
\label{fig:sigma}
\end{figure*}

\textbf{Privacy Budget Settings.} In the experiments, we set up five privacy levels from small to large: $\epsilon\in\{0.1, 0.5, 1, 5, 10\}$. We also consider three different $\delta$ from $\{\mathrm{1e\!-\!4}, \mathrm{1e\!-\!5}, \mathrm{1e\!-\!6}\}$. Due to space limitation, we only report the performance when $\delta=\mathrm{1e\!-\!6}$. More results can be found in Appendix \ref{sec:additional-exp}. To estimate the Lipschitz constants $G_\wbf$ and $G_\vbf$ (in Theorem \ref{thm:moments-accountant-privacy}), we first run the algorithms without adding noise. Then we calculate the maximum gradient norms of AUC loss w.r.t $\wbf$ and $\vbf$ and assign them as $G_\wbf$ and $G_\vbf$, respectively. According to these parameters, we calculate the noise parameter $\sigma$ by applying autodp\footnote{\url{https://github.com/yuxiangw/autodp}}, which is widely used in the existing works \citep{wang2019subsampled}.

\textbf{Compared Algorithms.} \citet{boob2021optimal} is the only existing paper that considers differential privacy in the convex-concave minimax problem. Therefore, we use their single-loop NSEG algorithm as our baseline method on the AUC optimization under the linear setting.
\subsection{Results}\label{experiments:results}
We report our evaluation and results on the utility and privacy trade-off of the DP-SGDA. Then we follow the experiment design by \citep{abadi2016deep} to study the effect of the parameters - hidden units and batch sizes.

\textbf{General AUC Performance vs Privacy.} The general performance of all algorithms under linear and MLP settings of AUC optimization is shown in Table \ref{tab:general_performance_partial}. Since the standard deviation of the AUC performance is around $[0, 0.1\%]$ and the difference between different algortihms is very small, we only report the average AUC performance. First, without adding noise into gradients, we can find the NSEG method and our DP-SGDA method have similar performance under the linear case. 
Furthermore, we can find the performance of the DP-SGDA with MLP model can outperform linear models on all datasets. This is because non-linear models have better expression power and therefore it can learn more information among features than linear models. Second, by adding noise into the gradients, we can find the AUC performance of all models is decreased on all datasets. However, by increasing the privacy budget $\epsilon$, the AUC performance is increased. The reason is that $\epsilon$ and $\sigma$ have opposite trends according to equation \eqref{eq:sigma-sigma}. The relation between $\epsilon$ and AUC score also verifies our Theorem \ref{thm:sgda-utility} and Theorem \ref{thm:utility-nonconvex}.
Third, to verify our statement in Remark \ref{rem:different-sensitivity},
we compare the $\sigma$ values from NSEG and DP-SGDA on all datasets in Figure \ref{fig:sigma}(a). From the figure, it is clear that the $\sigma$ from NSEG is larger than ours in all $\epsilon$ settings since it is calibrated based on the gradients' sensitivity from both $\wbf$ and $\vbf$. In fact, the sensitivity w.r.t. $\vbf$ is small as it is a one-dimensional variable for AUC maximization. Therefore, NSEG leads to overestimate on the noise addition towards $\vbf$. From Table \ref{tab:general_performance_partial} we observe our DP-SGDA achieves better AUC score than NSEG under the same privacy budget.

\textbf{Different Hidden Units.} 
In DP-SGDA under the MLP setting, the hidden unit is one of the most important factors affecting the model performance. Therefore, we compare the AUC performance with respect to the different hidden units in Figure \ref{fig:sigma}(b). If we provide a small number of hidden units, the model will suffer from poor generalization capability. Using a large number of hidden units will make the model easier to fit the training set. For SGDA (non-private) training, it is often helpful to apply a large number of hidden units, as long as the model does not overfit. In agreement with this intuition, we find the model performance improves with increasing hidden units in Figure \ref{fig:sigma}(b). However, for DP-SGDA training, more hidden units increase the sensitivity of the gradients, which leads to more noise added at each update. Therefore, in contrast to the non-private setting, we find the AUC performance decreases when the number of hidden units increases.   

\textbf{Different Mini-Batch Size.} 
From Theorem \ref{thm:moments-accountant-privacy} and Theorem \ref{thm:sgda-primal-opt}, we find mini-batch size can influence the Gaussian noise variances $\sigma_\wbf^2$ and $\sigma_\vbf^2$ as well as the convergence rate. Selecting the mini-batch size must balance two conflicting objectives.  On one hand, a small mini-batch size may lead to sub-optimal performance. On the other hand, for large batch sizes, the added noise has a smaller relative effect. Therefore, we show the AUC score for DP-SGDA with different mini-batch sizes in Figure \ref{fig:sigma}(c). The experimental results show that the mini-batch size has a relatively large impact on the AUC performance when the mini-batch size is small. 


\section{Conclusion}\label{sec: conclusion}
In this paper, we have used algorithmic stability to conduct utility analysis of the DP-SGDA algorithm for minimax problems under DP constraints.  For the convex-concave setting, we proved that   DP-SGDA can attain an optimal rate $\mathcal{O}(\frac{1}{\sqrt{n}} + \frac{\sqrt{d\log(1/\delta)}}{n \epsilon})$ in terms of the weak primal-dual population risk while providing $(\epsilon, \delta)$-DP for both smooth and nonsmooth cases. For the nonconvex-strongly-concave case, assuming that the empirical risk satisfies the PL condition we proved the excess primal population risk  of DP-SGDA can achieve a utility bound  $\Ocal\bigl(\frac{1}{n^{1/3}} + \frac{\sqrt{d\log(1/\delta)}}{n^{5/6}\epsilon}\bigr)$.  Experiments on three benchmark datasets  illustrate the effectiveness of DP-SGDA.

For future work, it would be interesting to improve the utility bound for the nonconvex-strongly-convex setting. It also remains unclear to us how to establish the utility bound for DP-SGDA when gradient clipping techniques are enforced at each iteration. Finally, it would also be interesting to evaluate the performance of DP-SGDA on other motivating examples such as GAN, MDP and robust optimization.

\begin{acknowledgements}
The work is supported by SUNY-IBM AI Alliance Research and NSF grants   (IIS-1816227, IIS-2008532, IIS-2103450, IIS-2110546 and DMS-2110836).  The authors would also like to thank Dr. Guzm\'{a}n and Dr. Boob for helpful discussions on differential privacy for minimax problems and for pointing out a gap in the proof of Lemma \ref{lem:sgda-opt-gap} in the Appendix  in an earlier  version of the paper. 
\end{acknowledgements}

\bibliography{yang_122}

\begin{thebibliography}{62}
\providecommand{\natexlab}[1]{#1}
\providecommand{\url}[1]{\texttt{#1}}
\expandafter\ifx\csname urlstyle\endcsname\relax
  \providecommand{\doi}[1]{doi: #1}\else
  \providecommand{\doi}{doi: \begingroup \urlstyle{rm}\Url}\fi

\bibitem[Abadi et~al.(2016)Abadi, Chu, Goodfellow, McMahan, Mironov, Talwar,
  and Zhang]{abadi2016deep}
Martin Abadi, Andy Chu, Ian Goodfellow, H~Brendan McMahan, Ilya Mironov, Kunal
  Talwar, and Li~Zhang.
\newblock Deep learning with differential privacy.
\newblock In \emph{Proceedings of the 2016 ACM SIGSAC conference on computer
  and communications security}, pages 308--318, 2016.

\bibitem[Abowd(2016)]{us-census-bureau-DP}
John~M Abowd.
\newblock The challenge of scientific reproducibility and privacy protection
  for statistical agencies.
\newblock \emph{Census Scientific Advisory Committee}, 2016.

\bibitem[Arjovsky et~al.(2017)Arjovsky, Chintala, and
  Bottou]{arjovsky2017wasserstein}
Martin Arjovsky, Soumith Chintala, and L{\'e}on Bottou.
\newblock Wasserstein generative adversarial networks.
\newblock In \emph{ICML}, pages 214--223. PMLR, 2017.

\bibitem[Audibert and Catoni(2011)]{audibert2011robust}
Jean-Yves Audibert and Olivier Catoni.
\newblock Robust linear least squares regression.
\newblock \emph{The Annals of Statistics}, 39\penalty0 (5):\penalty0
  2766--2794, 2011.

\bibitem[Bassily et~al.(2014)Bassily, Smith, and Thakurta]{bassily2014private}
Raef Bassily, Adam Smith, and Abhradeep Thakurta.
\newblock Private empirical risk minimization: Efficient algorithms and tight
  error bounds.
\newblock In \emph{2014 IEEE 55th Annual Symposium on Foundations of Computer
  Science}, pages 464--473. IEEE, 2014.

\bibitem[Bassily et~al.(2019)Bassily, Feldman, Talwar, and
  Guha~Thakurta]{bassily2019private}
Raef Bassily, Vitaly Feldman, Kunal Talwar, and Abhradeep Guha~Thakurta.
\newblock Private stochastic convex optimization with optimal rates.
\newblock \emph{Advances in Neural Information Processing Systems}, 32, 2019.

\bibitem[Bassily et~al.(2020)Bassily, Feldman, Guzm{\'a}n, and
  Talwar]{bassily2020stability}
Raef Bassily, Vitaly Feldman, Crist{\'o}bal Guzm{\'a}n, and Kunal Talwar.
\newblock Stability of stochastic gradient descent on nonsmooth convex losses.
\newblock \emph{Advances in Neural Information Processing Systems},
  33:\penalty0 4381--4391, 2020.

\bibitem[Beaulieu-Jones et~al.(2019)Beaulieu-Jones, Wu, Williams, Lee,
  Bhavnani, Byrd, and Greene]{beaulieu2019privacy}
Brett~K Beaulieu-Jones, Zhiwei~Steven Wu, Chris Williams, Ran Lee, Sanjeev~P
  Bhavnani, James~Brian Byrd, and Casey~S Greene.
\newblock Privacy-preserving generative deep neural networks support clinical
  data sharing.
\newblock \emph{Circulation: Cardiovascular Quality and Outcomes}, 12\penalty0
  (7):\penalty0 e005122, 2019.

\bibitem[Boob and Guzm{\'a}n(2021)]{boob2021optimal}
Digvijay Boob and Crist{\'o}bal Guzm{\'a}n.
\newblock Optimal algorithms for differentially private stochastic monotone
  variational inequalities and saddle-point problems.
\newblock \emph{arXiv preprint arXiv:2104.02988}, 2021.

\bibitem[Bousquet and Elisseeff(2002)]{bousquet2002stability}
Olivier Bousquet and Andr{\'e} Elisseeff.
\newblock Stability and generalization.
\newblock \emph{JMLR}, 2\penalty0 (Mar):\penalty0 499--526, 2002.

\bibitem[Chang and Lin(2011)]{chang2011libsvm}
Chih-Chung Chang and Chih-Jen Lin.
\newblock Libsvm: a library for support vector machines.
\newblock \emph{TIST}, 2\penalty0 (3):\penalty0 27, 2011.

\bibitem[Charles and Papailiopoulos(2018)]{charles2018stability}
Zachary Charles and Dimitris Papailiopoulos.
\newblock Stability and generalization of learning algorithms that converge to
  global optima.
\newblock In \emph{ICML}, pages 745--754. PMLR, 2018.

\bibitem[Diana et~al.(2021)Diana, Gill, Kearns, Kenthapadi, and
  Roth]{diana2021minimax}
Emily Diana, Wesley Gill, Michael Kearns, Krishnaram Kenthapadi, and Aaron
  Roth.
\newblock Minimax group fairness: Algorithms and experiments.
\newblock In \emph{Proceedings of the 2021 AAAI/ACM Conference on AI, Ethics,
  and Society}, pages 66--76, 2021.

\bibitem[Ding et~al.(2017)Ding, Kulkarni, and Yekhanin]{miscrosoft-DP}
Bolin Ding, Janardhan Kulkarni, and Sergey Yekhanin.
\newblock Collecting telemetry data privately.
\newblock \emph{arXiv preprint arXiv:1712.01524}, 2017.

\bibitem[Dwork et~al.(2006)Dwork, McSherry, Nissim, and
  Smith]{dwork2006calibrating}
Cynthia Dwork, Frank McSherry, Kobbi Nissim, and Adam Smith.
\newblock Calibrating noise to sensitivity in private data analysis.
\newblock In \emph{Theory of cryptography conference}, pages 265--284.
  Springer, 2006.

\bibitem[Dwork et~al.(2014)Dwork, Roth, et~al.]{dwork2014algorithmic}
Cynthia Dwork, Aaron Roth, et~al.
\newblock The algorithmic foundations of differential privacy.
\newblock \emph{Foundations and Trends{\textregistered} in Theoretical Computer
  Science}, 9\penalty0 (3--4):\penalty0 211--407, 2014.

\bibitem[Erlingsson et~al.(2014)Erlingsson, Pihur, and Korolova]{google-DP}
{\'U}lfar Erlingsson, Vasyl Pihur, and Aleksandra Korolova.
\newblock Rappor: Randomized aggregatable privacy-preserving ordinal response.
\newblock In \emph{ACM CCS}, pages 1054--1067, 2014.

\bibitem[Farnia and Ozdaglar(2021)]{farnia2021train}
Farzan Farnia and Asuman Ozdaglar.
\newblock Train simultaneously, generalize better: Stability of gradient-based
  minimax learners.
\newblock In \emph{ICML}, pages 3174--3185. PMLR, 2021.

\bibitem[Feldman et~al.(2020)Feldman, Koren, and Talwar]{feldman2020private}
Vitaly Feldman, Tomer Koren, and Kunal Talwar.
\newblock Private stochastic convex optimization: optimal rates in linear time.
\newblock In \emph{Proceedings of the 52nd Annual ACM SIGACT Symposium on
  Theory of Computing}, pages 439--449, 2020.

\bibitem[Gao et~al.(2013)Gao, Jin, Zhu, and Zhou]{gao2013one}
Wei Gao, Rong Jin, Shenghuo Zhu, and Zhi-Hua Zhou.
\newblock One-pass auc optimization.
\newblock In \emph{ICML}, pages 906--914. PMLR, 2013.

\bibitem[Goodfellow et~al.(2014)Goodfellow, Pouget-Abadie, Mirza, Xu,
  Warde-Farley, Ozair, Courville, and Bengio]{goodfellow2014generative}
Ian Goodfellow, Jean Pouget-Abadie, Mehdi Mirza, Bing Xu, David Warde-Farley,
  Sherjil Ozair, Aaron Courville, and Yoshua Bengio.
\newblock Generative adversarial nets.
\newblock \emph{Advances in neural information processing systems}, 27, 2014.

\bibitem[Hardt et~al.(2016)Hardt, Recht, and Singer]{hardt2016train}
Moritz Hardt, Ben Recht, and Yoram Singer.
\newblock Train faster, generalize better: Stability of stochastic gradient
  descent.
\newblock In \emph{ICML}, pages 1225--1234, 2016.

\bibitem[Jordon et~al.(2018)Jordon, Yoon, and Van Der~Schaar]{jordon2018pate}
James Jordon, Jinsung Yoon, and Mihaela Van Der~Schaar.
\newblock Pate-gan: Generating synthetic data with differential privacy
  guarantees.
\newblock In \emph{ICLR}, 2018.

\bibitem[Karimi et~al.(2016)Karimi, Nutini, and Schmidt]{karimi2016linear}
Hamed Karimi, Julie Nutini, and Mark Schmidt.
\newblock Linear convergence of gradient and proximal-gradient methods under
  the polyak-{\l}ojasiewicz condition.
\newblock In \emph{Joint European Conference on Machine Learning and Knowledge
  Discovery in Databases}, pages 795--811. Springer, 2016.

\bibitem[Kuzborskij and Lampert(2018)]{kuzborskij2018data}
Ilja Kuzborskij and Christoph Lampert.
\newblock Data-dependent stability of stochastic gradient descent.
\newblock In \emph{ICML}, pages 2815--2824. PMLR, 2018.

\bibitem[LeCun et~al.(1998)LeCun, Bottou, Bengio, and
  Haffner]{lecun1998gradient}
Yann LeCun, L{\'e}on Bottou, Yoshua Bengio, and Patrick Haffner.
\newblock Gradient-based learning applied to document recognition.
\newblock \emph{Proceedings of the IEEE}, 86\penalty0 (11):\penalty0
  2278--2324, 1998.

\bibitem[Lei and Ying(2020)]{lei2020fine}
Yunwen Lei and Yiming Ying.
\newblock Fine-grained analysis of stability and generalization for stochastic
  gradient descent.
\newblock In \emph{ICML}, pages 5809--5819. PMLR, 2020.

\bibitem[Lei and Ying(2021)]{lei2021sharper}
Yunwen Lei and Yiming Ying.
\newblock Sharper generalization bounds for learning with gradient-dominated
  objective functions.
\newblock In \emph{ICLR}, 2021.

\bibitem[Lei et~al.(2021)Lei, Yang, Yang, and Ying]{lei2021stability}
Yunwen Lei, Zhenhuan Yang, Tianbao Yang, and Yiming Ying.
\newblock Stability and generalization of stochastic gradient methods for
  minimax problems.
\newblock In \emph{ICML}, 2021.

\bibitem[Li et~al.(2019)Li, Sanjabi, Beirami, and Smith]{li2019fair}
Tian Li, Maziar Sanjabi, Ahmad Beirami, and Virginia Smith.
\newblock Fair resource allocation in federated learning.
\newblock \emph{arXiv preprint arXiv:1905.10497}, 2019.

\bibitem[Lin et~al.(2020)Lin, Jin, and Jordan]{lin2020gradient}
Tianyi Lin, Chi Jin, and Michael Jordan.
\newblock On gradient descent ascent for nonconvex-concave minimax problems.
\newblock In \emph{ICML}, pages 6083--6093. PMLR, 2020.

\bibitem[Liu et~al.(2020)Liu, Yuan, Ying, and Yang]{liu2019stochastic}
Mingrui Liu, Zhuoning Yuan, Yiming Ying, and Tianbao Yang.
\newblock Stochastic auc maximization with deep neural networks.
\newblock In \emph{ICLR}, 2020.

\bibitem[Luo et~al.(2020)Luo, Ye, Huang, and Zhang]{luo2020stochastic}
Luo Luo, Haishan Ye, Zhichao Huang, and Tong Zhang.
\newblock Stochastic recursive gradient descent ascent for stochastic
  nonconvex-strongly-concave minimax problems.
\newblock \emph{Advances in Neural Information Processing Systems},
  33:\penalty0 20566--20577, 2020.

\bibitem[Martinez et~al.(2020)Martinez, Bertran, and
  Sapiro]{martinez2020minimax}
Natalia Martinez, Martin Bertran, and Guillermo Sapiro.
\newblock Minimax pareto fairness: A multi objective perspective.
\newblock In \emph{International Conference on Machine Learning}, pages
  6755--6764. PMLR, 2020.

\bibitem[Mohri et~al.(2019)Mohri, Sivek, and Suresh]{mohri2019agnostic}
Mehryar Mohri, Gary Sivek, and Ananda~Theertha Suresh.
\newblock Agnostic federated learning.
\newblock In \emph{International Conference on Machine Learning}, pages
  4615--4625. PMLR, 2019.

\bibitem[Natole et~al.(2018)Natole, Ying, and Lyu]{Natole2018}
M.~Natole, Y.~Ying, and S.~Lyu.
\newblock Stochastic proximal algorithms for auc maximization.
\newblock In \emph{International Conference on Machine Learning}, pages
  3707--3716, 2018.

\bibitem[Nedi{\'c} and Ozdaglar(2009)]{nedic2009subgradient}
Angelia Nedi{\'c} and Asuman Ozdaglar.
\newblock Subgradient methods for saddle-point problems.
\newblock \emph{Journal of optimization theory and applications}, 142\penalty0
  (1):\penalty0 205--228, 2009.

\bibitem[Nemirovski et~al.(2009)Nemirovski, Juditsky, Lan, and
  Shapiro]{nemirovski2009robust}
Arkadi Nemirovski, Anatoli Juditsky, Guanghui Lan, and Alexander Shapiro.
\newblock Robust stochastic approximation approach to stochastic programming.
\newblock \emph{SIAM Journal on optimization}, 19\penalty0 (4):\penalty0
  1574--1609, 2009.

\bibitem[Nouiehed et~al.(2019)Nouiehed, Sanjabi, Huang, Lee, and
  Razaviyayn]{nouiehed2019solving}
Maher Nouiehed, Maziar Sanjabi, Tianjian Huang, Jason~D Lee, and Meisam
  Razaviyayn.
\newblock Solving a class of non-convex min-max games using iterative first
  order methods.
\newblock \emph{Advances in Neural Information Processing Systems}, 32, 2019.

\bibitem[Polyak(1964)]{polyak1964gradient}
Boris~T Polyak.
\newblock Gradient methods for solving equations and inequalities.
\newblock \emph{USSR Computational Mathematics and Mathematical Physics},
  4\penalty0 (6):\penalty0 17--32, 1964.

\bibitem[Puterman(2014)]{puterman2014markov}
Martin~L Puterman.
\newblock \emph{Markov decision processes: discrete stochastic dynamic
  programming}.
\newblock John Wiley \& Sons, 2014.

\bibitem[Rafique et~al.(2021)Rafique, Liu, Lin, and Yang]{rafique2021weakly}
Hassan Rafique, Mingrui Liu, Qihang Lin, and Tianbao Yang.
\newblock Weakly-convex--concave min--max optimization: provable algorithms and
  applications in machine learning.
\newblock \emph{Optimization Methods and Software}, pages 1--35, 2021.

\bibitem[Sinha et~al.(2017)Sinha, Namkoong, Volpi, and
  Duchi]{sinha2017certifying}
Aman Sinha, Hongseok Namkoong, Riccardo Volpi, and John Duchi.
\newblock Certifying some distributional robustness with principled adversarial
  training.
\newblock \emph{arXiv preprint arXiv:1710.10571}, 2017.

\bibitem[Song et~al.(2013)Song, Chaudhuri, and Sarwate]{song2013stochastic}
Shuang Song, Kamalika Chaudhuri, and Anand~D Sarwate.
\newblock Stochastic gradient descent with differentially private updates.
\newblock In \emph{2013 IEEE Global Conference on Signal and Information
  Processing}, pages 245--248. IEEE, 2013.

\bibitem[Wang et~al.(2019{\natexlab{a}})Wang, Chen, and
  Xu]{wang2019differentially}
Di~Wang, Changyou Chen, and Jinhui Xu.
\newblock Differentially private empirical risk minimization with non-convex
  loss functions.
\newblock In \emph{ICML}, pages 6526--6535. PMLR, 2019{\natexlab{a}}.

\bibitem[Wang et~al.(2020{\natexlab{a}})Wang, Xiao, Devadas, and
  Xu]{wang2020differentially}
Di~Wang, Hanshen Xiao, Srinivas Devadas, and Jinhui Xu.
\newblock On differentially private stochastic convex optimization with
  heavy-tailed data.
\newblock In \emph{International Conference on Machine Learning}, pages
  10081--10091. PMLR, 2020{\natexlab{a}}.

\bibitem[Wang(2017)]{wang2017primal}
Mengdi Wang.
\newblock Primal-dual $\pi $ learning: Sample complexity and sublinear run time
  for ergodic markov decision problems.
\newblock \emph{arXiv preprint arXiv:1710.06100}, 2017.

\bibitem[Wang et~al.(2021{\natexlab{a}})Wang, Lei, Ying, and Zhang]{WLYZ}
Puyu Wang, Yunwen Lei, Yiming Ying, and Hai Zhang.
\newblock Differentially private sgd with non-smooth loss.
\newblock \emph{Applied and Computational Harmonic Analysis (ACHA)},
  2021{\natexlab{a}}.

\bibitem[Wang et~al.(2021{\natexlab{b}})Wang, Yang, Lei, Ying, and
  Zhang]{wang2021differentially}
Puyu Wang, Zhenhuan Yang, Yunwen Lei, Yiming Ying, and Hai Zhang.
\newblock Differentially private empirical risk minimization for auc
  maximization.
\newblock \emph{Neurocomputing}, 461:\penalty0 419--437, 2021{\natexlab{b}}.

\bibitem[Wang et~al.(2020{\natexlab{b}})Wang, Guo, Narasimhan, Cotter, Gupta,
  and Jordan]{wang2020robust}
Serena Wang, Wenshuo Guo, Harikrishna Narasimhan, Andrew Cotter, Maya Gupta,
  and Michael Jordan.
\newblock Robust optimization for fairness with noisy protected groups.
\newblock \emph{Advances in Neural Information Processing Systems},
  33:\penalty0 5190--5203, 2020{\natexlab{b}}.

\bibitem[Wang et~al.(2019{\natexlab{b}})Wang, Balle, and
  Kasiviswanathan]{wang2019subsampled}
Yu-Xiang Wang, Borja Balle, and Shiva~Prasad Kasiviswanathan.
\newblock Subsampled r{\'e}nyi differential privacy and analytical moments
  accountant.
\newblock In \emph{The 22nd International Conference on Artificial Intelligence
  and Statistics}, pages 1226--1235. PMLR, 2019{\natexlab{b}}.

\bibitem[Wu et~al.(2017)Wu, Li, Kumar, Chaudhuri, Jha, and
  Naughton]{wu2017bolt}
Xi~Wu, Fengan Li, Arun Kumar, Kamalika Chaudhuri, Somesh Jha, and Jeffrey
  Naughton.
\newblock Bolt-on differential privacy for scalable stochastic gradient
  descent-based analytics.
\newblock In \emph{Proceedings of the 2017 ACM International Conference on
  Management of Data}, pages 1307--1322, 2017.

\bibitem[Xiao et~al.(2017)Xiao, Rasul, and Vollgraf]{xiao2017fashion}
Han Xiao, Kashif Rasul, and Roland Vollgraf.
\newblock Fashion-mnist: a novel image dataset for benchmarking machine
  learning algorithms.
\newblock \emph{arXiv preprint arXiv:1708.07747}, 2017.

\bibitem[Xie et~al.(2018)Xie, Lin, Wang, Wang, and Zhou]{xie2018differentially}
Liyang Xie, Kaixiang Lin, Shu Wang, Fei Wang, and Jiayu Zhou.
\newblock Differentially private generative adversarial network.
\newblock \emph{arXiv preprint arXiv:1802.06739}, 2018.

\bibitem[Xu et~al.(2009)Xu, Caramanis, and Mannor]{xu2009robustness}
Huan Xu, Constantine Caramanis, and Shie Mannor.
\newblock Robustness and regularization of support vector machines.
\newblock \emph{Journal of machine learning research}, 10\penalty0 (7), 2009.

\bibitem[Yan et~al.(2020)Yan, Xu, Lin, Liu, and Yang]{yan2020optimal}
Yan Yan, Yi~Xu, Qihang Lin, Wei Liu, and Tianbao Yang.
\newblock Optimal epoch stochastic gradient descent ascent methods for min-max
  optimization.
\newblock \emph{Advances in Neural Information Processing Systems},
  33:\penalty0 5789--5800, 2020.

\bibitem[Ying et~al.(2016)Ying, Wen, and Lyu]{ying2016stochastic}
Yiming Ying, Longyin Wen, and Siwei Lyu.
\newblock Stochastic online auc maximization.
\newblock \emph{Advances in neural information processing systems}, 29, 2016.

\bibitem[Zhang et~al.(2021{\natexlab{a}})Zhang, Hong, Wang, and
  Zhang]{zhang2021generalization}
Junyu Zhang, Mingyi Hong, Mengdi Wang, and Shuzhong Zhang.
\newblock Generalization bounds for stochastic saddle point problems.
\newblock In \emph{International Conference on Artificial Intelligence and
  Statistics}, pages 568--576. PMLR, 2021{\natexlab{a}}.

\bibitem[Zhang et~al.(2021{\natexlab{b}})Zhang, Ma, Lou, and
  Xiong]{zhang2021private}
Qiuchen Zhang, Jing Ma, Jian Lou, and Li~Xiong.
\newblock Private stochastic non-convex optimization with improved utility
  rates.
\newblock In \emph{Proceedings of the Thirtieth International Joint Conference
  on Artificial Intelligence, IJCAI-21}, pages 3370--3376, 2021{\natexlab{b}}.

\bibitem[Zhang et~al.(2018)Zhang, Ji, and Wang]{zhang2018differentially}
Xinyang Zhang, Shouling Ji, and Ting Wang.
\newblock Differentially private releasing via deep generative model (technical
  report).
\newblock \emph{arXiv preprint arXiv:1801.01594}, 2018.

\bibitem[Zhao et~al.(2011)Zhao, Hoi, Jin, and Yang]{zhao2011online}
Peilin Zhao, Steven~CH Hoi, Rong Jin, and Tianbao Yang.
\newblock Online auc maximization.
\newblock In \emph{ICML}, 2011.

\bibitem[Zhou et~al.(2020)Zhou, Chen, Hong, Wu, and Banerjee]{zhou2020private}
Yingxue Zhou, Xiangyi Chen, Mingyi Hong, Zhiwei~Steven Wu, and Arindam
  Banerjee.
\newblock Private stochastic non-convex optimization: Adaptive algorithms and
  tighter generalization bounds.
\newblock \emph{arXiv preprint arXiv:2006.13501}, 2020.

\end{thebibliography}

\onecolumn
\appendix

\begin{center}
\textbf{\Large Appendix for "Differentially Private SGDA for Minimax Problems"}
\end{center}

\section{Motivating Examples}\label{sec:motivating-example}
We provide several examples that can be formulated as a stochastic minimax problem. All these examples have corresponding empirical minimax formulations. 

\textbf{AUC Maximization.} Area Under the ROC Curve (AUC) is a widely used measure for binary classification. Optimizing AUC with square loss can be formulated as
\begin{align*}
\min_{\theta \in \Theta} \Ebb_{\zbf,\zbf'}[(1 - h(\theta;\xbf) + h(\theta;\xbf'))^2|y=1, y'=-1]
\end{align*}
where $h: \Theta\times \Rbb^d\rightarrow\Rbb$ is the scoring function for the classifier. It has been shown this problem is equivalent to a minimax problem once auxiliary variables $a, b, \vbf \in \Rbb$ are introduced \citep{ying2016stochastic-supp}.
\begin{align*}
\min_{\theta, a, b}\max_{\vbf}	F(\theta,a,b,c) = \Ebb_\zbf[f(\theta, a, b, \vbf;\zbf)]
\end{align*}
where $f = (1-p)(h(\theta;\xbf) - a)^2\Ibb[y=1] + p(h(\theta;\xbf) - b)^2\Ibb[y=-1] + 2(1+\vbf)(ph(\theta;\xbf)\Ibb[y=-1] - (1 - p)h(\theta;\xbf)\Ibb[y=1])] - p(1-p)\vbf^2$ and $p = \Pbb[y=1]$. Such problem is (non)convex-concave. In particular, \citet{liu2019stochastic-supp} showed that when $h$ is a one hidden layer neural network the objective $f$ satisfies the \PL condition. Differential privacy has been applied to learn private classifier by optimizing AUC \citep{wang2021differentially-supp}. The proposed privacy mechanisms there are objective perturbation and output perturbation. 
 
\textbf{Generative Adversarial Networks (GANs).} GAN is introduced in \citet{goodfellow2014generative-supp} which can be regarded as a game between a generator network $G_\vbf$ and a discriminator network $D_\wbf$. The generator network produces synthetic data from random noise $\xi$, while the
discriminator network discriminates between the true data and the synthetic data.
In particular, a popular variant of GAN named as WGAN \citep{arjovsky2017wasserstein-supp} can be written as a minimax problem
\begin{align*}
\min_{\wbf}\max_{\vbf} \mathbb{E}[f(\wbf,\vbf;\zbf,\xi)] := \mathbb{E}_{\zbf} [D_\wbf(\zbf)] - \mathbb{E}_{\xi} [D_\wbf(G_\vbf(\xi))].
\end{align*}
Recently \citet{sahiner2021hidden-supp}  showed that WGAN with a two-layer discriminator and generator can be expressed as a convex-concave problem. An heuristic differentially private version of RMSProp were employed to train GANs by \citet{xie2018differentially-supp}. Recently differential privacy has successfully applied to private synthetic data generation by GAN framework \citep{jordon2018pate-supp, beaulieu2019privacy-supp}.

\textbf{Markov Decision Process (MDP).}  Let $\Acal$ be a finite action space. For any $a \in \Acal$, $P(a) \in [0, 1]^{n\times n}$ is the state-transition probability matrix and $\rbf(a) \in [0,1]^n$ is the vector of
expected state-transition rewards. In the infinite-horizon average-reward Markov decision problem, one aims to find a stationary policy $\pi$ to make an infinite sequence of actions and optimize the average-per-time-step reward $\bar{v}$. By classical theory of dynamics programming \citep{puterman2014markov-supp}, finding an optimal policy is equivalent as solving the fixed-point Bellman equation
\begin{align*}
\bar{v}^* + h^*_i = \max_{a\in \Acal} \big\{ \sum_{j=1}^n (p_{ij}(a)h^*_i + p_{ij}(a)r_{ij}(a))\big\}, \quad \forall i
\end{align*}
where $\hbf \in \Rbb^n$ is the difference-of-value vector. \citet{wang2017primal-supp} showed that this problem is equivalent to the minimax problem as follow
\begin{align*}
\min_{\hbf \in \Hcal}\max_{\mu \in \Ucal} \mu^\top((P(a) - I)\hbf + \rbf(a))
\end{align*}
where $\Hcal$ and $\Ucal$ are the feasible regions chosen according to the mixing time and stationary distribution. We refer to \citet{zhang2021generalization-supp} for a discussion on the measure of population risk.


\textbf{Robust Optimization and Fairness.} Let $\Dcal_1, \cdots, \Dcal_m$ be $m$ different distributions on some support. The aim is to minimize the worst population risks $L$ parameterized by some  $\wbf$ among multiple scenarios: 
\begin{align*}
\min_{\wbf \in \Wcal} L(\wbf) = \max_{1 \leq i \leq m} \big\{\Ebb_{\zbf_1 \sim \Dcal_1}[\ell(\wbf; \zbf_1)], \cdots, \Ebb_{\zbf_m \sim \Dcal_m}[\ell(\wbf; \zbf_m)]\big\}    
\end{align*}
This problem can be reformulated as a zero-sum game between two players $\wbf$ and $\vbf$ as follow
\begin{align*}
\min_{\wbf \in \Wcal}\max_{\vbf\in \Delta_m} \sum_{i=1}^m v_i \Ebb_{\zbf_i \sim \Dcal_i}[\ell(\wbf; \zbf_i)] = \Ebb\Big[\sum_{i=1}^m v_i \ell(\wbf; \zbf_i)\Big]    
\end{align*}
where $\Delta_m = \bigl\{\vbf \in\Rbb^m:  v_i\geq 0, \sum_{i=1}^m v_i=1\bigr\}$ denotes the $m$-dimensional simplex. Such robust optimization formulation has been recently proposed to address fairness among subgroups \citep{mohri2019agnostic-supp} and federated learning on heterogeneous populations \citep{li2019fair-supp}.

\section{Proofs of Theorem \ref{thm:moments-accountant-privacy} and Remark \ref{rem:choice-of-param}}\label{sec:proof-privacy}

In this section, we prove the privacy guarantee of DP-SGDA based on the privacy-amplification by the subsampling result, which is a direct application of Theorem 1 in \citet{abadi2016deep-supp}. First we introduce some necessary definitions.

\begin{definition}\label{def:sensitivity}
Given a function $g: \Zcal^n \rightarrow \Rbb^d$, we say $g$ has $\Delta(g)$ $\ell_2$-sensitivity if for any neighboring datasets $S, S'$ we have
\begin{align*}
\|g(S) - g(S')\|_2 \leq \Delta(g).    
\end{align*}
\end{definition}

\begin{definition}[\citep{abadi2016deep-supp}]\label{def:moments-accountant}
For an (randomized) algorithm $A$, and neighboring datasets $S, S'$ the $\lambda$-th moment is given as 
\[
\alpha_A(\lambda, S, S') = \log\Ebb_{O\sim A(S)}\Big[\Big(\frac{\Pbb[A(S) = O]}{\Pbb[A(S') = O]}\Big)^\lambda\Big].
\]
The moments accountant is then defined as 
\[
\alpha_A(\lambda) = \sup_{S, S'}\alpha_A(\lambda, S, S').
\]
\end{definition}

\begin{lemma}[\citep{abadi2016deep-supp}]
Consider a sequence of mechanisms $\{A_t\}_{t\in[T]}$ and the composite mechanism $A = (A_1, \cdots, A_T)$. 
\begin{enumerate}
\item[a)] [Composability]\label{lem:moments-accountant-composition} For any $\lambda$,
\[
\alpha_A(\lambda) = \sum_{t=1}^T \alpha_{A_t}(\lambda).
\]
\item[b)][Tail bound]\label{lem:moments-accountant-tail} For any $\epsilon$, the mechanism $A$ is $(\epsilon, \delta)$ differentially private for 
\[
\delta = \min_\lambda \alpha_A(\lambda) - \lambda \epsilon.
\]
\end{enumerate}
\end{lemma}

\begin{lemma}[\citep{abadi2016deep-supp}]\label{lem:moments-accountant-privacy}
Consider a sequence of mechanisms $A_t = g_t(S_t) + \xi_t$ where $\xi \sim \Ncal(0, \sigma^2I)$. Here each function $g_t: \Zcal^m \rightarrow \Rbb^d$ has $\ell_2$-sensitivity of $1$. And each $S_t$ is a subsample of size $m$ obtained by uniform sampling without replacement \footnote{In our case we use uniform sampling on each iteration to construct $I_t$ and therefore $S_t$, as opposed to the Poisson sampling in  \citet{abadi2016deep-supp}. However, one can verify that similar moment estimates lead to our stated result \citep{wang2019subsampled-supp}} from $S$, i.e. $S_t \sim (Unif(S))^m$, Then
\begin{equation*}
\alpha_A(\lambda) \leq  \frac{m^2n\lambda(\lambda + 1)}{n^2(n-m)\sigma^2} + \Ocal(\frac{m^3\lambda^3}{n^3\sigma^3}).
\end{equation*}
\end{lemma}

\begin{theorem}[Theorem \ref{thm:moments-accountant-privacy} restated]
There exist constants $c_1, c_2$ and $c_3$ so that for any $\epsilon < c_1 T/n^2$, Algorithm \ref{alg:dp-sgda} is $(\epsilon, \delta)$-differentially private for any $\delta > 0$ if we choose
\begin{equation*}
\sigma_\wbf \geq \frac{c_2 G_\wbf \sqrt{T\log(1/\delta)}}{n\epsilon} \text{ and } \sigma_\vbf \geq \frac{c_3 G_\vbf \sqrt{T\log(1/\delta)}}{n\epsilon}.
\end{equation*}
\end{theorem}

\begin{proof}
Let $S = \{\zbf_1, \cdots, \zbf_n\}$ and $S' = \{\zbf'_1, \cdots, \zbf'_n\}$ be two neighboring datasets. At iteration $t$, we first focus on $A_t^\wbf =  \frac{1}{m}\sum_{j=1}^m\nabla_\wbf f(\wbf_t, \vbf_t; \zbf_{i_t^j}) + \xi_t$. Since $f(\cdot, \vbf; \zbf)$ is $G_\wbf$-Lipschitz continuous, it implies for any neighboring datasets $S, S'$, 
\begin{align*}
\Big\| \frac{1}{m}\sum_{j=1}^m\nabla_\wbf f(\wbf_t, \vbf_t; \zbf_{i_t^j}) - \frac{1}{m}\sum_{j=1}^m \nabla_\wbf f(\wbf_t, \vbf_t; \zbf'_{i_t^j})\Big\|_2  \leq \frac{2G_\wbf}{m}.  
\end{align*}
Therefore we can define $g_t(S_t) = \frac{1}{2G_\wbf}\sum_{j=1}^m \nabla_\wbf f(\wbf_t, \vbf_t, \zbf_{i_t^j})$ such that $\Delta(g_t) = 1$. By Lemma \ref{lem:moments-accountant-composition} b) and \ref{lem:moments-accountant-privacy}, the log moment of the composite mechanism $A^\wbf = (A_1^\wbf, \cdots, A_T^\wbf)$ can be bounded as follows
\begin{align*}
\alpha_{A^\wbf}(\lambda) \leq \frac{m^2T\lambda^2}{n^2\tilde{\sigma}_\wbf^2}. 
\end{align*}
where $\tilde{\sigma}_\wbf = \sigma_\wbf / 2G_\wbf$. Similarly, since $A_t^\vbf = \nabla_\wbf f(\wbf_t, \vbf_t; \zbf_{i_t}) + \zeta_t$ has $\ell_2$-sensitivity $2G_\vbf/m$, then the log moment of the final output $A = (A_1^\wbf, A_1^\vbf, \cdots, A_T^\wbf, A_T^\vbf)$ can be bounded as follows
\begin{align*}
\alpha_{A}(\lambda) \leq \alpha_{A^\vbf}(\lambda) + \alpha_{A^\wbf}(\lambda) \leq \frac{m^2T\lambda^2}{n^2\tilde{\sigma}_\wbf^2} + \frac{m^2T\lambda^2}{n^2\tilde{\sigma}_\vbf^2}. 
\end{align*} 
By Lemma \ref{lem:moments-accountant-tail} a), to guarantee $A$ to be $(\epsilon, \delta)$-differentially private, it suffices that
\begin{align*}
\frac{\lambda^2 m^2T}{n^2\tilde{\sigma}_\wbf^2} \leq \frac{\lambda \epsilon}{4}, \frac{\lambda^2 m^2T}{n^2\tilde{\sigma}_\vbf^2} \leq \frac{\lambda \epsilon}{4},
\exp(-\frac{\lambda \epsilon}{4}) \leq \delta, \lambda \leq  \tilde{\sigma}_\wbf^2\log(\frac{n}{m\tilde{\sigma}_\wbf}) \text{ and } \lambda \leq  \tilde{\sigma}_\vbf^2\log(\frac{n}{m\tilde{\sigma}_\vbf})  
\end{align*}
It is now easy to verify that when $\epsilon = c_1m^2T/n^2$, we can satisfy all these conditions by setting
\begin{equation*}
\tilde{\sigma}_\wbf \geq \frac{c_2  \sqrt{T\log(1/\delta)}}{n\epsilon} \text{ and } \tilde{\sigma}_\vbf \geq \frac{c_3  \sqrt{T\log(1/\delta)}}{n\epsilon}
\end{equation*}
for some explicit constants $c_1, c_2$ and $c_3$. The proof is complete.
\end{proof}

\begin{proof}[Proof of Remark \ref{rem:choice-of-param}]
Without loss of generality, we consider with only one $\sigma$ in the the proof of Theorem \ref{thm:moments-accountant-privacy}. Then algorithm $A$ is guaranteed to be $(\epsilon, \delta)$-DP if one can find $\lambda > 0$ such that
\begin{align*}
\frac{\lambda^2 m^2T}{n^2\sigma^2} \leq \frac{\lambda \epsilon}{2},\,
\exp(-\frac{\lambda \epsilon}{2}) \leq \delta, \text{ and } \lambda \leq  \sigma^2\log(\frac{n}{m\sigma})
\end{align*}
Given $\delta = \frac{1}{n^2}$, the second inequality can be reformulated as $\lambda \geq \frac{4\log(n)}{\epsilon}$. Therefore by choosing $\sigma^2 = \frac{8m^2 T \log(n)}{n^2\epsilon^2}$, the first inequality becomes $\lambda \leq \frac{4\log(n)}{\epsilon}$, indicating $\lambda = \frac{4\log(n)}{\epsilon}$. It suffices to show such choice of $\lambda$ satisfies the third inequality, which is straightforward by the choice of $m$ and $\epsilon\leq 1$. The proof is complete.
\end{proof}

\section{Proofs for the convex-concave setting in Section \ref{sec:convex}}\label{sec:proof-convex}

Recall that the error decomposition  \eqref{eq:weak-err-decomp} given in Section \ref{sec:convex} that  the weak PD risk can be decomposed as follows: 
\begin{align*}\label{eq:weak-err-decomp-1}
\triangle^w(\bar{\wbf}_T, \bar{\vbf}_T) = \triangle^w(\bar{\wbf}_T, \bar{\vbf}_T) - \triangle^w_S(\bar{\wbf}_T, \bar{\vbf}_T) + \triangle^w_S(\bar{\wbf}_T, \bar{\vbf}_T), \numberthis
\end{align*}
where the term $\triangle^w(\bar{\wbf}_T, \bar{\vbf}_T) - \triangle^w_S(\bar{\wbf}_T, \bar{\vbf}_T)$ is the generalization error and the term $\triangle^w_S(\bar{\wbf}_T, \bar{\vbf}_T)$ is the optimization error.

The proof of Theorem \ref{thm:sgda-utility} involves the estimation of the optimization error and generalization error which are performed in the subsequent subsection, respectively. 

\subsection{Estimation of Optimization Error}\label{sec:cc-opt}

We start by studying the optimization error for Algorithm \ref{alg:dp-sgda}. This is obtained as a direct corollary of \citet{nemirovski2009robust-supp}, with the existence of the Gaussian noise's variance and the mini-batch. Recall that $d = \max\{d_1, d_2\}.$
\begin{lemma}\label{lem:sgda-opt-gap}
Suppose \textbf{(A1)} holds, and $F_S$ is convex-concave. Let the stepsizes $\eta_{\wbf, t} = \eta_{\vbf, t} = \eta$, $t \in [T]$ for some $\eta > 0$. Then Algorithm \ref{alg:dp-sgda} satisfies
\[
\sup_{\vbf \in \Vcal} \Ebb_A[F_S(\bar{\wbf}_T,\vbf)]  - \inf_{\wbf\in \Wcal} \Ebb_A[F_S(\wbf,\bar{\vbf}_T)]\leq \frac{\eta (G_\wbf^2+G_\vbf^2)}{2} + \frac{D_\wbf^2 + D_\vbf^2}{\eta T} + \frac{(D_\wbf G_\wbf + D_\vbf G_\vbf)}{\sqrt{mT}} + \eta d(\sigma_\wbf^2 + \sigma_\vbf^2).
\]
\end{lemma}

\begin{proof}
According to the non-expansiveness of projection and update rule of Algorithm \ref{alg:dp-sgda}, for any $\wbf \in \Wcal$, we have
\begin{align*}
& \|\wbf_{t+1} - \wbf\|_2^2 \leq \Big\|\wbf_t - \wbf - \frac{\eta}{m}\sum_{j=1}^m\nabla_\wbf f(\wbf_t, \vbf_t; \zbf_{i_t^j}) - \eta \xi_t\Big\|_2^2 \\
\leq & \|\wbf_t - \wbf\|_2^2 + 2\eta\Big\langle \wbf - \wbf_t, \frac{1}{m}\sum_{j=1}^m\nabla_\wbf f(\wbf_t, \vbf_t; \zbf_{i_t^j}) + \xi_t \Big\rangle + \eta^2 \Big\|\frac{1}{m}\sum_{j=1}^m\nabla_\wbf f(\wbf_t, \vbf_t; \zbf_{i_t^j})\Big\|_2^2 + \eta^2\|\xi_t\|_2^2\\
& + 2\eta^2\Big\langle \frac{1}{m}\sum_{j=1}^m\nabla_\wbf f(\wbf_t, \vbf_t; \zbf_{i_t^j}), \xi_t\Big\rangle\\
\leq & \|\wbf_t - \wbf\|_2^2 + 2\eta\langle \wbf - \wbf_t, \nabla_\wbf F_S(\wbf_t, \vbf_t)\rangle  + 2\eta\Big\langle \wbf - \wbf_t, \frac{1}{m}\sum_{j=1}^m\nabla_\wbf f(\wbf_t, \vbf_t; \zbf_{i_t^j}) - \nabla_\wbf F_S(\wbf_t, \vbf_t)\Big\rangle\\
& + \eta^2 G_\wbf^2 + \eta^2\|\xi_t\|_2^2 + 2\eta^2\Big\langle \frac{1}{m}\sum_{j=1}^m\nabla_\wbf f(\wbf_t, \vbf_t; \zbf_{i_t^j}), \xi_t\Big\rangle +  2\eta\langle \wbf - \wbf_t, \xi_t\rangle,
\end{align*}
where in the last inequality we have used $f(\cdot, \vbf_t, \zbf_{i_t^j})$ is $G_\wbf$-Lipschitz continuous. According to the convexity of $F_S(\cdot, \vbf_t)$ we know
\begin{align*}
2\eta(F_S(\wbf_t, \vbf_t) \!-\! F_S(\wbf, \vbf_t)) \leq & \|\wbf_t \!-\! \wbf\|_2^2 \!-\! \|\wbf_{t+1} \!-\! \wbf\|_2^2 \!+\! 2\eta\Big\langle \wbf \!-\! \wbf_t, \frac{1}{m}\sum_{j=1}^m\nabla_\wbf f(\wbf_t, \vbf_t; \zbf_{i_t^j}) \!-\! \nabla_\wbf F_S(\wbf_t, \vbf_t)\Big\rangle \\
& + \eta^2 G_\wbf^2 + \eta^2\|\xi_t\|_2^2 + 2\eta^2\Big\langle \frac{1}{m}\sum_{j=1}^m\nabla_\wbf f(\wbf_t, \vbf_t; \zbf_{i_t^j}), \xi_t\Big\rangle +  2\eta\langle \wbf - \wbf_t, \xi_t\rangle.
\end{align*}
Taking a summation of the above inequality from $t=1$ to $T$ we derive
\begin{multline*}\label{eq:opt-before-expectation}
2\eta\sum_{t=1}^T(F_S(\wbf_t, \vbf_t) - F_S(\wbf, \vbf_t)) \leq \|\wbf_1 - \wbf\|_2^2 + 2\eta\sum_{t=1}^T\Big\langle \wbf - \wbf_t, \frac{1}{m}\sum_{j=1}^m\nabla_\wbf f(\wbf_t, \vbf_t; \zbf_{i_t^j}) - \nabla_\wbf F_S(\wbf_t, \vbf_t)\Big\rangle \\
+ T\eta^2 G_\wbf^2 + \eta^2\sum_{t=1}^T\|\xi_t\|_2^2 + 2\eta^2\sum_{t=1}^T\Big\langle \frac{1}{m}\sum_{j=1}^m\nabla_\wbf f(\wbf_t, \vbf_t; \zbf_{i_t^j}), \xi_t\Big\rangle +  2\eta\langle \wbf - \wbf_t, \xi_t\rangle.
\end{multline*}
It then follows from the concavity of $F_S(\wbf, \cdot)$ and Schwartz's inequality that
\begin{multline*}
2\sum_{t=1}^T\eta(F_S(\wbf_t, \vbf_t) - F_S(\wbf, \bar{\vbf}_T)) \leq  2D_\wbf^2 - 2\eta\sum_{t=1}^T\Big\langle\wbf_t, \frac{1}{m}\sum_{j=1}^m\nabla_\wbf f(\wbf_t, \vbf_t; \zbf_{i_t^j}) - \nabla_\wbf F_S(\wbf_t, \vbf_t)\Big\rangle\\
 + 2D_\wbf \eta\Big\|\sum_{t=1}^T(\frac{1}{m}\sum_{j=1}^m\nabla_\wbf f(\wbf_t, \vbf_t; \zbf_{i_t^j}) - \nabla_\wbf F_S(\wbf_t, \vbf_t)\Big\|_2\\
 + T\eta^2 G_\wbf^2 + \eta^2\sum_{t=1}^T\|\xi_t\|_2^2 + 2\eta^2\sum_{t=1}^T\Big\langle \frac{1}{m}\sum_{j=1}^m\nabla_\wbf f(\wbf_t, \vbf_t; \zbf_{i_t^j}), \xi_t\Big\rangle +  2\eta\langle \wbf - \wbf_t, \xi_t\rangle. \numberthis
\end{multline*}
We can take expectations on the randomness of $A$ over both sides of\eqref{eq:opt-before-expectation} and get
\begin{align*}
2\eta\sum_{t=1}^T\Ebb_A[F_S(\wbf_t, \vbf_t) \!-\!F_S(\wbf, \bar{\vbf}_T)] \leq & 2D_\wbf^2  \!+\! 2D_\wbf\eta \Ebb_A\Big[\Big\|\sum_{t=1}^T \frac{1}{m}\sum_{j=1}^m\nabla_\wbf f(\wbf_t, \vbf_t; \zbf_{i_t^j}) \!-\! \nabla_\wbf F_S(\wbf_t, \vbf_t)\Big\|_2\Big]\\
& + T\eta^2 G_\wbf^2 +\eta^2d_1\sigma_\wbf^2, 
\end{align*}
where we used that the variance $\Ebb_A[\|\xi_t\|_2^2] = d_1\sigma_\wbf^2$, the unbiasedness $\Ebb_A[\langle\wbf_t, \frac{1}{m}\sum_{j=1}^m\nabla_\wbf f(\wbf_t, \vbf_t; \zbf_{i_t^j}) - \nabla_\wbf F_S(\wbf_t, \vbf_t)\rangle] = 0$, the independence  $\Ebb_A[\langle \frac{1}{m}\sum_{j=1}^m\nabla_\wbf f(\wbf_t, \vbf_t; \zbf_{i_t^j}), \xi_t\rangle] = 0$ and $\Ebb_A[\langle \wbf - \wbf_t, \xi_t\rangle]=0$.
Since the above inequality holds for all $\wbf$, we further get
\begin{align*}\label{eq:opt-before-variance}
2\eta\sum_{t=1}^T\Ebb_A[F_S(\wbf_t, \vbf_t)] \!-\! \inf_{\wbf \in \Wcal}\Ebb_A[F_S(\wbf, \bar{\vbf}_T)] \leq & 2D_\wbf^2  \!+\! 2D_\wbf\eta \Ebb_A\Big[\Big\|\sum_{t=1}^T \frac{1}{m}\sum_{j=1}^m\nabla_\wbf f(\wbf_t, \vbf_t; \zbf_{i_t^j}) \!-\! \nabla_\wbf F_S(\wbf_t, \vbf_t)\Big\|_2\Big]\\
& + T\eta^2 G_\wbf^2 +\eta^2d_1\sigma_\wbf^2, \numberthis
\end{align*}
According to Jensen's inequality and $G_\wbf$-Lipschitz continuity we further derive
\begin{align*}
& \Big(\Ebb_A\Big[\Big\|\sum_{t=1}^T (\frac{1}{m}\sum_{j=1}^m\nabla_\wbf f(\wbf_t, \vbf_t; \zbf_{i_t^j}) \!-\! \nabla_\wbf F_S(\wbf_t, \vbf_t)\Big\|_2)\Big]\Big)^2\\
\leq &  \Ebb_A\Big[\Big\|\sum_{t=1}^T (\frac{1}{m}\sum_{j=1}^m\nabla_\wbf f(\wbf_t, \vbf_t; \zbf_{i_t^j}) \!-\! \nabla_\wbf F_S(\wbf_t, \vbf_t))\Big\|_2^2\Big] = \sum_{t=1}^T \Ebb_A\Big[\Big\|\frac{1}{m}\sum_{j=1}^m\nabla_\wbf f(\wbf_t, \vbf_t; \zbf_{i_t^j}) - \nabla_\wbf F_S(\wbf_t, \vbf_t)\Big\|_2^2\Big]\\
\leq & \frac{TG_\wbf^2}{m}.
\end{align*}
Plugging the above estimate into \eqref{eq:opt-before-variance} we arrive
\[
2\eta\sum_{t=1}^T\Ebb_A[F_S(\wbf_t, \vbf_t)] - \inf_{\wbf\in \Wcal}\Ebb_A[F_S(\wbf, \bar{\vbf}_T)] \leq 2D_\wbf^2  + \frac{2D_\wbf \eta G_\wbf\sqrt{T}}{\sqrt{m}} + T\eta^2 G_\wbf^2 + T\eta^2d_1\sigma_\wbf^2.
\]
By dividing $2\eta T$ on both sides we have
\begin{equation}\label{eq:opt-w}
\frac{1}{T}\sum_{t=1}^T\Ebb_A[F_S(\wbf_t, \vbf_t)] - \inf_{\wbf\in \Wcal}\Ebb_A[F_S(\wbf, \bar{\vbf}_T)] \leq \frac{D_\wbf^2}{\eta T}  + \frac{D_\wbf G_{\wbf}}{\sqrt{mT}} + \frac{\eta G^2_{\wbf}}{2} + \frac{\eta d_1\sigma_\wbf^2}{2}.
\end{equation}
In a similar way, we can show that 
\begin{equation}\label{eq:opt-v}
\frac{1}{T}\sum_{t=1}^T\sup_{\vbf \in \Vcal}\Ebb_A[F_S(\bar{\wbf}_T, \vbf)] - \Ebb_A[F_S(\wbf_t, \vbf_t)] \leq \frac{D_\vbf^2}{\eta T}  + \frac{D_\vbf G_\vbf}{\sqrt{mT}} + \frac{\eta G_\vbf^2}{2} + \frac{\eta d_2\sigma_\vbf^2}{2}.
\end{equation}
The stated bound then follows from \eqref{eq:opt-w} and \eqref{eq:opt-v} and the fact that $d = \max\{d_1, d_2\}.$
\end{proof}

\subsection{Estimation of Generalization Error}\label{sec:cc-gen}
Next we move on to the generalization error. Firstly, we introduce a lemma that bridges the generalization and the stability. We say the randomized algorithm $A$ is  {\em $\varepsilon$-weakly-stable} if,  for any neighboring datasets $S, S'$,  there holds 
\begin{align*}
\sup_\zbf\Big(\sup_{\vbf \in \Vcal}\Ebb_{A}[f(A_\wbf(S), \vbf; \zbf) - f(A_\wbf(S'), \vbf; \zbf)] + \sup_{\wbf \in \Wcal}\Ebb_{A}[f(\wbf, A_\vbf(S); \zbf) - f(\wbf, A_\vbf(S'); \zbf)]\Big) \leq \varepsilon.     
\end{align*}

\begin{lemma}{\citep{lei2021stability-supp}}\label{lem:weak-gen-via-weak-stab}
If $A$ is $\varepsilon$-weakly-stable, then there holds $$
    \triangle^w(A_{\wbf}(S),A_{\vbf}(S))-\triangle^w_S(A_{\wbf}(S),A_{\vbf}(S))\leq\varepsilon.$$
\end{lemma}

We also need the following standard lemma before we prove the stability of DP-SGDA.
\begin{lemma}[\citep{rockafellar1976monotone-supp}]\label{lem:monotone}
Let $f$ be a convex-concave function. Then
\begin{equation*}
\left\langle \begin{pmatrix} \wbf - \wbf'\\ \vbf - \vbf' \end{pmatrix},  \begin{pmatrix} \nabla_\wbf f(\wbf, \vbf) - \nabla_\wbf f(\wbf', \vbf')\\ \nabla_\vbf f(\wbf', \vbf') - \nabla_\vbf f(\wbf, \vbf) \end{pmatrix}\right\rangle \geq 0.
\end{equation*}
\end{lemma}

The stability analysis is given in the following lemma. This lemma is an extension of the uniform argument stability results in \citet{lei2021stability-supp} to the case of mini-batch DP-SGDA.

\begin{lemma}\label{lem:sgda-gen-gap}
Suppose the function $F_S$ is convex-concave. Let the stepsizes $\eta_{\wbf, t} = \eta_{\vbf, t} = \eta$ for some $\eta > 0$. 
\begin{enumerate}
\item[a)] Assume \textbf{(A1)} and \textbf{(A3)} hold, then Algorithm \ref{alg:dp-sgda} satisfies
\[
\triangle^w(\bar{\wbf}_T, \bar{\vbf}_T) - \triangle^w_S(\bar{\wbf}_T, \bar{\vbf}_T) \leq \frac{4\sqrt{e(T+T^2/n)}(G_\wbf+G_\vbf)^2\eta\exp(L^2T\eta^2/2)}{\sqrt{n}}.
\]
\item[b)] Assume \textbf{(A1)} holds, then Algorithm \ref{alg:dp-sgda} satisfies
\[
\triangle^w(\bar{\wbf}_T, \bar{\vbf}_T) - \triangle^w_S(\bar{\wbf}_T, \bar{\vbf}_T) \leq 4\sqrt{2}\eta (G_\wbf + G_\vbf)^2\Big(\sqrt{T} + \frac{T}{n}\Big).
\]
\end{enumerate}
\end{lemma}

\begin{proof}
Without loss of generality, let $S=\{\zbf_1,\cdots,\zbf_n\},S'=\{\zbf_1',\cdots,\zbf_n'\}$ be neighboring datasets differing by the last element, i.e. $\zbf_n \neq \zbf'_n$. Let $\{\wbf_t,\vbf_t\},\{\wbf_t',\vbf_t'\}$ be the sequence produced by Algorithm \ref{alg:dp-sgda} w.r.t. $S$ and $S'$, respectively. We first prove Part a). In the case $n\not\in I_t$, by the non-expansiveness of projection, we have
\begin{align*}
  &\left\|\begin{pmatrix}
           \wbf_{t+1}-\wbf_{t+1}' \\
           \vbf_{t+1}-\vbf_{t+1}'
         \end{pmatrix}\right\|_2^2 \leq \left\|\begin{pmatrix}                 \wbf_t-\frac{\eta}{m}\sum_{j=1}^m\nabla_{\wbf}f(\wbf_t,\vbf_t;z_{i_t^j}) - \eta \xi_t -\wbf_t'+\frac{\eta}{m}\sum_{j=1}^m\nabla_{\wbf}f(\wbf_t',\vbf'_t;z_{i_t^j}) + \eta \xi_t \\  
                   \vbf_t+\frac{\eta}{m}\sum_{j=1}^m\nabla_{\vbf}f(\wbf_t,\vbf_t;z_{i_t^j}) + \eta \zeta_t-\vbf_t'-\frac{\eta}{m}\sum_{j=1}^m\nabla_{\vbf}f(\wbf_t',\vbf'_t;z_{i_t^j}) - \eta\zeta_t
                 \end{pmatrix}\right\|_2^2\\
        & =  \left\|\begin{pmatrix}
           \wbf_t-\wbf_t' \\
           \vbf_t-\vbf_t'
         \end{pmatrix}\right\|_2^2 + \frac{\eta}{m}\sum_{j=1}^m \left\langle \begin{pmatrix} \wbf_t - \wbf'_t\\ \vbf_t - \vbf'_t \end{pmatrix},  \begin{pmatrix} \nabla_{\wbf}f(\wbf_t,\vbf_t;z_{i_t^j}) - \nabla_{\wbf}f(\wbf_t',\vbf'_t;z_{i_t^j})\\ \nabla_{\vbf}f(\wbf'_t,\vbf'_t;z_{i_t^j}) - \nabla_{\vbf}f(\wbf_t,\vbf_t;z_{i_t^j}) \end{pmatrix}\right\rangle\\
         & + \left\|\begin{pmatrix}
                   \frac{\eta}{m}\sum_{j=1}^m(\nabla_{\wbf}f(\wbf_t,\vbf_t;z_n)-\nabla_{\wbf}f(\wbf_t',\vbf'_t;z'_n)) \\
                   \frac{\eta}{m}\sum_{j=1}^m(\nabla_{\vbf}f(\wbf_t,\vbf_t;z_n)-\nabla_{\vbf}f(\wbf_t',\vbf'_t;z'_n))
                 \end{pmatrix}\right\|_2^2\\
         & \leq 
         (1+L^2\eta^2) \left\|\begin{pmatrix}
           \wbf_{t}-\wbf_{t}' \\
           \vbf_{t}-\vbf_{t}'
         \end{pmatrix}\right\|_2^2,
\end{align*}
where the last inequality follows from Lemma \ref{lem:monotone} and the $L$-smoothness assumption. If $n \in I_t$, then it follows that
\begin{align*}\label{stab-gda-2}
  & \left\|\begin{pmatrix}
           \wbf_{t+1}-\wbf_{t+1}' \\
           \vbf_{t+1}-\vbf_{t+1}'
         \end{pmatrix}\right\|_2^2
         \leq\left\|\begin{pmatrix}                 \wbf_t-\frac{\eta}{m}\sum_{j=1}^m\nabla_{\wbf}f(\wbf_t,\vbf_t;z_{i_t^j}) - \eta \xi_t -\wbf_t'+\frac{\eta}{m}\sum_{j=1}^m\nabla_{\wbf}f(\wbf_t',\vbf'_t;z'_{i_t^j}) + \eta \xi_t \\  
                   \vbf_t+\frac{\eta}{m}\sum_{j=1}^m\nabla_{\vbf}f(\wbf_t,\vbf_t;z_{i_t^j}) + \eta \zeta_t-\vbf_t'-\frac{\eta}{m}\sum_{j=1}^m\nabla_{\vbf}f(\wbf_t',\vbf'_t;z'_{i_t^j}) - \eta\zeta_t
                 \end{pmatrix}\right\|_2^2\\
        & \leq\frac{1}{m}\sum_{i_t^j \in I_t, i_t^j \neq n}\left\|\begin{pmatrix}
                   \wbf_t-\eta\nabla_{\wbf}f(\wbf_t,\vbf_t;z_{i_t^j})-\wbf_t'+\eta\nabla_{\wbf}f(\wbf_t',\vbf'_t;z'_{i_t^j}) \\
                   \vbf_t+\eta\nabla_{\vbf}f(\wbf_t,\vbf_t;z_{i_t^j})-\vbf_t'-\eta\nabla_{\vbf}f(\wbf_t',\vbf'_t;z'_{i_t^j})
                 \end{pmatrix}\right\|_2^2\\
        & + \frac{1}{m}\left\|\begin{pmatrix}
                   \wbf_t-\eta\nabla_{\wbf}f(\wbf_t,\vbf_t;z_n)-\wbf_t'+\eta\nabla_{\wbf}f(\wbf_t',\vbf'_t;z'_n) \\
                   \vbf_t+\eta\nabla_{\vbf}f(\wbf_t,\vbf_t;z_n)-\vbf_t'-\eta\nabla_{\vbf}f(\wbf_t',\vbf'_t;z'_n)
                 \end{pmatrix}\right\|_2^2\\
         & \leq\frac{m-1}{m}(1+L^2\eta^2) \left\|\begin{pmatrix}
           \wbf_{t}-\wbf_{t}' \\
           \vbf_{t}-\vbf_{t}'
         \end{pmatrix}\right\|_2^2 + \frac{1+p}{m}\left\|\begin{pmatrix}
           \wbf_{t}-\wbf_{t}' \\
           \vbf_{t}-\vbf_{t}'
         \end{pmatrix}\right\|_2^2
        \\ & +\frac{1+1/p}{m}\eta^2\left\|\begin{pmatrix}
                                 \nabla_{\wbf}f(\wbf_t,\vbf_t;z_n)-\nabla_{\wbf}f(\wbf_t',\vbf_t';z'_n) \\
                                 \nabla_{\vbf}f(\wbf_t,\vbf_t;z_n)-\nabla_{\vbf}f(\wbf_t',\vbf_t';z'_n)
                               \end{pmatrix}\right\|_2^2,\numberthis
\end{align*}
where in the last inequality we used the elementary inequality $(a+b)^2\leq(1+p)a^2+(1+1/p)b^2$ ($p>0$). Since $I_t$ are drawn uniformly at random with replacement, the event $n\not\in I_t$ happens with probability $1-m/n$ and the event $n\in I_t$ happens with probability $m/n$. Therefore,
we know
\begin{align*}
  \Ebb_{i_t}\left[\left\|\begin{pmatrix}
           \wbf_{t+1}-\wbf_{t+1}' \\
           \vbf_{t+1}-\vbf_{t+1}'
         \end{pmatrix}\right\|_2^2\right] & \leq \frac{(n-m)(1+L^2\eta^2)}{n} \left\|\begin{pmatrix}
           \wbf_{t}-\wbf_{t}' \\
           \vbf_{t}-\vbf_{t}'
         \end{pmatrix}\right\|_2^2 + \frac{m(1+L^2\eta^2)}{n}\frac{m-1}{m} \left\|\begin{pmatrix}
           \wbf_{t}-\wbf_{t}' \\
           \vbf_{t}-\vbf_{t}'
         \end{pmatrix}\right\|_2^2\\
         & +  \frac{m}{n}\frac{1+p}{m}\left\|\begin{pmatrix}
           \wbf_{t}-\wbf_{t}' \\
           \vbf_{t}-\vbf_{t}'
         \end{pmatrix}\right\|_2^2+\frac{m}{n}\frac{4(1+1/p)}{m}\eta^2(G_\wbf^2 + G_\vbf^2)\\
         & \leq \Big(1+L^2\eta^2+p/n\Big)\left\|\begin{pmatrix}
           \wbf_{t}-\wbf_{t}' \\
           \vbf_{t}-\vbf_{t}'
         \end{pmatrix}\right\|_2^2+\frac{4(1+1/p)}{n}\eta^2(G_\wbf^2 + G_\vbf^2).
\end{align*}
Applying this inequality recursively, we derive
\[
\Ebb_A\left[\left\|\begin{pmatrix}
           \wbf_{t+1}-\wbf_{t+1}' \\
           \vbf_{t+1}-\vbf_{t+1}'
         \end{pmatrix}\right\|_2^2\right]\leq
         \frac{4(1+1/p)}{n}(G_\wbf^2 + G_\vbf^2)\sum_{k=1}^{t}\eta^2\prod_{j=k+1}^{t}\Big(1+L^2\eta^2+p/n\Big).
\]
By the elementary inequality $1+a\leq\exp(a)$, we further derive
\begin{align*}
\Ebb_A\left[\left\|\begin{pmatrix}
           \wbf_{t+1}-\wbf_{t+1}' \\
           \vbf_{t+1}-\vbf_{t+1}'
         \end{pmatrix}\right\|_2^2\right] & \leq \frac{4(1+1/p)}{n}(G_\wbf^2 + G_\vbf^2)\sum_{k=1}^{t}\eta^2\prod_{j=k+1}^{t}\exp\Big(L^2\eta^2+p/n\Big)\\
         & = \frac{4(1+1/p)}{n}(G_\wbf^2 + G_\vbf^2)\sum_{k=1}^{t}\eta^2\exp\Big(L^2\sum_{j=k+1}^{t}\eta^2+p(t-k)/n\Big)\\
         & \leq \frac{4(1+1/p)}{n}(G_\wbf^2 + G_\vbf^2)\exp\Big(L^2\sum_{j=1}^{t}\eta^2+pt/n\Big)\sum_{k=1}^{t}\eta^2.
\end{align*}
By taking $p=n/t$ we get
\[
\Ebb_A\left[\left\|\begin{pmatrix}
           \wbf_{t+1}-\wbf_{t+1}' \\
           \vbf_{t+1}-\vbf_{t+1}'
         \end{pmatrix}\right\|_2^2\right]
         \leq
         \frac{4e(G_\wbf^2 + G_\vbf^2)(1+t/n)}{n}\exp\Big(L^2\sum_{j=1}^{t}\eta^2\Big)\sum_{k=1}^{t}\eta^2.
\]
Now by the Lipschitz continuity and Jensen's inequality we ave
\begin{align*}
& \sup_\zbf\Big(\sup_{\vbf \in \Vcal}\Ebb_{A}[f(A_\wbf(S), \vbf; \zbf) - f(A_\wbf(S'), \vbf; \zbf)] + \sup_{\wbf \in \Wcal}\Ebb_{A}[f(\wbf, A_\vbf(S); \zbf) - f(\wbf, A_\vbf(S'); \zbf)]\Big)\\
\leq & G_\wbf \Ebb_A[\|\bar{\wbf}_T - \bar{\wbf}'_T\|_2] + G_\vbf \Ebb_A[\|\bar{\vbf}_T - \bar{\vbf}'_T\|_2] \leq \frac{4\sqrt{e(T+T^2/n)}(G_\wbf+G_\vbf)^2\eta\exp(L^2T\eta^2/2)}{\sqrt{n}}.
\end{align*}
According to Lemma \ref{lem:weak-gen-via-weak-stab} we know 
\begin{align*}
\triangle^w(\bar{\wbf}_T, \bar{\vbf}_T) - \triangle^w_S(\bar{\wbf}_T, \bar{\vbf}_T) \leq \frac{4\sqrt{e(T+T^2/n)}(G_\wbf+G_\vbf)^2\eta\exp(L^2T\eta^2/2)}{\sqrt{n}}.
\end{align*}

Next we focus on Part b). We consider two cases at the $t$-th iteration. If $n\not\in I_t$, then analogous to the discussions in \citet{lei2021stability-supp} we can show
\begin{align}
  \left\|\begin{pmatrix}
           \wbf_{t+1}-\wbf_{t+1}' \\
           \vbf_{t+1}-\vbf_{t+1}'
         \end{pmatrix}\right\|_2^2
         &\leq\left\|\begin{pmatrix}
                   \wbf_t-\frac{\eta}{m}\sum_{j=1}^m\nabla_{\wbf}f(\wbf_t,\vbf_t;z_{i_t^j}) - \eta \xi_t-\wbf_t'+\frac{\eta}{m}\sum_{j=1}^m\nabla_{\wbf}f(\wbf_t',\vbf'_t;z_{i_t^j}) + \eta \xi_t \\
                   \vbf_t+\frac{\eta}{m}\sum_{j=1}^m\nabla_{\vbf}f(\wbf_t,\vbf_t;z_{i_t^j})+ \eta \zeta_t -\vbf_t'-\frac{\eta}{m}\sum_{j=1}^m\nabla_{\vbf}f(\wbf_t',\vbf'_t;z_{i_t^j}) - \eta \zeta_t
                 \end{pmatrix}\right\|_2^2\notag \\
         & \leq \left\|\begin{pmatrix}
           \wbf_{t}-\wbf_{t}' \\
           \vbf_{t}-\vbf_{t}'
         \end{pmatrix}\right\|_2^2+4(G_\wbf^2 + G_\vbf^2)\eta^2.\label{stab-gda-1}
\end{align}
Combining the preceding inequality with \eqref{stab-gda-2} and using the probability of $n\not\in I_t$, we derive
\begin{align*}
 & \Ebb_{i_t}\left[\left\|\begin{pmatrix}
           \wbf_{t+1}-\wbf_{t+1}' \\
           \vbf_{t+1}-\vbf_{t+1}'
         \end{pmatrix}\right\|_2^2\right]  \leq \frac{n-1}{n}\left(\left\|\begin{pmatrix}
           \wbf_{t}-\wbf_{t}' \\
           \vbf_{t}-\vbf_{t}'
         \end{pmatrix}\right\|_2^2+4(G_\wbf^2 + G_\vbf^2)\eta^2\right)  \\ & + \frac{1+p}{n}\left\|\begin{pmatrix}
           \wbf_{t}-\wbf_{t}' \\
           \vbf_{t}-\vbf_{t}'
         \end{pmatrix}\right\|_2^2+\frac{4(1+1/p)}{n}(G_\wbf^2 + G_\vbf^2)\eta^2 \\
         & = (1+p/n)\left\|\begin{pmatrix}
           \wbf_{t}-\wbf_{t}' \\
           \vbf_{t}-\vbf_{t}'
         \end{pmatrix}\right\|_2^2+4(G_\wbf^2 + G_\vbf^2)\eta^2(1+1/(np)).
\end{align*}
Applying this inequality recursively implies that 
\begin{align*}
 &  \Ebb_A\left[\left\|\begin{pmatrix}
           \wbf_{t+1}-\wbf_{t+1}' \\
           \vbf_{t+1}-\vbf_{t+1}'
         \end{pmatrix}\right\|_2^2\right]  \leq 4(G_\wbf^2 + G_\vbf^2)\eta^2\big(1+1/(np)\big)\sum_{k=1}^{t}\Big(1+\frac{p}{n}\Big)^{t-k}
   \\ & = 4(G_\wbf^2 + G_\vbf^2)\eta^2\Big(1+\frac{1}{np}\Big)\frac{n}{p}\Big(\Big(1+\frac{p}{n}\Big)^t-1\Big) = 4(G_\wbf^2 + G_\vbf^2)\eta^2\Big(\frac{n}{p}+\frac{1}{p^2}\Big)\Big(\Big(1+\frac{p}{n}\Big)^t-1\Big).
\end{align*}
By taking $p=n/t$ in the above inequality and using $(1+1/t)^t\leq e$, we get
\[
\Ebb_A\left[\left\|\begin{pmatrix}
           \wbf_{t+1}-\wbf_{t+1}' \\
           \vbf_{t+1}-\vbf_{t+1}'
         \end{pmatrix}\right\|_2^2\right]\leq 16(G_\wbf^2 + G_\vbf^2)\eta^2\Big(t+\frac{t^2}{n^2}\Big).
\]
Now by the Lipschitz continuity and Jensen's inequality we ave
\begin{align*}
& \sup_\zbf\Big(\sup_{\vbf \in \Vcal}\Ebb_{A}[f(A_\wbf(S), \vbf; \zbf) - f(A_\wbf(S'), \vbf; \zbf)] + \sup_{\wbf \in \Wcal}\Ebb_{A}[f(\wbf, A_\vbf(S); \zbf) - f(\wbf, A_\vbf(S'); \zbf)]\Big)\\
\leq & G_\wbf \Ebb_A[\|\bar{\wbf}_T - \bar{\wbf}'_T\|_2] + G_\vbf \Ebb_A[\|\bar{\vbf}_T - \bar{\vbf}'_T\|_2] \leq 4\sqrt{2}(G_\wbf + G_\vbf)^2\eta^2\Big(\sqrt{T}+\frac{T}{n}\Big).
\end{align*}
According to Lemma \ref{lem:weak-gen-via-weak-stab} we know 
\begin{align*}
\triangle^w(\bar{\wbf}_T, \bar{\vbf}_T) - \triangle^w_S(\bar{\wbf}_T, \bar{\vbf}_T) \leq 32(G_\wbf + G_\vbf)^2\eta^2\Big(\sqrt{T}+\frac{T}{n}\Big).
\end{align*} 
\end{proof}

\subsection{Proof of Theorem  \ref{thm:sgda-utility}}
Finally we are ready to present the proof of Theorem \ref{thm:sgda-utility}.

\begin{theorem}[Theorem \ref{thm:sgda-utility} restated]
Suppose the function $F_S$ is convex-concave. Let the stepsizes $\eta_{\wbf, t} = \eta_{\vbf, t} = \eta$, $t = [T]$ for some $\eta > 0$. 
\begin{enumerate}
\item[a)] Assume \textbf{(A1)} and \textbf{(A3)} hold. If we choose $T \asymp n$ and $\eta \asymp 1/\Big(\sqrt{L}\max\{\sqrt{n}, \sqrt{d\log(1/\delta)}/\epsilon\}\Big)$, then Algorithm \ref{alg:dp-sgda} satisfies
\[
\triangle^w(\bar{\wbf}_T,\bar{\vbf}_T) = \Ocal\Big(\max\{G_\wbf^2 + G_\vbf^2, (G_\wbf + G_\vbf)^2, D_\wbf^2 + D_\vbf^2, D_\wbf G_\wbf + D_\vbf G_\vbf\} \max\Big\{\frac{1}{\sqrt{n}}, \frac{\sqrt{d\log(1/\delta)}}{n\epsilon}\Big\}\Big).
\]
\item[b)] Assume \textbf{(A1)} holds. If we choose $T \asymp n^2$ and $\eta \asymp 1/\Big(n\max\{\sqrt{n}, \sqrt{d\log(1/\delta)}/\epsilon\}\Big)$, then Algorithm \ref{alg:dp-sgda} satisfies
\[
\triangle^w(\bar{\wbf}_T,\bar{\vbf}_T) = \Ocal\Big(\max\{G_\wbf^2 + G_\vbf^2, (G_\wbf + G_\vbf)^2, D_\wbf^2 + D_\vbf^2, D_\wbf G_\wbf + D_\vbf G_\vbf\}\max\Big\{\frac{1}{\sqrt{n}}, \frac{\sqrt{d\log(1/\delta)}}{n\epsilon}\Big\}\Big).
\]
\end{enumerate}
\end{theorem}

\begin{proof}[Proof of Theorem \ref{thm:sgda-utility}]
We first focus on Part a). According to Part a) of  Lemma \ref{lem:sgda-gen-gap}  we know
\[
\triangle^w(\bar{\wbf}_T, \bar{\vbf}_T) - \triangle^w_S(\bar{\wbf}_T, \bar{\vbf}_T) \leq \frac{4\sqrt{e(T+T^2/n)}(G_\wbf + G_\vbf)^2\eta\exp(L^2T\eta^2/2)}{\sqrt{n}}
\]
and by Lemma \ref{lem:sgda-opt-gap} we know
\[
\triangle^w_S(\bar{\wbf}_T, \bar{\vbf}_T) \leq \frac{\eta (G_\wbf^2 + G_\vbf^2)}{2} + \frac{D_\wbf^2 + D_\vbf^2}{2\eta T} + \frac{D_\wbf G_\wbf + D_\vbf G_\vbf}{\sqrt{mT}} + \eta d(\sigma_\wbf^2 + \sigma_\vbf^2).
\] 	
Combining the above two quantities we have
\begin{align*}\label{eq:risk-before-noise}
\triangle^w(\bar{\wbf}_T, \bar{\vbf}_T) \leq & \frac{4\sqrt{e(T+T^2/n)}(G_\wbf + G_\vbf)^2\eta\exp(L^2T\eta^2/2)}{\sqrt{n}} + \frac{\eta (G_\wbf^2 + G_\vbf^2)}{2} + \frac{D_\wbf^2 + D_\vbf^2}{2\eta T} \\
& + \frac{D_\wbf G_\wbf + D_\vbf G_\vbf}{\sqrt{mT}} + \eta d(\sigma_\wbf^2 + \sigma_\vbf^2). \numberthis
\end{align*}
Furthermore, by Theorem \ref{thm:moments-accountant-privacy}, we know
\[
\sigma_\wbf^2 = \Ocal\Big(\frac{G_\wbf^2T\log(1/\delta)}{n^2\epsilon^2}\Big), \quad \sigma_\vbf^2 = \Ocal\Big(\frac{G_\vbf^2T\log(1/\delta)}{n^2\epsilon^2}\Big).
\]
Plugging it back into \eqref{eq:risk-before-noise} we have
\begin{multline*}
\triangle^w(\bar{\wbf}_T, \bar{\vbf}_T) = \Ocal\Big(\frac{\sqrt{(T+T^2/n)}(G_\wbf + G_\vbf)^2\eta\exp(L^2T\eta^2)}{\sqrt{n}}\\
+ \frac{\eta (G_\wbf^2 + G_\vbf^2)}{2} + \frac{D_\wbf^2 + D_\vbf^2}{2\eta T} + \frac{D_\wbf G_\wbf + D_\vbf G_\vbf}{\sqrt{mT}} + \frac{\eta (G_\wbf^2 + G_\vbf^2) Td\log(1/\delta)}{n^2\epsilon^2}\Big).
\end{multline*}
By picking $T \asymp n$ and $\eta \asymp 1/\Big(L\max\{\sqrt{n}, \sqrt{d\log(1/\delta)}/\epsilon\}\Big)$ we have $\exp(L^2T\eta^2) = \Ocal\Big(\min\{1, \frac{n\epsilon^2}{d\log(1/\delta)}\}\Big) = \Ocal(1)$
and
\[
\triangle^w(\bar{\wbf}_T, \bar{\vbf}_T) = \Ocal\Big(\max\{G_\wbf^2 + G_\vbf^2, (G_\wbf + G_\vbf)^2, D_\wbf^2 + D_\vbf^2, D_\wbf G_\wbf + D_\vbf G_\vbf\} \max\Big\{\frac{1}{\sqrt{n}}, \frac{\sqrt{d\log(1/\delta)}}{n\epsilon}\Big\}\Big).
\]
We now turn to Part b). According to Lemma \ref{lem:sgda-gen-gap} Part b) we know 
\[
\triangle^w(\bar{\wbf}_T, \bar{\vbf}_T) - \triangle^w_S(\bar{\wbf}_T, \bar{\vbf}_T) \leq 4\sqrt{2}\eta (G_\wbf + G_\vbf)^2 \Big(\sqrt{T} + \frac{T}{n}\Big).
\]
Similar to Part a) we have
\[
\triangle^w(\bar{\wbf}_T, \bar{\vbf}_T) \!=\! \Ocal\Big(\eta (G_\wbf + G_\vbf)^2 \Big(\sqrt{T} + \frac{T}{n}\Big) + \frac{\eta (G_\wbf^2 \!+\! G_\vbf^2)}{2} + \frac{D_\wbf^2 \!+\! D_\vbf^2}{2\eta T} + \frac{D_\wbf G_\wbf \!+\! D_\vbf G_\vbf}{\sqrt{mT}} + \frac{\eta (G_\wbf^2 \!+\! G_\vbf^2) Td\log(1/\delta)}{n^2\epsilon^2}\Big).
\]
By picking $T \asymp n^2$ and $\eta \asymp 1/\Big(n\max\{\sqrt{n}, \sqrt{d\log(1/\delta)}/\epsilon\}\Big)$ we have
\[
\triangle^w(\bar{\wbf}_T, \bar{\vbf}_T) = \Ocal\Big(\max\{G_\wbf^2 + G_\vbf^2, (G_\wbf + G_\vbf)^2, D_\wbf^2 + D_\vbf^2, D_\wbf G_\wbf + D_\vbf G_\vbf\}\max\Big\{\frac{1}{\sqrt{n}}, \frac{\sqrt{d\log(1/\delta)}}{n\epsilon}\Big\}\Big).
\]
The proof is complete.
\end{proof}



\section{Proofs for the nonconvex-strongly-concave setting in Section \ref{sec:nonconvex-strongly-concave}}\label{sec:proof-nonconvex}

In this section, we will provide the proofs for the theorems in Section \ref{sec:nonconvex-strongly-concave}. Recall that we define $R^*_\S = \min_{\wbf\in\Wcal} R_\S(\wbf), \text{ and } R^* = \min_{\wbf\in\Wcal} R(\wbf).$ Then, for any $\wbf^* \in \arg\min_\wbf R(\wbf)$ we have the error decomposition:
\begin{align*}
\Ebb[R(\wbf_T) - R^*] = & \Ebb[R(\wbf_T) - R_\S(\wbf_T)] + \Ebb[R_\S(\wbf_T) - R_\S^*] +  \Ebb[R_\S^* - R_\S(\wbf^*)] + \Ebb[R_\S(\wbf^*) - R(\wbf^*)]\\
\leq & \Ebb[R(\wbf_T) - R_\S(\wbf_T)] + \Ebb[R_\S(\wbf^*) - R(\wbf^*)] + \Ebb[R_\S(\wbf_T) - R_\S^*].
\end{align*}
The term $\Ebb[R_\S(\wbf_T) - R_\S^*]$ is the {\em optimization error} which characterizes the discrepancy between the primal empirical risk of an output of Algorithm \ref{alg:dp-sgda} and the least possible one. The term $\Ebb[R(\wbf_T) - R_\S(\wbf_T)]  + \Ebb[R_\S(\wbf^*) - R(\wbf^*)]$ is called the {\em generalization error} which measures the discrepancy  between the primal population risk and the empirical one. The estimations for these two errors are described as follows. 

\subsection{Proof of Theorem \ref{thm:sgda-primal-opt}}\label{sec:sgda-primal-opt}

To prove Theorem \ref{thm:sgda-primal-opt}, i.e., optimization error,  we introduce several necessary lemmas. The first lemma is an application of Danskin's Theorem.

\begin{lemma}[\citep{lin2020gradient-supp}]\label{lem:primal-smoothness}
Assume \textbf{(A3)}  holds and $F_S(\wbf, \cdot)$ is $\rho$-strongly concave. Assume $\Vcal$ is a convex and bounded set. Then the function $R_S(\wbf)$ is $L + L^2/\rho$-smooth and $\nabla R_S(\wbf) = \nabla_\wbf F_S(\wbf, \hat{\vbf}_S(\wbf))$, where $\hat{\vbf}_S(\wbf) = \arg\max_{\vbf \in \Vcal} F_S(\wbf, \vbf)$. And $\hat{\vbf}_S(\wbf)$ is $L/\rho$ Lipschitz continuous. 
\end{lemma}

The second lemma shows that $R_S$ also satisfies the PL condition whenever $F_S$ does. 
\begin{lemma}\label{lem:primal-pl}
Assume \textbf{(A3)}  holds. Assume $F_S(\cdot, \vbf)$ satisfies PL condition with constant $\mu$ and $F_S(\wbf, \cdot)$ is $\rho$-strongly concave. Then the function $R_S(\wbf)$ satisfies the PL condition with $\mu$.
\end{lemma}

\begin{proof}
From Lemma \ref{lem:primal-smoothness}, $\|\nabla R_S(\wbf)\|_2^2 = \|\nabla_\wbf F_S(\wbf, \hat{\vbf}_S(\wbf))\|_2^2$. Since $F_S$ satisfies PL condition with constant $\mu$, we get 
\begin{equation}\label{eq:intermediate-pl}
\|\nabla R_S(\wbf)\|_2^2 \geq 2\mu  \big(F_S(\wbf, \hat{\vbf}_S(\wbf)) - \min_{\wbf' \in \Wcal} F_S(\wbf', \hat{\vbf}_S(\wbf))\big).  
\end{equation}
Also, since $F_S(\wbf', \hat{\vbf}_S(\wbf)) \leq \max_{\vbf \in \Vcal} F_S(\wbf', \vbf)$, we have
\begin{equation}\label{eq:intermediate-mm}
\min_{\wbf' \in \Wcal}F_S(\wbf', \hat{\vbf}_S(\wbf)) \leq \min_{\wbf' \in \Wcal}\max_{\vbf \in \Vcal} F_S(\wbf', \vbf) =   \min_{\wbf' \in \Wcal} R_S(\wbf')
\end{equation}
Combining equation \eqref{eq:intermediate-pl} and \eqref{eq:intermediate-mm}, we have
\begin{equation*}
\|\nabla R_S(\wbf)\|_2^2 \geq 2\mu  \big(R_S(\wbf) - \min_{\wbf' \in \Wcal} R_S(\wbf')\big).  
\end{equation*}
The proof is complete.
\end{proof}

Now we present two key lemmas for the convergence analysis. The next lemma characterizes the descent behavior of $R_S(\wbf_t)$. 

\begin{lemma}\label{lem:primal-gap-coupled}
Assume \textbf{(A2)}  and \textbf{(A3)}  hold. Assume $F_S(\cdot, \vbf)$ satisfies the $\mu$-PL condition and $F_S(\wbf, \cdot)$ is $\rho$-strongly concave. For Algorithm \ref{alg:dp-sgda}, the iterates $\{\wbf_t, \vbf_t\}_{t \in [T]}$ satisfies the following inequality 
\begin{align*}
\Ebb[R_S(\wbf_{t+1}) - R_S^*] \leq & (1 - \mu\eta_{\wbf, t})\Ebb[R_S(\wbf_t) - R_S^*] + \frac{L^2\eta_{\wbf, t}}{2} \Ebb[\|\hat{\vbf}_S(\wbf_t) - \vbf_t\|_2^2]\\
& + \frac{(L+L^2/\rho)\eta_{\wbf, t}^2}{2}(\frac{B_\wbf^2}{m} + d\sigma_\wbf^2).     
\end{align*}
\end{lemma}

\begin{proof}
Because $R_S$ is $L + L^2/\rho$-smooth by Lemma \ref{lem:primal-smoothness}, we have
\begin{align*}
R_S(\wbf_{t+1}) - R_S^* \leq & R_S(\wbf_t) - R_S^* + \langle\nabla R_S(\wbf_t), \wbf_{t+1} - \wbf_t\rangle + \frac{L+L^2/\rho}{2}\|\wbf_{t+1} - \wbf_t\|_2^2 \\
= & R_S(\wbf_t) - R_S^* - \eta_{\wbf, t}  \langle\nabla R_S(\wbf_t), \frac{1}{m}\sum_{j=1}^m\nabla_\wbf f(\wbf_t, \vbf_t; \zbf_{i_t^j}) + \xi_t \rangle\\
& + \frac{(L+L^2/\rho)\eta_{\wbf, t}^2}{2}\|\frac{1}{m}\sum_{j=1}^m\nabla_\wbf f(\wbf_t, \vbf_t; \zbf_{i_t^j}) + \xi_t\|_2^2.
\end{align*}
We denote $\Ebb_t$ as the conditional expectation of given $\wbf_t$ and $\vbf_t$. Taking this conditional expectation of both sides, we get
\begin{align*}
\Ebb_t[R_S(\wbf_{t+1}) - R_S^*] = & R_S(\wbf_{t}) - R_S^* - \eta_{\wbf, t} \langle\nabla R_S(\wbf_t), \nabla_\wbf F_S(\wbf_t, \vbf_t)\rangle\\
& + \frac{(L+L^2/\rho)\eta_{\wbf, t}^2}{2} \|\frac{1}{m}\sum_{j=1}^m\nabla_\wbf f(\wbf_t, \vbf_t; \zbf_{i_t^j}) - \nabla_\wbf F_S(\wbf_t, \vbf_t) + \nabla_\wbf F_S(\wbf_t, \vbf_t) - \xi_t\|_2^2\\
\leq & R_S(\wbf_{t}) - R_S^* - \eta_{\wbf, t} \langle\nabla R_S(\wbf_t), \nabla_\wbf F_S(\wbf_t, \vbf_t)\rangle\\
& + \frac{(L+L^2/\rho)\eta_{\wbf, t}^2}{2} \|\nabla_\wbf F_S(\wbf_t, \vbf_t)\|_2^2 + \frac{(L+L^2/\rho)\eta_{\wbf, t}^2}{2} (\frac{B_\wbf^2}{m} + d\sigma_\wbf^2)\\
\leq & R_S(\wbf_t) - R_S^* - \frac{\eta_{\wbf, t}}{2}  \|\nabla R_S(\wbf_t)\|_2^2  + \frac{\eta_{\wbf, t}}{2}  \|\nabla R_S(\wbf_t) - \nabla_\wbf F_S(\wbf_t, \vbf_t)\|_2^2\\
& + \frac{(L+L^2/\rho)\eta_{\wbf, t}^2}{2}(\frac{B_\wbf^2}{m} + d\sigma_\wbf^2),
\end{align*}
where in first inequality since $\Ebb_t[\|\frac{1}{m}\sum_{j=1}^m\nabla_\wbf f(\wbf_t, \vbf_t; \zbf_{i_t^j}) - \nabla_\wbf F_S(\wbf_t, \vbf_t)\|_2^2] = \frac{1}{m}\sum_{j=1}^m\Ebb_t[\|\nabla_\wbf f(\wbf_t, \vbf_t; \zbf_{i_t^j}) - \nabla_\wbf F_S(\wbf_t, \vbf_t)\|_2^2] \leq \frac{B_\wbf^2}{m}$ and $\Ebb_t[\|\xi_t\|_2^2] = d_1\sigma_\wbf^2 \leq d\sigma_\wbf^2$, and the last inequality we use $\eta_\wbf \leq 1/(L+L^2/\rho)$. Because $R_S$ satisfies PL condition with $\mu$ by Lemma \ref{lem:primal-pl}, we have
\begin{align*}
\Ebb_t[R_S(\wbf_{t+1}) - R_S^*] \leq & (1 - \mu\eta_{\wbf, t})(R_S(\wbf_t) - R_S^*) + \frac{\eta_{\wbf, t}}{2}  \|\nabla R_S(\wbf_t) - \nabla_\wbf F_S(\wbf_t, \vbf_t)\|_2^2\\
& + \frac{(L+L^2/\rho)\eta_{\wbf, t}^2}{2}(\frac{B_\wbf^2}{m} + d\sigma_\wbf^2)\\
\leq &  (1 - \mu\eta_{\wbf, t})(R_S(\wbf_t) - R_S^*) + \frac{L^2\eta_{\wbf, t}}{2}  \|\hat{\vbf}_S(\wbf_t) - \vbf_t\|_2^2 + \frac{(L+L^2/\rho)\eta_{\wbf, t}^2}{2}(\frac{B_\wbf^2}{m} + d\sigma_\wbf^2),
\end{align*}
where the second we use $F_S$ is $L$-smooth. Now taking expectation of both sides yields the claimed bound. The proof is complete.
\end{proof}

The next lemma characterizes the descent behavior of $\vbf_t$.
\begin{lemma}\label{lem:dual-point-coupled}
Assume \textbf{(A2)}  and \textbf{(A3)}  hold. Assume $F_S(\cdot, \vbf)$ satisfies PL condition with constant $\mu$ and $F_S(\wbf, \cdot)$ is $\rho$-strongly concave.  Let $\hat{\vbf}_S(\wbf) = \arg\max_{\vbf \in \Vcal} F_S(\wbf, \vbf)$. For Algorithm \ref{alg:dp-sgda} and any $\epsilon > 0$, the iterates $\{\wbf_t, \vbf_t\}$ satisfies the following inequality
\begin{align*}
\Ebb[\|\vbf_{t+1} \!-\! \hat{\vbf}_S(\wbf_{t+1})\|_2^2] \leq &((1\!+\!\frac{1}{\epsilon})2 L^4/\rho \eta_{\wbf,t}^2 \!+\! (1\!+\!\epsilon)(1 \!-\! \rho\eta_{\vbf, t})) \Ebb[\|\vbf_t \!-\! \hat{\vbf}_S(\wbf_t)\|_2^2]  \!+\! (1\!+\!\frac{1}{\epsilon})\eta_{\wbf,t}^2L^2/\rho^2 (\frac{B_\wbf^2}{m} \!+\! d\sigma_\wbf^2)\\
& + (1+\frac{1}{\epsilon}) 4 L^2/\rho^2(L+L^2/\rho) \eta_{\wbf,t}^2\Ebb[R_S(\wbf_t) - R_S^*] + (1+\epsilon)\eta_{\vbf, t}^2 (\frac{B_\vbf^2}{m} + d\sigma_\vbf^2).  
\end{align*}
\end{lemma}

\begin{proof}
By Young's inequality, we have
\begin{equation*}
\|\vbf_{t+1} - \hat{\vbf}_S(\wbf_{t+1})\|_2^2 \leq (1 + \epsilon) \|\vbf_{t+1} - \hat{\vbf}_S(\wbf_t)\|_2^2 + (1 + \frac{1}{\epsilon}) \|\hat{\vbf}_S(\wbf_t) - \hat{\vbf}_S(\wbf_{t+1})\|_2^2.
\end{equation*}
For the term $\|\hat{\vbf}_S(\wbf_t) - \hat{\vbf}_S(\wbf_{t+1})\|_2^2$, since $\hat{\vbf}_S(\cdot)$ is $L/\rho$-Lipschitz by Lemma \ref{lem:primal-smoothness}, taking conditional expectation, we have
\begin{multline*}
\Ebb_t[\|\hat{\vbf}_S(\wbf_{t+1}) - \hat{\vbf}_S(\wbf_t)\|_2^2] \leq L^2/\rho^2\Ebb_t[\|\wbf_{t+1} - \wbf_t\|_2^2] = L^2/\rho^2 \eta_{\wbf,t}^2\Ebb_t[\|\frac{1}{m}\sum_{j=1}^m\nabla_\wbf f(\wbf_t, \vbf_t; \zbf_{i_t^j}) + \xi_t\|_2^2]\\
\leq L^2/\rho^2  \eta_{\wbf,t}^2\|\nabla_\wbf F_S(\wbf_t, \vbf_t)\|_2^2 + L^2/\rho^2 \eta_{\wbf,t}^2 (\frac{B_\wbf^2}{m} + d\sigma_\wbf^2)\\
\leq 2 L^2/\rho^2  \eta_{\wbf,t}^2\|\nabla R_S(\wbf_t) - \nabla_\wbf F_S(\wbf_t, \vbf_t)\|_2^2 + 2 L^2/\rho^2  \eta_{\wbf,t}^2\|\nabla R_S(\wbf_t) \|_2^2 + L^2/\rho^2 \eta_{\wbf,t}^2 (\frac{B_\wbf^2}{m} + d\sigma_\wbf^2)\\
\leq 2 L^4/\rho^2 \eta_{\wbf,t}^2\|\hat{\vbf}_S(\wbf_t) - \vbf_t\|_2^2 + 2 L^2/\rho^2  \eta_{\wbf,t}^2\|\nabla R_S(\wbf_t) \|_2^2 + L^2/\rho^2 \eta_{\wbf,t}^2(\frac{B_\wbf^2}{m} + d\sigma_\wbf^2),
\end{multline*}
where the last step uses the fact that $F_S$ is $L$-smooth. Because $R_S$ is $L + L^2/\rho$-smooth by Lemma \ref{lem:primal-smoothness} we have
$\frac{1}{2(L+L^2\rho)}\|\nabla R_S(\wbf_t) \|_2^2 \leq R_S(\wbf_t) - R_S^*$.
Therefore
\begin{align*}\label{eq:eq:dual-opt-pt}
\Ebb_t[\|\hat{\vbf}_S(\wbf_{t+1}) - \hat{\vbf}_S(\wbf_t)\|_2^2] \leq &2 L^4/\rho^2 \eta_{\wbf,t}^2\|\hat{\vbf}_S(\wbf_t) - \vbf_t\|_2^2 + 4 L^2/\rho^2 (L+L^2/\rho) \eta_{\wbf,t}^2(R_S(\wbf_t) - R_S(\wbf^*))\\
& + L^2/\rho^2 \eta_{\wbf,t}^2(\frac{B_\wbf^2}{m} + d\sigma_\wbf^2). \numberthis
\end{align*}
For the term $\|\vbf_{t+1} - \hat{\vbf}_S(\wbf_t)\|_2^2$, by the contraction of projection, we have
\begin{multline*}
\Ebb_t[\|\vbf_{t+1} - \hat{\vbf}_S(\wbf_t)\|_2^2] \leq \Ebb_t[\|\vbf_t + \eta_{\vbf,t} (\frac{1}{m}\sum_{j=1}^m\nabla_\vbf f(\wbf_t, \vbf_t; \zbf_{i_t^j}) + \zeta_t) - \hat{\vbf}_S(\wbf_t)\|_2^2]  \\
\leq \|\vbf_t - \hat{\vbf}_S(\wbf_t)\|_2^2 + 2\eta_{\vbf,t}\Ebb_t[\langle\vbf_t - \hat{\vbf}_S(\wbf_t),  \frac{1}{m}\sum_{j=1}^m\nabla_\vbf f(\wbf_t, \vbf_t; \zbf_{i_t^j}) \rangle] + \eta_{\vbf, t}^2 \Ebb_t[\|\frac{1}{m}\sum_{j=1}^m\nabla_\vbf f(\wbf_t, \vbf_t; \zbf_{i_t^j}) + \zeta_t\|_2^2]\\
\leq \|\vbf_t - \hat{\vbf}_S(\wbf_t)\|_2^2 + 2\eta_{\vbf,t}\langle\vbf_t - \hat{\vbf}_S(\wbf_t),  \nabla_\vbf F_S(\wbf_t, \vbf_t)\rangle + \eta_{\vbf, t}^2 \|\nabla_\vbf F_S(\wbf_t, \vbf_t)\|_2^2 + \eta_{\vbf, t}^2 (\frac{B_\vbf^2}{m} + d\sigma_\vbf^2)\\
\leq (1 - \rho\eta_{\vbf, t})\|\vbf_t - \hat{\vbf}_S(\wbf_t)\|_2^2 + 2\eta_{\vbf,t}(F_S(\wbf_t, \vbf_t) - F_S(\wbf_t, \hat{\vbf}_S(\wbf_t)) + \eta_{\vbf, t}^2 \|\nabla_\vbf F_S(\wbf_t, \vbf_t)\|_2^2 + \eta_{\vbf, t}^2 (\frac{B_\vbf^2}{m} + d\sigma_\vbf^2),
\end{multline*}
where the third inequality we use the $F_S(\wbf, \cdot)$ is $\rho$-strongly concave. Since $F_S$ is $L$-smooth, by choosing $\eta_{\vbf, t} \leq 1/L$, we have
\begin{align*}\label{eq:dual-pt-conv}
\Ebb_t[\|\vbf_{t+1} \!-\! \hat{\vbf}_S(\wbf_t)\|_2^2] \leq & (1 \!-\! \rho\eta_{\vbf, t})\|\vbf_t \!-\! \hat{\vbf}_S(\wbf_t)\|_2^2 \!-\! \frac{\eta_{\vbf,t}}{L}\|\nabla_\vbf F_S(\wbf_t, \vbf_t)\|_2^2 \!+\! \eta_{\vbf, t}^2 \|\nabla_\vbf F_S(\wbf_t, \vbf_t)\|_2^2 \!+\! \eta_{\vbf, t}^2 (\frac{B_\vbf^2}{m} \!+\! d\sigma_\vbf^2)\\
\leq & (1 - \rho\eta_{\vbf, t})\|\vbf_t - \hat{\vbf}_S(\wbf_t)\|_2^2 + \eta_{\vbf, t}^2(\frac{B_\vbf^2}{m} + d\sigma_\vbf^2). \numberthis
\end{align*}
Combining \eqref{eq:dual-pt-conv} and \eqref{eq:eq:dual-opt-pt} we have
\begin{align*}
\Ebb_t[\|\vbf_{t+1} \!-\! \hat{\vbf}_S(\wbf_{t+1})\|_2^2] \leq &((1\!+\!\frac{1}{\epsilon})2 L^4/\rho^2 \eta_{\wbf,t}^2 \!+\! (1\!+\!\epsilon)(1 \!-\! \rho\eta_{\vbf, t})) \|\vbf_t \!-\! \hat{\vbf}_S(\wbf_t)\|_2^2 \!+\! (1\!+\!\frac{1}{\epsilon})\eta_{\wbf,t}^2L^2/\rho^2  (\frac{B_\wbf^2}{m} \!+\! d\sigma_\wbf^2)\\
& + (1+\frac{1}{\epsilon}) 4 L^2/\rho^2(L+L^2/\rho) \eta_{\wbf,t}^2(R_S(\wbf_t) - R_S(\wbf^*)) + (1+\epsilon)\eta_{\vbf, t}^2 (\frac{B_\vbf^2}{m} + d\sigma_\vbf^2).
\end{align*}
Taking expectation on both sides yields the desired bound. The proof is complete.
\end{proof}

\begin{lemma}\label{lem:coupled-recursive}
Assume \textbf{(A2)}  and \textbf{(A3)}  hold. Assume $F_S(\cdot, \vbf)$ satisfies PL condition with constant $\mu$ and $F_S(\wbf, \cdot)$ is $\rho$-strongly concave. Define $a_t = \Ebb[R_S(\wbf_t) - R_S(\wbf^*)]$ and $b_t = \Ebb[\|\hat{\vbf}_S(\wbf_t) - \vbf_t\|_2^2]$. For Algorithm \ref{alg:dp-sgda}, if $\eta_{\wbf, t} \leq 1/(L+L^2/\rho)$ and $\eta_{\vbf, t} \leq 1/L$, then for any non-increasing sequence $\{\lambda_t>0\}$ and $\epsilon > 0$,  the iterates $\{\wbf_t, \vbf_t\}_{t \in [T]}$ satisfy the following inequality 
\begin{multline*}
a_{t+1} + \lambda_{t+1} b_{t+1} \leq k_{1, t} a_t + k_{2, t} \lambda_t b_t\\
+  \frac{(L+L^2/\rho)\eta_{\wbf, t}^2}{2}(\frac{B_\wbf^2}{m} + d\sigma_\wbf^2) + 2(1 + \frac{1}{\epsilon})\lambda_t L^2/\rho^2\eta_{\wbf,t}^2 (\frac{B_\wbf^2}{m} + d\sigma_\wbf^2) + \lambda_t(1 + \epsilon)\eta_{\vbf, t}^2 (\frac{B_\vbf^2}{m} + d\sigma_\vbf^2),
\end{multline*}
where 
\begin{align*}
k_{1,t} = & (1-\mu \eta_{\wbf, t}) + \lambda_t (1 + \frac{1}{\epsilon})4L^2/\rho^2(L+L^2/\rho)\eta_{\wbf,t}^2,  \\
k_{2,t} = & \frac{L^2\eta_{\wbf, t}}{2\lambda_t} + (1 + \epsilon)(1 - \rho\eta_{\vbf, t}) + (1 + \frac{1}{\epsilon})2 L^4/\rho^2 \eta_{\wbf,t}^2.
\end{align*}
\end{lemma}

\begin{proof}
Combining Lemma \ref{lem:primal-gap-coupled} and Lemma \ref{lem:dual-point-coupled}, we have for any $\lambda_{t+1} > 0$, we have 
\begin{align*}
a_{t+1} + \lambda_{t+1} b_{t+1} \leq & ((1-\mu \eta_{\wbf, t}) + \lambda_{t+1} (1 + \frac{1}{\epsilon})4L^2/\rho^2(L+L^2/\rho)\eta_{\wbf,t}^2 )a_t\\
& + (\frac{L^2\eta_{\wbf, t}}{2} + \lambda_{t+1} (1 + \epsilon)(1 - \rho\eta_{\vbf, t}) + \lambda_{t+1}(1 + \frac{1}{\epsilon})2 L^4/\rho^2 \eta_{\wbf,t}^2) b_t \\
& \!+\! \frac{(L+L^2/\rho)\eta_{\wbf, t}^2}{2}(\frac{B_\wbf^2}{m} \!+\! d\sigma_\wbf^2) \!+\! 2(1 \!+\! \frac{1}{\epsilon})\lambda_{t+1} L^2/\rho^2\eta_{\wbf,t}^2 (\frac{B_\wbf^2}{m} \!+\! d\sigma_\wbf^2) \!+\! \lambda_{t+1}(1 \!+\! \epsilon)\eta_{\vbf, t}^2 (\frac{B_\vbf^2}{m} \!+\! d\sigma_\vbf^2) \\
\leq & ((1-\mu \eta_{\wbf, t}) + \lambda_t (1 + \frac{1}{\epsilon})4L^2/\rho^2(L+L^2/\rho)\eta_{\wbf,t}^2 )a_t\\
& + (\frac{L^2\eta_{\wbf, t}}{2} + \lambda_t (1 + \epsilon)(1 - \rho\eta_{\vbf, t}) + \lambda_t(1 + \frac{1}{\epsilon})2 L^4/\rho^2 \eta_{\wbf,t}^2) b_t \\
& \!+\! \frac{(L+L^2/\rho)\eta_{\wbf, t}^2}{2}(\frac{B_\wbf^2}{m} \!+\! d\sigma_\wbf^2) \!+\! 2(1 \!+\! \frac{1}{\epsilon})\lambda_t L^2/\rho^2\eta_{\wbf,t}^2 (\frac{B_\wbf^2}{m} \!+\! d\sigma_\wbf^2) \!+\! \lambda_t(1 \!+\! \epsilon)\eta_{\vbf, t}^2 (\frac{B_\vbf^2}{m} \!+\! d\sigma_\vbf^2) \\
= & ((1-\mu \eta_{\wbf, t}) + \lambda_t (1 + \frac{1}{\epsilon})4L^2/\rho^2(L+L^2/\rho)\eta_{\wbf,t}^2 )a_t \\
& + \lambda_t (\frac{L^2\eta_{\wbf, t}}{2\lambda_t} + (1 + \epsilon)(1 - \rho\eta_{\vbf, t}) + (1 + \frac{1}{\epsilon})2 L^4/\rho^2 \eta_{\wbf,t}^2) b_t\\
& +\! \frac{(L+L^2/\rho)\eta_{\wbf, t}^2}{2}(\frac{B_\wbf^2}{m} \!+\! d\sigma_\wbf^2) \!+\! 2(1 \!+\! \frac{1}{\epsilon})\lambda_t L^2/\rho^2\eta_{\wbf,t}^2 (\frac{B_\wbf^2}{m} \!+\! d\sigma_\wbf^2) \!+\! \lambda_t(1 \!+\! \epsilon)\eta_{\vbf, t}^2 (\frac{B_\vbf^2}{m} \!+\! d\sigma_\vbf^2).
\end{align*}
where the first inequality we used $\lambda_{t+1} \leq \lambda_t$. The proof is completed.
\end{proof}

We are now ready to state the convergence theorem of Algorithm \ref{alg:dp-sgda}. 

\begin{theorem}[Theorem \ref{thm:sgda-primal-opt} restated]\label{thm:sgda-conv}
Assume \textbf{(A2)}  and \textbf{(A3)}  hold. Assume $F_S(\cdot, \vbf)$ satisfies PL condition with constant $\mu$ and $F_S(\wbf, \cdot)$ is $\rho$-strongly concave. Assume $\mu\leq 2L^2$ and Let $\kappa = \frac{L}{\rho}$. For Algorithm \ref{alg:dp-sgda}, if $\eta_{\wbf, t} = \Ocal(\frac{1}{\mu t})$ and $\eta_{\vbf, t} = \Ocal(\frac{\kappa^2\max\{1, \sqrt{\kappa/\mu}\}}{\mu t^{2/3}})$, then the iterates $\{\wbf_t, \vbf_t\}_{t \in [T]}$ satisfy the following inequality
\begin{align*}\label{eq:sgda-primal-opt}
\Ebb[R_S(\wbf_{T+1}) - R_S^*] = \Ocal(\min\Big\{\frac{1}{L}, \frac{1}{\mu}\Big\}(\frac{B_{\wbf}^2/m + d\sigma_\wbf^2}{T^{2/3}}) + \max\Big\{1, \sqrt{\frac{L\kappa}{\mu}}\Big\}\frac{L\kappa^3}{\mu^2}(\frac{B_{\vbf}^2/m+d\sigma_\vbf^2}{T^{2/3}})). \numberthis
\end{align*}
Furthermore, if $\sigma_\wbf, \sigma_\vbf$ are given by \eqref{eq:sigma-sigma}, we have
\begin{align*}\label{eq:D82}
& \Ebb[R_S(\wbf_{T+1}) - R_S^*]\\
= & \Ocal(\min\Big\{\frac{1}{L}, \frac{1}{\mu}\Big\}(\frac{B_{\wbf}^2}{mT^{2/3}} + \frac{G_{\wbf}^2d T^{1/3}\log(1/\delta)}{n^2\epsilon^2}) + \max\Big\{1, \sqrt{\frac{L\kappa}{\mu}}\Big\}\frac{L\kappa^3}{\mu^2}(\frac{B_{\vbf}^2}{mT^{2/3}} + \frac{G_{\vbf}^2d T^{1/3}\log(1/\delta)}{n^2\epsilon^2})).\numberthis
\end{align*}
\end{theorem}

\begin{proof}
Since $\eta_{\vbf, t} \leq 1/L$, we can pick $\epsilon = \frac{\rho\eta_{\vbf, t}}{2(1 - \rho\eta_{\vbf, t})}$. Then we have $(1 + \epsilon)(1 - \rho\eta_{\vbf, t})=1 - \frac{ \rho\eta_{\vbf, t}}{2}$ and $1 + \frac{1}{\epsilon} \leq \frac{2}{\rho \eta_{\vbf,t}}$. Therefore Lemma \ref{lem:coupled-recursive} can be simplified as 
\begin{align*}
k_{1, t} \leq & (1-\mu \eta_{\wbf, t}) + \lambda_t\frac{8 L^2/\rho^2(L+L^2/\rho)\eta_{\wbf,t}^2}{\rho \eta_{\vbf,t}},  \\
k_{2, t} \leq & \frac{L^2\eta_{\wbf, t}}{2\lambda_t} + 1 - \frac{\rho\eta_{\vbf, t}}{2} + \frac{4 L^4/\rho^2 \eta_{\wbf,t}^2}{\rho\eta_{\vbf,t}}.
\end{align*}
If we choose $\lambda_t = \frac{4L^2\eta_{\wbf,t}}{\rho\eta_{\vbf,t}}$  and $\eta_{\wbf, t} \leq \min\{\frac{\sqrt{\mu}}{8 \kappa^2\sqrt{L+L^2/\rho}}, \frac{1}{4\sqrt{2}\kappa^2}\}\eta_{\vbf, t}$, then further we have $k_{1, t} \leq 1 - \frac{\mu\eta_{\wbf,t}}{2}$ and $k_{2, t} \leq 1 - \frac{\rho\eta_{\vbf, t}}{4}$. By Lemma \ref{lem:coupled-recursive} we have
\begin{align*}
a_{t+1} + \lambda_{t+1} b_{t+1} \leq & (1 - \min\{\frac{\mu}{2}, L^2\}\eta_{\wbf, t})(a_t + \lambda_t b_t) + \frac{(L+L^2/\rho)\eta_{\wbf, t}^2}{2}(\frac{B_\wbf^2}{m} + d\sigma_\wbf^2)\\     
&  + \frac{16L^4/\rho^3\eta_{\wbf,t}^3}{\rho \eta_{\vbf,t}^2} (\frac{B_\wbf^2}{m} + d\sigma_\wbf^2) + \frac{4L^2(2-\rho\eta_{\vbf,t})\eta_{\wbf,t}\eta_{\vbf, t}}{2\rho(1 - \rho\eta_{\vbf,t})} (\frac{B_\vbf^2}{m} + d\sigma_\vbf^2)\\
\leq & (1 - \frac{\mu\eta_{\wbf, t}}{2})(a_t + \lambda_t b_t) + \frac{(L+L^2/\rho)\eta_{\wbf, t}^2}{2}(\frac{B_\wbf^2}{m} + d\sigma_\wbf^2)\\     
&  + \frac{16L^4/\rho^3\eta_{\wbf,t}^3}{\rho \eta_{\vbf,t}^2} (\frac{B_\wbf^2}{m} + d\sigma_\wbf^2) + \frac{4L^2(2-\rho\eta_{\vbf,t})\eta_{\wbf,t}\eta_{\vbf, t}}{2\rho(1 - \rho\eta_{\vbf,t})} (\frac{B_\vbf^2}{m} + d\sigma_\vbf^2), 
\end{align*}
where we used $\mu \leq 2L^2$. Taking $\eta_{\wbf, t} = \frac{2}{\mu t}$ and $\eta_{\vbf, t} = \max\{8 \kappa^2\sqrt{(L+L^2/\rho)/\mu}, 4\sqrt{2}\kappa^2\}\frac{2}{\mu t^{2/3}}$ and multiplying the preceding inequality with $t$ on both sides,  there holds
\begin{multline*}
t(a_{t+1} + \lambda_{t+1}b_{t+1}) \leq (t-1) (a_t + \lambda_t b_t) +   \frac{2(L+L^2/\rho)}{\mu^2 t}(\frac{B_\wbf^2}{m} + d\sigma_\wbf^2) \\
+ \frac{32L^4/\rho^3\min\{\frac{\sqrt{\mu}}{8 \kappa^2\sqrt{L+L^2/\rho}}, \frac{1}{4\sqrt{2}\kappa^2}\}^2}{\mu\rho t^{2/3}} (\frac{B_\wbf^2}{m} + d\sigma_\wbf^2) + \frac{16L^2\max\{8 \kappa^2\sqrt{(L+L^2/\rho)/\mu}, 4\sqrt{2}\kappa^2\}}{2\mu^2\rho t^{2/3}} (\frac{B_\vbf^2}{m} + d\sigma_\vbf^2).       
\end{multline*}
Applying the preceding inequality inductively from $t=1$ to $T$, we have
\begin{align*}
T(a_{T+1} + \lambda_{T+1}b_{T+1}) \leq & \frac{2(L+L^2/\rho)}{\mu^2}(\frac{B_\wbf^2}{m} + d\sigma_\wbf^2)\log(T) + \frac{32L^4/\rho^3\min\{\frac{\sqrt{\mu}}{8 \kappa^2\sqrt{L+L^2/\rho}}, \frac{1}{4\sqrt{2}\kappa^2}\}^2}{\mu\rho} (\frac{B_\wbf^2}{m} + d\sigma_\wbf^2) T^{1/3}\\
& + \frac{16L^2\max\{8 \kappa^2\sqrt{(L+L^2/\rho)/\mu}, 4\sqrt{2}\kappa^2\}}{2\mu^2\rho} (\frac{B_\vbf^2}{m} + d\sigma_\vbf^2)    T^{1/3}.   
\end{align*}
Consequently, 
\begin{align*}\label{eq:sgda-conv-before-sigma}
\Ebb[R_S(\wbf_{T+1}) - R_S^*] \leq & a_{T+1} + \lambda_{T+1}b_{T+1}\\
\leq & \frac{2(L+L^2/\rho)(B_\wbf^2/m + d\sigma_\wbf^2)}{\mu^2}\frac{\log(T)}{T}
\!+\! \frac{32 (B_\wbf^2/m \!+\! d\sigma_\wbf^2)L^4\!/\!\rho^3\!\min\{\frac{\sqrt{\mu}}{8 \kappa^2\!\sqrt{L\!+\!L^2/\rho}}, \frac{1}{4\sqrt{2}\kappa^2}\}^2}{\mu\rho} \frac{1}{T^{2\!/\!3}}\\
& \!+\! \frac{16 (B_\vbf^2/m \!+\! d\sigma_\vbf^2)L^2\!\max\{8 \kappa^2\!\sqrt{(L\!+\!L^2\!/\!\rho)/\mu}, 4\sqrt{2}\kappa^2\}}{2\mu^2\rho} \frac{1}{T^{2\!/\!3}}. \numberthis
\end{align*}
Therefore, the estimation \eqref{eq:sgda-primal-opt} follows from the fact that  $\kappa= L/\rho.$

The result in Theorem \ref{thm:sgda-primal-opt} follows by observing $\max\Big\{1, \sqrt{\frac{L\kappa}{\mu}}\Big\}\frac{L\kappa^3}{\mu^2} \geq \min\Big\{\frac{1}{L}, \frac{1}{\mu}\Big\}$. Substituting  the values of $\sigma_\wbf, \sigma_\vbf$, i.e.,  $\sigma_\wbf \!=\! \frac{c_2 G_\wbf \sqrt{T\log(\frac{1}{\delta})}}{n\epsilon}$ and $  \sigma_\vbf \!=\! \frac{c_3 G_\vbf \sqrt{T\log(\frac{1}{\delta})}}{n\epsilon}$,  into \eqref{eq:sgda-primal-opt} yields the desired estimation  \eqref{eq:D82}.  
\end{proof}

\subsection{Proof of Theorem \ref{thm:sgda-primal-gen} (Generalization Error)}\label{sec:sgda-primal-gen}

We first focus on to the generalization error $\Ebb[R(\wbf_T) - R_S(\wbf_T)]$. Firstly, we introduce a lemma that bridges the generalization and the uniform argument stability. We modify the lemma so that it satisfies our needs.

\begin{lemma}[\citep{lei2021stability-supp}]\label{lem:stab-gen}
Let $A$ be a randomized algorithm and $\epsilon>0$. If for all neighboring datasets $S, S'$, there holds
\begin{align*}
\Ebb_A[\|A_\wbf(S) - A_\wbf(S')\|_2] \leq \varepsilon.    
\end{align*}
Furthermore, if the function $F(\wbf,\cdot)$ is $\rho$-strongly-concave and Assumptions \ref{ass:lipschitz}, \textbf{(A3)}  hold, then the primal generalization error satisfies
  \begin{align*}
      \Ebb_{S,A}\Big[R(A_{\wbf}(S))-R_S(A_{\wbf}(S))\Big]\leq \big(1+L/\rho\big)G_\wbf\varepsilon.
  \end{align*}
\end{lemma}

The next proposition states the set of saddle points is unique with respect to the variable $\vbf$ when $F_S(\wbf, \cdot)$ is strongly concave.

\begin{proposition}\label{lem:unique-v}
Assume $F_S(\wbf, \cdot)$ is $\rho$-strongly concave with $\rho > 0$. Let $(\hat{\wbf}_S, \hat{\vbf}_S)$ and $(\hat{\wbf}'_S, \hat{\vbf}'_S)$ be two saddle points of $F_S$. Then we have $\hat{\vbf}_S = \hat{\vbf}'_S$.
\end{proposition}

\begin{proof}
Given $\hat{\wbf}_S$, by the strong concavity, we have
\begin{align*}
F_S(\hat{\wbf}_S, \hat{\vbf}_S) \geq  F_S(\hat{\wbf}_S, \hat{\vbf}'_S) + \langle \nabla_\vbf F_S(\hat{\wbf}_S, \hat{\vbf}_S) , \hat{\vbf}_S - \hat{\vbf}'_S\rangle + \frac{\rho}{2}\|\hat{\vbf}_S - \hat{\vbf}'_S\|_2^2 .
\end{align*}
Since $(\hat{\wbf}_S, \hat{\vbf}_S)$ is a saddle point of $F_S$, it implies $\hat{\vbf}_S$ attains maximum of $F_S(\hat{\wbf}_S, \cdot)$. By the first order optimality we know $\langle \nabla_\vbf F_S(\hat{\wbf}_S, \hat{\vbf}_S) , \hat{\vbf}_S - \hat{\vbf}'_S\rangle \geq 0$ and therefore 
\begin{align*}\label{eq:strong-saddle-pt}
F_S(\hat{\wbf}_S, \hat{\vbf}_S) \geq  F_S(\hat{\wbf}_S, \hat{\vbf}'_S) +  \frac{\rho}{2}\|\hat{\vbf}_S - \hat{\vbf}'_S\|_2^2 \geq F_S(\hat{\wbf}'_S, \hat{\vbf}'_S) +  \frac{\rho}{2}\|\hat{\vbf}_S - \hat{\vbf}'_S\|_2^2, \numberthis
\end{align*}
where in the second inequality we used $(\hat{\wbf}'_S, \hat{\vbf}'_S)$ is also a saddle point of $F_S$. Similarly, given $\hat{\wbf}'_S$ we can show
\begin{align}\label{eq:strong-saddle-pt2}
F_S(\hat{\wbf}'_S, \hat{\vbf}'_S) \geq  F_S(\hat{\wbf}_S, \hat{\vbf}_S) +  \frac{\rho}{2}\|\hat{\vbf}_S - \hat{\vbf}'_S\|_2^2.
\end{align}
Adding   \eqref{eq:strong-saddle-pt} and \eqref{eq:strong-saddle-pt2} together implies that $\rho\|\hat{\vbf}_S - \hat{\vbf}'_S\|_2^2 \le 0.$
This implies $\hat{\vbf}_S = \hat{\vbf}'_S$ which  completes the proof.
\end{proof}

Recall that $\pi_S:\Wcal \rightarrow \Wcal$ is the projection onto the set of saddle points $\Omega_S = \{\hat{\wbf}_S: (\hat{\wbf}_S, \hat{\vbf}_S \in \arg\min\max F_S(\wbf, \vbf)\}$. i.e. $\pi_S(\wbf) = \arg\min_{\hat{\wbf}_S \in \Omega_S} \frac{1}{2}\|\wbf - \hat{\wbf}_S\|_2^2$. Proposition \ref{lem:unique-v} makes sure the projection is well-defined. The next lemma shows that PL condition implies quadratic growth (QG) condition. The proof follows straightforward from \citet{karimi2016linear-supp} and we omit it for brevity. 
\begin{lemma}\label{lem:pl-to-qg}
Suppose the function $F_S(\cdot, \vbf)$ satisfies $\mu$-PL condition. Then $F_S$ satisfies the QG condition with respect to $\wbf$ with constant $4\mu$, i.e.
\begin{equation*}
F_S(\wbf, \vbf) -  F_S(\pi_S(\wbf), \vbf) \geq 2\mu\|\wbf - \pi_S(\wbf)\|_2^2, \quad \forall \vbf \in \Vcal
\end{equation*}
\end{lemma}

With the help of Assumption \ref{ass:unique-projection} and the preceding lemmas, we can derive the uniform argument stability.  

\begin{lemma}\label{lem:pl-stability}
Assume \textbf{(A1)}, \textbf{(A3)} and \textbf{(A4)}  hold. Assume $F_S(\cdot, \vbf)$ satisfies PL condition with constant $\mu$ and $F_S(\wbf, \cdot)$ is $\rho$-strongly concave. Let $A$ be a randomized algorithm. If for any $S$, $\Ebb[\|A_\wbf(S) - \pi_S(A_\wbf(S))\|_2] = \Ocal(\varepsilon_A)$, then we have
\begin{equation*}
\Ebb[\|A_\wbf(S) - A_\wbf(S')\|_2] \leq \Ocal(\varepsilon_A) + \frac{1}{n}\sqrt{\frac{G_\wbf^2}{4\mu^2} + \frac{G_\vbf^2}{\rho\mu}}.
\end{equation*}
\end{lemma}

\begin{proof} 
Let $(\pi_S(A_\wbf(S)), \hat{\vbf}_S) \in \arg\min_\wbf\max_\vbf F_S(\wbf,\vbf)$ and $(\pi_{S'}(A_\wbf(S')), \hat{\vbf}_{S'})$ defined in the similar way. By triangle inequality we have
\begin{align*}
\Ebb[\|A_\wbf(S) \!-\! A_\wbf(S')\|_2] \leq & \Ebb[\|A_\wbf(S) \!-\! \pi_S(A_\wbf(S))\|_2] \!+\! \|\pi_S(A_\wbf(S)) \!-\! \pi_{S'}(A_\wbf(S'))\|_2 \!+\! \Ebb[\|A_\wbf(S') \!-\! \pi_{S'}(A_\wbf(S'))\|_2]\\
= & \|\pi_S(A_\wbf(S)) - \pi_{S'}(A_\wbf(S'))\|_2 + \Ocal(\varepsilon_A).
\end{align*}
Since $\pi_S(A_\wbf(S)) \in \arg\min_{\wbf \in \Wcal} F_S(\wbf, \hat{\vbf}_S)$ and by Assumption  \textbf{(A4)} we know that $\pi_S(A_\wbf(S))$ is the closest optimal point of $F_S$ to $\pi_{S'}(A_\wbf(S'))$. And since $\hat{\vbf}_S$ is fixed, by Lemma \ref{lem:pl-to-qg}, we have 
\begin{align*}
2\mu\|\pi_S(A_\wbf(S)) - \pi_{S'}(A_\wbf(S'))\|_2^2 \leq &  F_S(\pi_{S'}(A_\wbf(S')), \hat{\vbf}_S) - F_S(\pi_S(A_\wbf(S)), \hat{\vbf}_S).
\end{align*}
Similarly, we have 
\begin{align*}
2\mu\|\pi_S(A_\wbf(S)) - \pi_{S'}(A_\wbf(S'))\|_2^2 \leq & F_{S'}(\pi_S(A_\wbf(S)), \hat{\vbf}_{S'}) - F_{S'}(\pi_{S'}(A_\wbf(S')), \hat{\vbf}_{S'}).
\end{align*}
Summing up the above two inequalities we have
\begin{align*}\label{eq:qg-two-times}
4\mu\|\pi_S(A_\wbf(S)) - \pi_{S'}(A_\wbf(S'))\|_2^2 \leq & F_S(\pi_{S'}(A_\wbf(S')), \hat{\vbf}_S) - F_S(\pi_S(A_\wbf(S)), \hat{\vbf}_S) \\
& + F_{S'}(\pi_S(A_\wbf(S)), \hat{\vbf}_{S'}) - F_{S'}(\pi_{S'}(A_\wbf(S')), \hat{\vbf}_{S'}). \numberthis
\end{align*}
On the other hand, by the $\rho$-strong concavity of $F_S(\cdot, \vbf)$ and $\hat{\vbf}_S = \arg\max_{\vbf\in \Vcal} F_S(\pi_S(A_\wbf(S)), \vbf)$, we have
\begin{align*}
\frac{\rho}{2}\|\hat{\vbf}_S - \hat{\vbf}_{S'}\|_2^2 \leq & F_S(\pi_S(A_\wbf(S)), \hat{\vbf}_S) - F_S(\pi_S(A_\wbf(S)), \hat{\vbf}_{S'}).
\end{align*}
Similarly, we have
\begin{align*}
\frac{\rho}{2}\|\hat{\vbf}_S - \hat{\vbf}_{S'}\|_2^2 \leq & F_{S'}(\pi_{S'}(A_\wbf(S')), \hat{\vbf}_{S'}) - F_{S'}(\pi_{S'}(A_\wbf(S')), \hat{\vbf}_S).
\end{align*}
Summing up the above two inequalities we have
\begin{align*}\label{eq:sc-two-times}
\rho\|\hat{\vbf}_S - \hat{\vbf}_{S'}\|_2^2 \leq & F_S(\pi_S(A_\wbf(S)), \hat{\vbf}_S) - F_S(\pi_S(A_\wbf(S)), \hat{\vbf}_{S'})\\
& + F_{S'}(\pi_{S'}(A_\wbf(S')), \hat{\vbf}_{S'}) - F_{S'}(\pi_{S'}(A_\wbf(S')), \hat{\vbf}_S). \numberthis
\end{align*}
Summing up \eqref{eq:qg-two-times} and \eqref{eq:sc-two-times} rearranging terms, we have
\begin{align*}
& 4\mu\|\pi_S(A_\wbf(S)) - \pi_{S'}(A_\wbf(S'))\|_2^2 + \rho\|\hat{\vbf}_S - \hat{\vbf}_{S'}\|_2^2\\
\leq & F_S(\pi_{S'}(A_\wbf(S')), \hat{\vbf}_S) - F_{S'}(\pi_{S'}(A_\wbf(S')), \hat{\vbf}_S) + F_{S'}(\pi_S(A_\wbf(S)), \hat{\vbf}_{S'}) - F_S(\pi_S(A_\wbf(S)), \hat{\vbf}_{S'})\\
= & \frac{1}{n} \big(f(\pi_{S'}(A_\wbf(S')), \hat{\vbf}_S; \zbf) - f(\pi_{S'}(A_\wbf(S')), \hat{\vbf}_S; \zbf') + f(\pi_S(A_\wbf(S)), \hat{\vbf}_{S'}; \zbf') - f(\pi_S(A_\wbf(S)), \hat{\vbf}_{S'}; \zbf)\big) \\
\leq & \frac{2G_\wbf}{n} \|\pi_S(A_\wbf(S)) - \pi_{S'}(A_\wbf(S'))\|_2 +  \frac{2G_\vbf}{n} \|\hat{\vbf}_S - \hat{\vbf}_{S'}\|_2 \\
\leq & \frac{1}{n}\sqrt{\frac{G_\wbf^2}{\mu} + \frac{4G_\vbf^2}{\rho}} \times \sqrt{4\mu\|\pi_S(A_\wbf(S)) - \pi_{S'}(A_\wbf(S'))\|_2^2 + \rho\|\hat{\vbf}_S - \hat{\vbf}_{S'}\|_2^2},
\end{align*}
where the second inequality is due to Lipschitz continuity of $f$, the third inequality is due to Cauchy-Schwartz inequality. Therefore
\begin{align*}
2\sqrt{\mu}\|\pi_S(A_\wbf(S)) - \pi_{S'}(A_\wbf(S'))\|_2 \leq \sqrt{4\mu\|\pi_S(A_\wbf(S)) - \pi_{S'}(A_\wbf(S'))\|_2 ^2 + \rho\|\hat{\vbf}_S - \hat{\vbf}_{S'}\|_2^2} \leq \frac{1}{n}\sqrt{\frac{G_\wbf^2}{\mu} + \frac{4G_\vbf^2}{\rho}}.
\end{align*}
The proof is complete.
\end{proof}

We are now ready to present the generalization error of Algorithm \ref{alg:dp-sgda} in terms of $\wbf_T$.

\begin{theorem}\label{thm:sgda-gen}
Assume \textbf{(A1)}, \textbf{(A3)} and \textbf{(A4)}  hold. Assume $F_S(\cdot, \vbf)$ satisfies PL condition with constant $\mu$ and $f(\wbf, \cdot; \zbf)$ is $\rho$-strongly concave. For Algorithm \ref{alg:dp-sgda},  the iterates $\{\wbf_t, \vbf_t\}$ satisfies the following inequality
\begin{align*}
\Ebb[R(\wbf_T) - R_S(\wbf_T)] \leq (1 + \frac{L}{\rho})G_\wbf \Big(\sqrt{\frac{\varepsilon_T}{2\mu}} + \frac{1}{n}\sqrt{\frac{G_\wbf^2}{4\mu^2} + \frac{G_\vbf^2}{\rho\mu}}\Big).
\end{align*}
\end{theorem}

\begin{proof}
Since $R_S$ satisfies $\mu$-PL, by Lemma \ref{lem:pl-to-qg} and Theorem \ref{thm:sgda-conv}, we have 
\begin{align*}
\Ebb[\|\wbf_T - \pi(\wbf_T)\|_2] \leq \sqrt{\Ebb[\|\wbf_T - \pi(\wbf_T)\|_2^2]} \leq \sqrt{\Ebb[\frac{1}{2\mu} (R_S(\wbf_T) - R_S^*)]} \leq \sqrt{\frac{\varepsilon_T}{2\mu}}.
\end{align*}
By Lemma \ref{lem:pl-stability}, we have
\begin{align*}
\Ebb[\|\wbf_T - \wbf'_T\|_2] \leq \sqrt{\frac{\varepsilon_T}{2\mu}} + \frac{1}{n}\sqrt{\frac{G_\wbf^2}{4\mu^2} + \frac{G_\vbf^2}{\rho\mu}}.     
\end{align*}
By Part b) of Lemma \ref{lem:stab-gen}, we have
\begin{align*}
\Ebb[R(\wbf_T) - R_S(\wbf_T)] \leq (1 + \frac{L}{\rho})G_\wbf \Big(\sqrt{\frac{\varepsilon_T}{2\mu}} + \frac{1}{n}\sqrt{\frac{G_\wbf^2}{4\mu^2} + \frac{G_\vbf^2}{\rho\mu}}\Big).
\end{align*}
The proof is complete.
\end{proof}

The next theorem establishes the generalization bound for the empirical maximizer of a strongly concave objective, i.e. $\Ebb[R_S(\wbf^*) - R(\wbf^*)]$. The proof follows from \citet{shalev2009stochastic-supp}.

\begin{theorem}\label{thm:sc-stab-gen}
Assume \textbf{(A1)}  holds. Assume $F_S(\wbf, \cdot)$ is $\rho$-strongly concave. Assume that for any $\wbf$ and $S$, the function $\vbf \mapsto F_S(\wbf,\vbf)$ is $\rho$-strongly-concave. Then
\begin{align*}
\Ebb\big[R_S(\wbf^*) - R(\wbf^*)\big] \leq \frac{4G_\vbf^2}{\rho n}.
\end{align*}
\end{theorem}

\begin{proof}
We decompose the term $\Ebb[R_S(\wbf^*) - R(\wbf^*)]$ as
\[
\Ebb\big[R_S(\wbf^*) - R(\wbf^*)\big] = \Ebb\big[F_S(\wbf^*,\hat{\vbf}^*_S) - F(\wbf^*,\vbf^*)\big] = \Ebb\big[F_S(\wbf^*,\hat{\vbf}^*_S) - F(\wbf^*,\hat{\vbf}^*_S)\big] + \Ebb\big[F(\wbf^*,\hat{\vbf}^*_S) - F(\wbf^*,\vbf^*)\big],
\]
where $\hat{\vbf}^*_S=\arg\max_{\vbf}F_S(\wbf^*,\vbf)$.
The second term $\Ebb\big[F(\wbf^*,\hat{\vbf}^*_S) - F(\wbf^*,\vbf^*)\big] \leq 0$ since $(\wbf^*,\vbf^*)$ is a saddle point of $F$. Hence it suffices to bound $\Ebb\big[F_S(\wbf^*,\hat{\vbf}^*_S) - F(\wbf^*,\hat{\vbf}^*_S)\big]$. Let $S'=\{z'_1,\ldots,z'_n\}$ be drawn independently from $\rho$. For any $i\in[n]$, define $S^{(i)}=\{z_1,\ldots,z_{i-1},z_i',z_{i+1},\ldots,z_n\}$. Denote $\hat{\vbf}^*_{S^{(i)}} = \arg\max_{\vbf\in \Vcal}F_{S^{(i)}}(\wbf^*,\vbf)$. Then
\begin{align*}\label{eq:sc-stab}
F_S(\wbf^*,\hat{\vbf}^*_S) - F_S(\wbf^*,\hat{\vbf}^*_{S^{(i)}}) = & \frac{1}{n}\sum_{j\neq i}\Big(f(\wbf^*,\hat{\vbf}^*_S;z_j) - f(\wbf^*,\hat{\vbf}^*_{S^{(i)}};z_j) \Big)   + \frac{1}{n}\Big(f(\wbf^*,\hat{\vbf}^*_S;z_i) - f(\wbf^*,\hat{\vbf}^*_{S^{(i)}};z_i)\Big)\\
= &  \frac{1}{n}\Big( f(\wbf^*,\hat{\vbf}^*_{S^{(i)}};z'_i)-f(\wbf^*,\hat{\vbf}^*_S;z'_i) \Big) + \frac{1}{n}\Big(f(\wbf^*,\hat{\vbf}^*_S;z_i) - f(\wbf^*,\hat{\vbf}^*_{S^{(i)}};z_i)\Big)\\
& + F_{S^{(i)}}(\wbf^*,\hat{\vbf}^*_S) - F_{S^{(i)}}(\wbf^*,\hat{\vbf}^*_{S^{(i)}})\\
\leq & \frac{1}{n}\Big( f(\wbf^*,\hat{\vbf}^*_{S^{(i)}};z'_i)-f(\wbf^*,\hat{\vbf}^*_S;z'_i)\Big) + \frac{1}{n}\Big(f(\wbf^*,\hat{\vbf}^*_S;z_i) - f(\wbf^*,\hat{\vbf}^*_{S^{(i)}};z_i)\Big)\\
\leq & \frac{2G_\vbf}{n}\big\|\hat{\vbf}^*_S - \hat{\vbf}^*_{S^{(i)}}\big\|_2, \numberthis
\end{align*}
where the first inequality follows from the fact that $\hat{\vbf}^*_{S^{(i)}}$ is the maximizer of $F_{S^{(i)}}(\wbf^*,\cdot)$ and the second inequality follows the Lipschitz continuity. Since $F_S$ is strongly-concave and $\hat{\vbf}^*_S$ maximizes $F_S(\wbf^*,\cdot)$, we know
\begin{align*}
\frac{\rho}{2}\big\|\hat{\vbf}^*_S - \hat{\vbf}^*_{S^{(i)}}\big\|_2^2 \leq F_S(\wbf^*,\hat{\vbf}^*_S) - F_S(\wbf^*,\hat{\vbf}^*_{S^{(i)}}).
\end{align*}
Combining it with \eqref{eq:sc-stab} we get $\big\|\hat{\vbf}^*_S - \hat{\vbf}^*_{S^{(i)}}\big\|_2 \leq 4G_\vbf/(\rho n)$. By Lipschitz continuity, the following inequality holds for any $z$
\begin{align*}
\big|f(\wbf^*,\hat{\vbf}^*_S ;z) - f(\wbf^*,\hat{\vbf}^*_{S^{(i)}};z)\big|  \leq \frac{4G_\vbf^2}{\rho n}.
\end{align*}
Since $z_i$ and $z'_i$ are i.i.d., we have
\begin{align*}
\Ebb\big[F(\wbf^*,\hat{\vbf}^*_S)\big]   =   \Ebb\big[F(\wbf^*,\hat{\vbf}^*_{S^{(i)}})\big] =  \frac{1}{n}\sum_{i=1}^n\Ebb\big[f(\wbf^*,\hat{\vbf}^*_{S^{(i)}};z_i)\big],
\end{align*}
where the last identity holds since $z_i$ is independent of $\hat{\vbf}^*_{S^{(i)}}$.
Therefore
\begin{align*}
\Ebb\big[F_S(\wbf^*,\hat{\vbf}^*_S) - F(\wbf^*,\hat{\vbf}^*_S)\big] = \frac{1}{n}\sum_{i=1}^n \Ebb\big[f(\wbf^*,\hat{\vbf}^*_S;z_i) - f(\wbf^*,\hat{\vbf}^*_{S^{(i)}};z_i)\big] \leq \frac{4G_\vbf^2}{\rho n}.
\end{align*}
The proof is complete.
\end{proof}

\begin{theorem}[Theorem \ref{thm:sgda-primal-gen} restated]
Assume the function $f(\wbf, \cdot; \zbf)$ is $\rho$-strongly concave and $F_S(\cdot, \vbf)$ satisfies $\mu$-PL condition. Suppose \textbf{(A1)}  and \textbf{(A3)}  hold. If $\Ebb[R_S(\wbf_{T+1}) - R_S^*] \leq \varepsilon_T$, then
\begin{align*}
&\Ebb[R(\wbf_T) - R_S(\wbf_T)]  \leq (1+\kappa)G_\wbf\Big(\sqrt{\frac{\varepsilon_T}{2\mu}} + \frac{1}{n}\sqrt{\frac{G_\wbf^2}{4\mu^2} + \frac{G_\vbf^2}{\rho\mu}} \Big), 
\end{align*} and 
\begin{align*}
\Ebb[R_S(\wbf^*) - R(\wbf^*)] \leq   \frac{4G_\vbf^2}{\rho n}.
\end{align*}
 
\end{theorem}

\begin{proof}
It follows directly from Theorem \ref{thm:sgda-gen} and \ref{thm:sc-stab-gen}.
\end{proof}

\subsection{Proof of Theorem \ref{thm:utility-nonconvex}}\label{sec:agda-utility}

\begin{theorem}[Theorem \ref{thm:utility-nonconvex} restated]
Assume \textbf{(A1)}, \textbf{(A3)} and \textbf{(A4)}  hold. Assume $F_S(\cdot, \vbf)$ satisfies PL condition with constant $\mu$ and $f(\wbf, \cdot; \zbf)$ is $\rho$-strongly concave. For SGDA, if $\Ebb[R_S(\wbf_T) - R_S^*] = \Ocal(\varepsilon_T)$, then iterates $\{\wbf_t, \vbf_t\}$ satisfies the following inequality
\begin{equation*}
\Ebb [R(\wbf_T) - R^*] = \Ocal(\varepsilon_T + (1 + \frac{L}{\rho})G_\wbf \Big(\sqrt{\frac{\varepsilon_T}{2\mu}} + \frac{1}{n}\sqrt{\frac{G_\wbf^2}{4\mu^2} + \frac{G_\vbf^2}{\rho\mu}}\Big) +  \frac{4G_\vbf^2}{\rho n}).
\end{equation*}
Furthermore, if we choose $T = \Ocal(n)$, $\eta_{\wbf, t} = \Ocal(\frac{1}{\mu t})$ and $\eta_{\vbf, t} = \Ocal(\frac{\kappa^2\max\{1, \sqrt{\kappa/\mu}\}}{\mu t^{2/3}})$, then
\begin{align*}
\Ebb[R(\wbf_T) - R^*] = \Ocal\bigl(\frac{\kappa^{2.75}}{\mu^{1.75}}(\frac{1}{n^{1/3}} + \frac{\sqrt{d\log(1/\delta)}}{n^{5/6}\epsilon})\bigr).
\end{align*}
\end{theorem}

\begin{proof}
For any $\wbf^* \in \arg\min_\wbf R(\wbf)$, recall that we have the error decomposition \eqref{eq:err-decomp}, which is
\begin{align*}
\Ebb[R(\wbf_T) - R^*] = & \Ebb[R(\wbf_T) - R_S(\wbf_T)] + \Ebb[R_S(\wbf_T) - R_S^*] + \Ebb[R_S^* - R_S(\wbf^*)] + \Ebb[R_S(\wbf^*) - R(\wbf^*)]\\
\leq & \Ebb[R(\wbf_T) - R_S(\wbf_T)] + \Ebb[R_S(\wbf_T) - R_S^*] + \Ebb[R_S(\wbf^*) - R(\wbf^*)],
\end{align*}
where the inequality is by $R_S^* - R_S(\wbf^*) \leq 0$. 
By Theorem \ref{thm:sgda-gen}, we have
\begin{align*}
\Ebb[R(\wbf_T) - R_S(\wbf_T)] \leq (1 + \frac{L}{\rho})G_\wbf \Big(\sqrt{\frac{\varepsilon_T}{2\mu}} + \frac{1}{n}\sqrt{\frac{G_\wbf^2}{4\mu^2} + \frac{G_\vbf^2}{\rho\mu}}\Big).
\end{align*}
And by Theorem \ref{thm:sc-stab-gen}, we have
\begin{align*}
\Ebb[R_S(\wbf^*) - R(\wbf^*)] \leq \frac{4G_\vbf^2}{\rho n}.
\end{align*}
We can plug the above two inequalities into \eqref{eq:err-decomp}, and get
\begin{equation*}
\Ebb [R(\wbf_T) - R^*] = \Ocal(\varepsilon_T + (1 + \frac{L}{\rho})G_\wbf \Big(\sqrt{\frac{\varepsilon_T}{2\mu}} + \frac{1}{n}\sqrt{\frac{G_\wbf^2}{4\mu^2} + \frac{G_\vbf^2}{\rho\mu}}\Big) +  \frac{4G_\vbf^2}{\rho n}).
\end{equation*}
Now by the choice of $\eta_{\wbf,t}, \eta_{\vbf,t}$, and Theorem \ref{thm:sgda-primal-opt} , we have $\varepsilon_T = \Ocal(\frac{\kappa^{3.5}}{\mu^{2.5}}\frac{1/m + d(\sigma_\wbf^2 + \sigma_\vbf^2)}{T^{2/3}})$. Assume $m$ is a constant. Plugging $\varepsilon_T$ into the preceding inequality and letting $T = \Ocal(n)$ yields the second statement.
\end{proof}

\section{Additional Experimental Details}\label{sec:add-details}
\subsection{Source Code}
For the purpose of double-blind peer-review, the source code is accessible in the supplementary file.

\subsection{Computing Infrastructure Description}
All algorithms are implemented in Python 3.6 and trained and tested on an Intel(R) Xeon(R) CPU W5590
@3.33GHz with 48GB of RAM and an NVIDIA Quadro RTX 6000 GPU with 24GB memory. The PyTorch version
is 1.6.0.

\subsection{Description of Datasets}
In experiments, we use three benchmark datasets. Specifically, ijcnn1 dataset from LIBSVM repsitory, MNIST dataset and Fashion-MNIST dataset are from \citet{lecun1998gradient-supp}, and \citet{xiao2017fashion-supp}. The details of these datasets are shown in Table \ref{tab:general_performance}. For the ijcnn1 dataset, we normalize the features into [0,1]. For MNIST and Fashion-MNIST datasets, we first normalize the features of them into [0,1] then normalize them according to the mean and standard deviation. 
\begin{table*}[th!]
\centering
\setlength\tabcolsep{2.5pt}
\begin{tabular}{c|cccc}
\hline
Dataset &  \#Classes & \#Training Samples & \#Testing Samples & \#Features \\ \hline
ijcnn1 & 2 & 39,992  & 9,998 & 22 \\ 
MNIST & 10 & 60,000 & 10,000 & 784 \\ 
Fashion-MNIST & 10 & 60,000 & 10,000 & 784 \\ \hline
\end{tabular}
\caption{\small \it Statistical information of each dataset for AUC optimization.}
\label{tab:datasets}
\end{table*}

\subsection{Training Settings}
The training settings for NSEG and DP-SGDA on all datasets are shown in Table \ref{tab:training-settings}.
\begin{table*}[th!]
\centering
\setlength\tabcolsep{2.5pt}
\begin{tabular}{|c|c|c|c|c|c|c|c|c|c|c|}
\hline
\multirow{3}{*}{Methods} & \multirow{3}{*}{Datasets} & \multirow{3}{*}{Batch Size} & \multicolumn{4}{c|}{Learning Rate}                         & \multicolumn{2}{c|}{Epochs}                 & \multicolumn{2}{c|}{Projection Size}                 \\ \cline{4-11} 
                  &                   &                   & \multicolumn{2}{c|}{Ori} & \multicolumn{2}{c|}{DP} & \multirow{2}{*}{Ori} & \multirow{2}{*}{DP} & \multirow{2}{*}{Ori} & \multirow{2}{*}{DP} \\ \cline{4-7}
                  &                   &                   &      $\wbf$     &     $\vbf$      &     $\wbf$     &     $\vbf$      &                   &                   &                   &                   \\ \hline
\multirow{3}{*}{NSEG} &ijcnn1& 64 &300&300&350&350 &1000 &15& 100&100 \\ \cline{2-11} 
                  &         MNIST          & 64 &11&11&5&5&100 &15&2&2 \\ \cline{2-11} 
                  &         Fashion-MNIST          &64 & 11&11 &5 &5 &100 &15&3&3 \\ \hline
\multirow{3}{*}{\makecell{DP-SGDA\\(Linear)}} &ijcnn1&64 &300&300&350&350&100&15&10&10  \\ \cline{2-11} 
                  &         MNIST          &64 &11&11&5&5&100&15&2&2 \\ \cline{2-11} 
                  &         Fashion-MNIST          &64 &11&11&5&5&100&15&3&3\\ \hline
\multirow{3}{*}{\makecell{DP-SGDA\\(MLP)}} &ijcnn1&64 & 3000&3001&500&501&10&10 &100&100 \\ \cline{2-11} 
                  &         MNIST          &64 &900&1000&100&210&10&10&2&2 \\ \cline{2-11} 
                  &         Fashion-MNIST          &64 &900&1000&100&210&10&10&2&2\\ \hline
\end{tabular}
\caption{\small \it Training settings for each model and each dataset.}
\label{tab:training-settings}
\end{table*}

\subsection{DP-SGDA for AUC Maximization}

\begin{algorithm}[ht]
\caption{DP-SGDA for AUC Maximization\label{alg:dp-auc}}
\begin{algorithmic}[1]
\STATE {\bf Inputs:} Private dataset $S = \{\zbf_i: i \in [n]\}$, privacy budget $\epsilon, \delta$, number of iterations $T$, learning rates $\{\gamma_t, \lambda_t\}_{t=1}^T$, initial points $(\theta_0, a_0, b_0, \vbf_0)$
\STATE Compute $n_+ = \sum_{i=1}^n \Ibb[y_i=1]$ and $n_- = \sum_{i=1}^n \Ibb[y_i=-1]$
\STATE Compute noise parameters $\sigma_1$ and  $\sigma_2$ based on Eq. \eqref{eq:sigma-sigma}
\FOR{$t=1$ to $T$}
\STATE Randomly select a batch $S_t$
\STATE For each $j \in I_t$, compute gradient $\nabla_\theta f(\theta_t, a_t, b_t, \vbf_t; \zbf_j), \nabla_a f(\theta_t, a_t, b_t, \vbf_t; \zbf_j), \nabla_b f(\theta_t, a_t, b_t, \vbf_t; \zbf_j)$ and $\nabla_c f(\theta_t, a_t, b_t, \vbf_t; \zbf_j)$ based on Eq. \eqref{eq:auc-grad}
\STATE Sample independent noises $\xi_t \sim \Ncal(0, \sigma_1^2 I_{d+2})$ and $\zeta_t \sim \Ncal(0, \sigma_2^2)$
\STATE Update \begin{align*}
\begin{pmatrix}\theta_{t+1}\\a_{t+1}\\b_{t+1}\end{pmatrix} = & \Pi\Bigg\{ \begin{pmatrix}\theta_t\\a_t\\b_t\end{pmatrix} - \gamma_t \Big( \frac{1}{m}\sum_{j \in I_t}\begin{pmatrix}\nabla_\theta f(\theta_t, a_t, b_t, \vbf_t; \zbf_j)\\\nabla_a f(\theta_t, a_t, b_t, \vbf_t; \zbf_j)\\\nabla_b f(\theta_t, a_t, b_t, \vbf_t; \zbf_j)\end{pmatrix} + \xi_t\Big)\Bigg\}  \\
\vbf_{t+1} = & \Pi\Big\{\vbf_t + \lambda_t (\frac{1}{m}\sum_{j \in I_t} \nabla_\vbf f(\theta_t, a_t, b_t, \vbf_t; \zbf_j) + \zeta_t)\Big\}
\end{align*}
\ENDFOR
\STATE {\bf Outputs:} $(\theta_T, a_T, b_T, \vbf_T)$ or $(\bar{\theta}_T, \bar{a}_T, \bar{b}_T, \bar{\vbf}_T)$
\end{algorithmic}
\end{algorithm}

In this section, we provide details of using DP-SGDA to learn AUC maximization problem. AUC maximization with square loss can be reformulated as 
\begin{multline*}
	F(\theta,a,b,\vbf) = \Ebb_\zbf[(1-p)(h(\theta; \xbf) - a)^2\Ibb[y=1] + p(h(\theta; \xbf) - b)^2\Ibb[y=-1]\\ + 2(1+\vbf)(ph(\theta;\xbf)\Ibb[y=-1] - (1 - p)h(\theta; \xbf)\Ibb[y=1])] - p(1-p)\vbf^2]
\end{multline*}
where $\zbf= (\xbf,y)$ and $p = \Pbb[y=1]$. The empirical risk formulation is given as 
\begin{multline*}
F_S(\theta, a, b, \vbf) = \frac{1}{n}\sum_{i=1}^n\Big\{\frac{1}{n_+}(h(\theta; \xbf_i) - a)^2\Ibb[y_i=1] + \frac{1}{n_-}(h(\theta; \xbf_i) - b)^2\Ibb[y_i=-1]\\
 + 2(1+\vbf)\Big(\frac{1}{n_-}h(\theta; \xbf_i)\Ibb[y_i=-1] - \frac{1}{n_+}h(\theta; \xbf_i)\Ibb[y_i=1]\Big) - \frac{1}{n}\vbf^2\Big\}
\end{multline*}

For any subset $S_t$ of size $m$, let $I_t$ denote the set of indices in $S_t$, the gradients of any $j \in I_t$
are given by
\begin{align*}\label{eq:auc-grad}
\nabla_\theta f(\theta, a, b, \vbf; \zbf_j) = & \frac{2}{n_+} (h(\theta; \xbf_j) - a) \nabla h(\theta; \xbf_j)\Ibb[y_j=1] + \frac{2}{n_-} (h(\theta; \xbf_j) - b) \nabla h(\theta; \xbf_j)\Ibb[y_j=-1] \\
& + 2(1+\vbf)\Big(\frac{1}{n_-}\nabla h(\theta; \xbf_j)\Ibb[y_j=-1] - \frac{1}{n_+}\nabla h(\theta; \xbf_j)\Ibb[y_j=1]\Big) \\
\nabla_a f(\theta, a, b, \vbf; \zbf_j) = & \frac{2}{n_+} (a - h(\theta; \xbf_j))\Ibb[y_j=1], \ \ \ \ 
\nabla_b f(\theta, a, b, \vbf; \zbf_j) =  \frac{2}{n_-} (b - h(\theta; \xbf_j))\Ibb[y_j=-1] \\
\nabla_\vbf f(\theta, a, b, \vbf; \zbf_j) = & 2\Big(\frac{1}{n_-}h(\theta; \xbf_j)\Ibb[y_j=-1] - \frac{1}{n_+}h(\theta; \xbf_j)\Ibb[y_j=1]\Big) -\frac{2}{n} \vbf \numberthis
\end{align*}
The pseudo-code can be found in Algorithm \ref{alg:dp-auc}.

\section{Additional Experimental Results}\label{sec:additional-exp}
We show the details of NSEG and DP-SGDA (Linear and MLP settings) performance with using five different $\epsilon\in\{0.1,0.5,1,5,10\}$ and three different $\delta\in \{1e-4,1e-5,1e-6\}$ in Table \ref{tab:general_performance}. From Table \ref{tab:general_performance}, we can find that the performance will be decreased when decrease the value of $\delta$ in the same $\epsilon$ settings. The reason is that the small $\delta$ is corresponding to a large value of $\sigma$ based on Theorem \ref{thm:moments-accountant-privacy}. A large $\sigma$ means a large noise will be added to the gradients during the training updates. Therefore, the AUC performance will be decreased as $\delta$ decreasing. On the other hand, we can find that our DP-SGDA(Linear) outperforms NSEG under the same settings. This is because the NSEG method will add a larger noise than DP-SGDA into the gradients in the training and we have discussed this detail in the Section \ref{experiments:results}. 

We also compare the $\sigma$ values from NSEG and DP-SGDA methods on all datasets in Figure \ref{fig:sigma_delta4_and_5} (a) with setting $\delta$=1e-5 and (b) $\delta$=1e-4. From the figure, it is clear that the $\sigma$ from NSEG is larger than ours in all $\epsilon$ settings. This implies the noise generated from NSEG is also larger than ours.

\begin{table*}[th!]
\centering
\setlength\tabcolsep{2.5pt}
\begin{tabular}{|c|c|cc|c|cc|c|cc|c|}
\hline
\multicolumn{2}{|c|}{Dataset}                  & \multicolumn{3}{c|}{ijcnn1}         & \multicolumn{3}{c|}{MNIST}         & \multicolumn{3}{c|}{Fashion-MNIST}         \\ \hline
\multicolumn{2}{|c|}{\multirow{2}{*}{Algorithm}} & \multicolumn{2}{c|}{Linear} &  MLP  & \multicolumn{2}{c|}{Linear} &   MLP    & \multicolumn{2}{c|}{Linear} &    MLP   \\ \cline{3-11} 
\multicolumn{2}{|c|}{}                  &    NSEG       &    DP-SGDA       &    DP-SGDA   &    NSEG       &    DP-SGDA       &    DP-SGDA       &    NSEG       &    DP-SGDA       &    DP-SGDA       \\ \hline
\multicolumn{2}{|c|}{Original}                  &    92.191 &  92.448 &  96.609  & 93.306  & 93.349  &99.546  & 96.552 &    96.523    &   98.020   \\ \hline
\multirow{5}{*}{$\delta$=1e-4}           &    \makecell{$\epsilon$=0.1}       &   90.231        &     91.229     &   94.020   &    91.285  & 91.962&  98.300 &   95.490 &  95.637 &  96.312   \\ \cline{2-11} 
                            &      \makecell{$\epsilon$=0.5}     &   90.352  &  91.366   &  96.108    &  91.328   &  92.067&  98.703 &  95.533&   95.829 &  97.098  \\ \cline{2-11} 
                            &     \makecell{$\epsilon$=1}      &  90.358       &  91.376      &  96.316   &   91.331 &  92.073&  98.722&    95.536 &   95.840 & 97.143  \\ \cline{2-11} 
                            &    \makecell{$\epsilon$=5}       &   90.363     &   91.385      &  96.326    &   91.334  &  92.079&  98.746&   95.539&  95.849 &   97.208 \\ \cline{2-11} 
                            &     \makecell{$\epsilon$=10}      &  90.363     &  91.387  & 96.329   &  91.335   &   92.080& 98.750 &   95.539 &  95.850  &  97.219\\ \hline
\multirow{5}{*}{$\delta$=1e-5}           &    \makecell{$\epsilon$=0.1}       &  90.168   &  91.169  & 93.274  &  91.266  &  91.910 & 98.092 &  95.468 &  95.535    &  95.989 \\ \cline{2-11} 
                            &     \makecell{$\epsilon$=0.5}      &  90.349  &  91.362  & 96.029 &  91.326  &  92.063&  98.675 &   95.531 &   95.823   &   97.031 \\ \cline{2-11} 
                            &    \makecell{$\epsilon$=1}       &  90.357  & 91.373   & 96.209 &   91.330 &   92.071&  98.714&   95.535  &  95.837  &  97.122 \\ \cline{2-11} 
                            &     \makecell{$\epsilon$=5}      & 90.363   &  91.384 &  96.300   &   91.334 &    92.079&  98.743&   95.538  &   95.848  &  97.200 \\ \cline{2-11} 
                            &     \makecell{$\epsilon$=10}      & 90.363   &  91.386  &  96.301  &   91.334 &  92.080& 98.747 &   95.539 &  95.850 &  97.213 \\ \hline
\multirow{5}{*}{$\delta$=1e-6}           &     \makecell{$\epsilon$=0.1}     & 90.106    &  91.110   &  92.763 &   91.247  &  91.858 &  97.878&  95.446  &  95.468   & 95.692\\ \cline{2-11} 
                            &      \makecell{$\epsilon$=0.5}    &  90.346 & 91.357   &  95.840 &   91.324 &  92.058 &  98.656 &  95.530&  95.816    & 96.988   \\ \cline{2-11} 
                            &     \makecell{$\epsilon$=1}     & 90.355 & 91.371 &  96.167 &   91.330 &  92.070 & 98.705&  95.534 &  95.834  & 97.102\\ \cline{2-11} 
                            &     \makecell{$\epsilon$=5}   & 90.363 & 91.383 &  96.294 &  91.334  &  92.078 & 98.742 & 95.538&   95.848 & 97.198\\ \cline{2-11} 
                            &     \makecell{$\epsilon$=10}  & 90.363& 91.386 &  96.297  &   91.334 &  92.080 &  98.747& 95.539&    95.850  & 97.213  \\ \hline
\end{tabular}
\caption{\small \it Comparison of AUC performance in NSEG and DP-SGDA (Linear and MLP settings) on three datasets with different $\epsilon$ and different $\delta$. The ``Original'' means no noise ($\epsilon=\infty$) is added in the algorithms.}
\label{tab:general_performance}
\end{table*}


\begin{figure*}[ht]
\begin{subfigure}[t]{0.48\linewidth}
    \includegraphics[width=\linewidth]{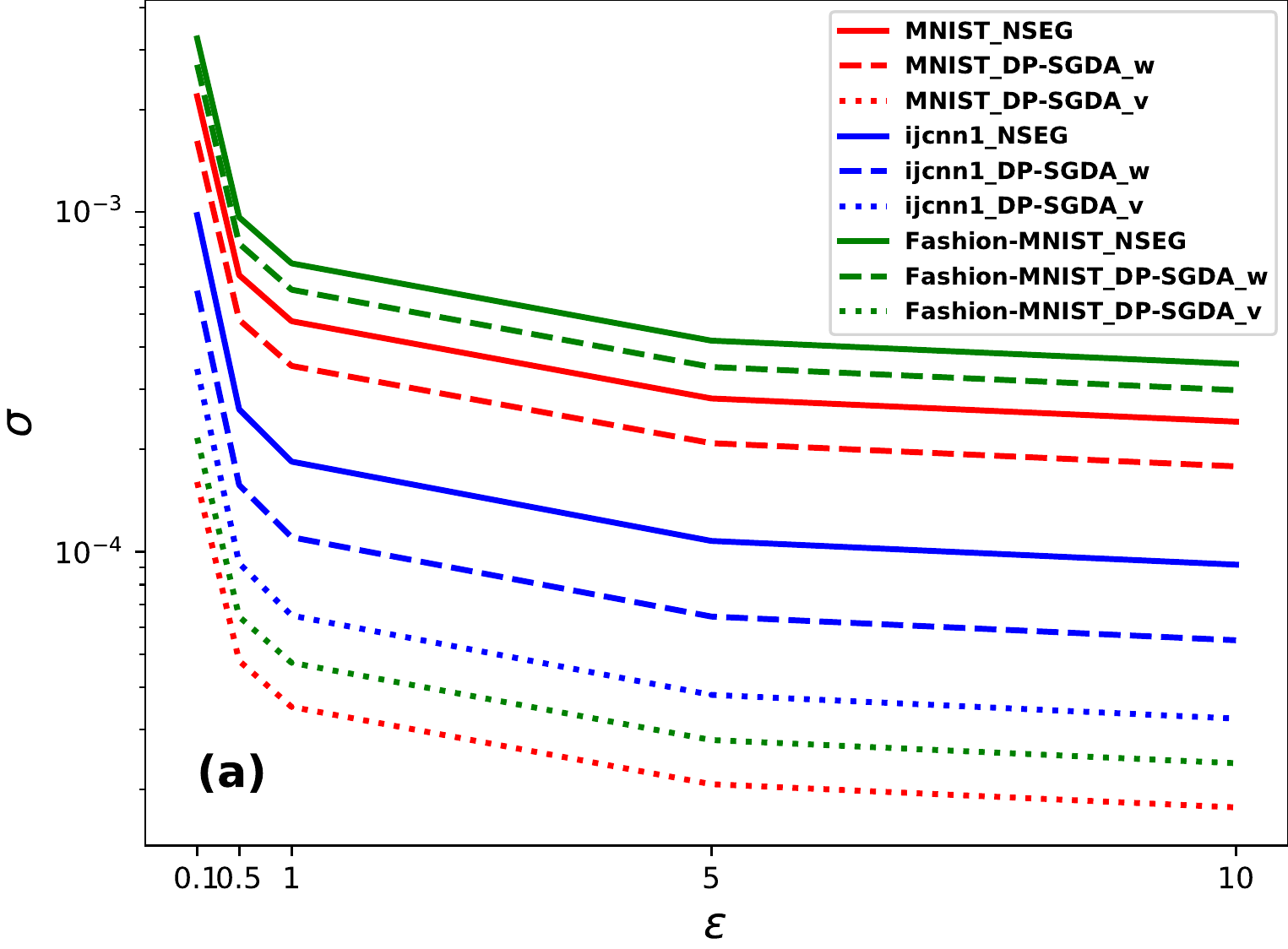}
\end{subfigure}%
    \hfill%
\begin{subfigure}[t]{0.48\linewidth}
    \includegraphics[width=\linewidth]{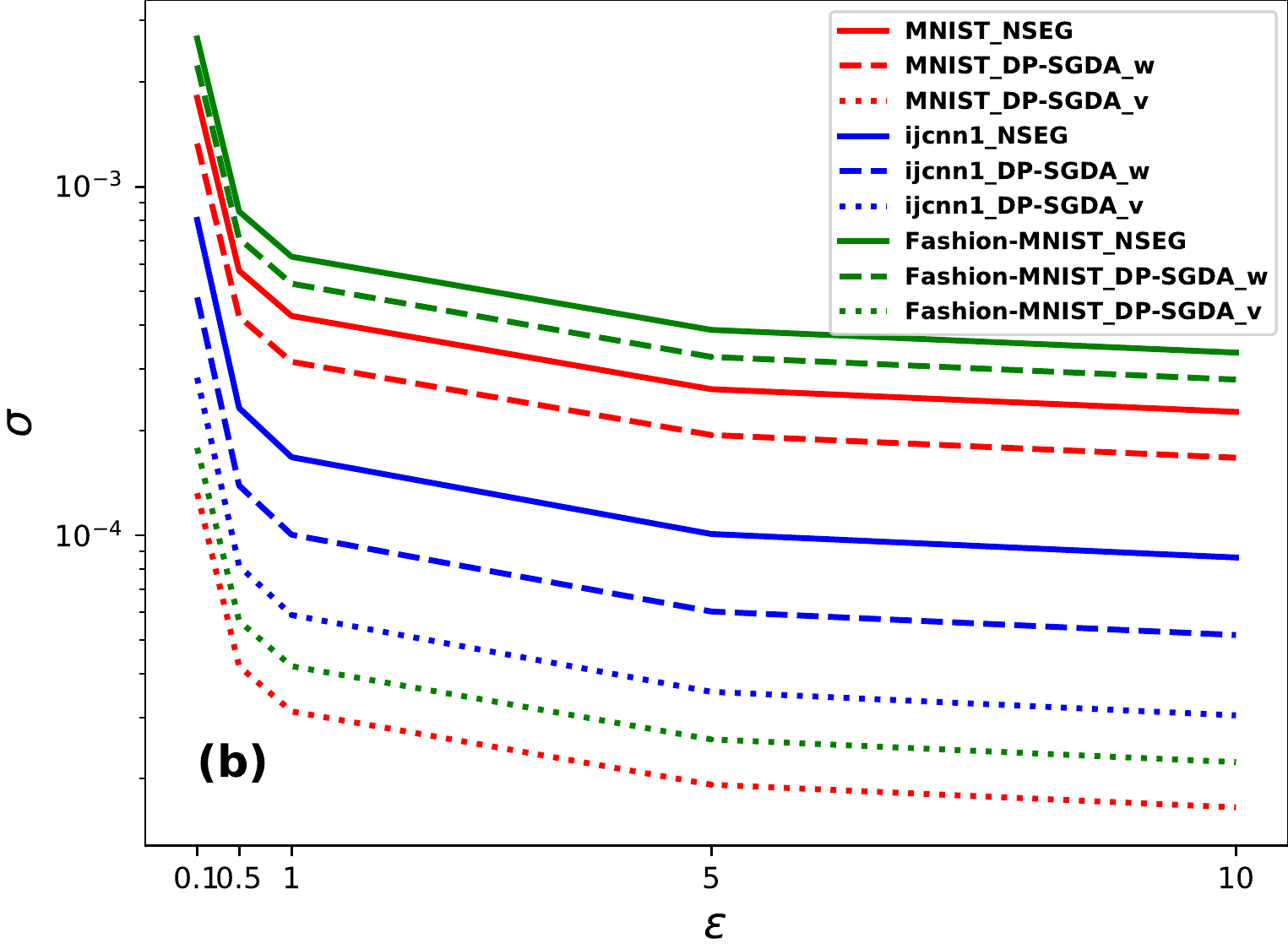}
\end{subfigure}
\caption{\em  Comparison of $\sigma$ in NSEG and DP-SGDA (with Linear setting) on three datasets with different $\epsilon$ and (a) $\delta$=1e-5 and (b) $\delta$=1e-4.}
\label{fig:sigma_delta4_and_5}
\end{figure*}



\end{document}